\documentclass[11pt]{article} % For LaTeX2e
\usepackage{url}
\usepackage[usenames,dvipsnames,table]{xcolor}
\usepackage{smile}
\usepackage{CJKutf8}
\usepackage{graphicx} % more modern
\usepackage{todonotes}
\usepackage{epstopdf}
\usepackage{wrapfig}
\usepackage[colorlinks, linkcolor=blue, anchorcolor=blue, citecolor=blue]{hyperref}
\usepackage[margin=1in]{geometry}
\usepackage[normalem]{ulem}
\usepackage[export]{adjustbox}% http://ctan.org/pkg/adjustbox
\usepackage{mathtools, cuted}
\usepackage{natbib}
\usepackage{enumerate}
\usepackage{microtype}
\usepackage{graphicx}
\usepackage{booktabs} % for professional tables
\usepackage{wrapfig}
\usepackage{xspace}
\usepackage{array}
\usepackage{tabularx}

\usepackage{tcolorbox}

\usepackage{kpfonts}
\DeclareMathAlphabet{\mathsf}{OT1}{cmss}{m}{n}
\SetMathAlphabet{\mathsf}{bold}{OT1}{cmss}{bx}{n}

\providecommand{\norm}[1]{\|#1\|}

\newcommand\redout{\bgroup\markoverwith {\textcolor{red}{\rule[.5ex]{2pt}{0.4pt}}}\ULon}

\newcommand{\ours}{{ENIGMA}}

\begin{document}

\title{Towards Automatic Evaluation of Dialog Systems: A Model-Free Off-Policy Evaluation Approach}

\author{Haoming Jiang, Bo Dai, Mengjiao Yang, Tuo Zhao, Wei Wei \thanks{Work was done during Haoming Jiang's internship at Google Cloud AI. Haoming Jiang and Tuo Zhao are affiliated with Georgia Institute of Technology. Bo Dai and Mengjiao Yang are affiliated with Google Brain. Wei Wei is affiliated with Google Cloud AI. Emails: \texttt{jianghm@gatech.edu}, \texttt{\{bodai,sherryy\}@google.com}, \texttt{tourzhao@gatech.edu},  \texttt{wewei@gatech.edu}. } \thanks{Code on Github: \url{https://github.com/google-research/google-research/tree/master/dialogue_ope}}}

\date{}

\maketitle

%!TEX root = main_dialogue_ope.tex

% \vspace{-0.1in}
\begin{abstract}
% \iffalse
% Reliable automatic evaluation of dialog systems in an {\it interactive} environment has long been overdue. An ideal environment for evaluating dialog systems, also known as the Turing test, requires human interaction, which is usually not affordable for large scale experiments. Though researchers have attempted to use metrics (e.g., perplexity, BLEU) in language generation tasks or some \textit{model-based} reinforcement learning methods for automatic evaluation, these methods only show very weak correlation with the actual human evaluation in practice. \borevise{To achieve a balances between cost and quality in evaluating the dialogue systems}, we propose a new framework named {\ours} for estimating human evaluation score based on recent advances of off-policy evaluation in reinforcement learning. \borevise{{\ours} only requires a handful pre-collected experience data, and therefore does not involve extra human interaction during evaluation, which reduces the cost and makes large-scale automatic evaluations feasible.} More importantly, {\ours} is {\it model-free} and {\it agnostic to the behavior policies} for collecting the experience data (see details in Section~\ref{sec:background}), which significantly alleviates the technical difficulties of modeling complex dialogue environments and human behaviors. 
% Our experiments showed that {\ours} significantly outperforms existing methods in terms of correlation with human evaluation scores.
% \fi
%\iffalse
Reliable automatic evaluation of dialogue systems under an {\it interactive} environment has long been overdue. An ideal environment for evaluating dialog systems, also known as the Turing test, needs to involve human interaction, which is usually not affordable for large scale experiments. Though researchers have attempted to use metrics (e.g., perplexity, BLEU) in language generation tasks or some \textit{model-based} reinforcement learning methods (e.g., self-play evaluation) for automatic evaluation, these methods only show very weak correlation with the actual human evaluation in practice. To bridge such a gap, we propose a new framework named {\ours} for estimating human evaluation scores based on recent advances of off-policy evaluation in reinforcement learning. {\ours} only requires a handful of pre-collected experience data, and therefore does not involve human interaction with the target policy during the evaluation, making automatic evaluations feasible. More importantly, {\ours} is {\it model-free} and {\it agnostic to the behavior policies} for collecting the experience data (see details in Section~\ref{sec:background}), which significantly alleviates the technical difficulties of modeling complex dialogue environments and human behaviors. Our experiments show that {\ours} significantly outperforms existing methods in terms of correlation with human evaluation scores. 
%\fi
%To the best of our knowledge, {\ours} is the first generic model-free off-policy evaluation framework for automatically evaluating {\it multi-turn} dialogue systems for both goal-oriented and open-domain chit-chat settings.
%\Bo{Is this claim too strong? the self-play is also an automatic dialog evaluation system. }
%some model-based methods try to replace the actual human in the interaction by a learned model. However, mimicing human is beyond the current technical limit. In contrast, we propose a model-free framework that need not to learn a human model. 
%Our framework leverages SOTA model-free off-policy evaluation (OPE) methods of reinforcement learning (RL) to estimate the performance of the target model using the historical interaction logs between human and multiple behavior models. 
%The framework is very general and can be applied to both task-oriented and chit-chat dialogue. 
%Thorough experiments show that the proposed framework can make accurate score estimation, which significantly outperforms existing methods in terms of correlation with human evaluation. 

%can be applied to both task-oriented and chit-chat dialogue. 
%The experience data can be collected by multiple behavior policies. 

\end{abstract}

% Reliable automatic evaluation of dialogue systems is long-overdue. Existing approaches 

%!TEX root = DialogOPE.tex

%\Bo{where does the abbreviation of \ours~come from? }

% \section{Guidence}

% 1. existing eval methods not making sense 

% 2. we have fundamental limitation/restriction, need to have some assumption 

% 2. we propsoe a new eval pipeline. 

% 3. collect task specific data

% \vspace{-0.175in}
\section{Introduction}
\label{sec:intro}
% \vspace{-0.075in}

%%%%% Emphasis End2End, compare to stack tracking

% We propose model-free rl framework (MDP) ope, 1. behavior agnostic
% as an application, we use the popular OPE DICE
% Huge improvement R^2 > 0.9. We brigde this gap. 
% Experience data is limited --> pretraining. 
% Issues (OPE difficulties): 1. infinit horizion. 2. experience limited. OPE usually requires large enough coverage --> pretraining. 

%%%%%%%%%%%%%%%% Related work:
%%%%%%%%%%%%%%%% Single Turn (static, post-hoc):
%%%%%%%%%%%%%%%%   mitchell2008vector
%%%%%%%%%%%%%%%%   devault2011toward
%%%%%%%%%%%%%%%%   xiang2014problematic
%%%%%%%%%%%%%%%%   higashinaka2014evaluating
%%%%%%%%%%%%%%%%   gandhe2016semi
%%%%%%%%%%%%%%%%   lowe2017towards
%%%%%%%%%%%%%%%%   dziri2019evaluating
%%%%%%%%%%%%%%%% Multi-Turn (Model-based):
%%%%%%%%%%%%%%%%   li2016deep (hand-crafted features)
%%%%%%%%%%%%%%%%   yu2016strategy (hand-crafted features)
%%%%%%%%%%%%%%%%   shah2018bootstrapping
%%%%%%%%%%%%%%%%   ghandeharioun2019approximating (chitchat,hand-crafted feature)
%%%%%%%%%%%%%%%%   moller2006memo (task)
%%%%%%%%%%%%%%%%   

Building dialog systems that can communicate unhindered with humans in natural language has been one of the most important goals of artificial general intelligence research since the 1950's \citep{turing1950computing}. 
One of the fundamental research bottlenecks for developing such dialog systems falls in evaluation, namely how to measure the performance of these systems in an automatic and scalable manner. 
Different from supervised natural language understanding tasks (e.g., text classification and machine translation), an ideal environment for evaluating dialog systems, also known as the Turing test, involves multi-turn human interaction~\citep{turing1950computing,liu2016not,ghandeharioun2019approximating,see2019what}. While online platforms such as Amazon Mechanical Turk can provide human-based evaluation, they are often expensive and not scalable\citep{lowe2017towards}. 

%Existing dialog systems mainly fall into two categories: One focuses on open-domain chit-chat conversations, and aims to mimic human-human interaction with high quality language (Cite XXX); The other one focuses on goal-oriented conversation, and aims to collect information from user to help complete a particular goal, such as booking flight tickets and dialog-based navigation (Cite XXX). 

Researchers have adopted language quality metrics for single-turn response generation given a fixed context (e.g., BLEU score and perplexity) to automatically evaluate dialog systems~\citep{devault2011toward,xiang2014problematic,higashinaka2014evaluating,gandhe2016semi,lowe2017towards}. However, 
% these metrics are only post-hoc judgments of static experience data, which does not reflect the dynamic and interactive nature of dialogs. 
% Therefore, they fail to capture the dynamic nature of conversation, and cannot well characterize the quality of the conversation with interactive context. 
these metrics only weakly correlate to human evaluation in practice~\citep{liu2016not,ghandeharioun2019approximating}. One cause of such weak correlation is that language quality metrics rely on the exact match between generated text and ground-truth, which generally do not fully overlap. While certain embedding-based metrics have been developed to combat this lack of coverage \citep{mitchell2008vector,dziri2019evaluating}, they are only post-hoc judgments based on static experience data, and does not necessarily reflect the dynamic quality of multi-turn interactive dialog well \citep{ghandeharioun2019approximating}. Moreover, evaluation of goal-oriented dialog systems should be based on how well dialog systems collect information from users and whether the goal is completed; language quality metrics are thus unable to meet these requirements.

%\Bo{This section ignores the major selling point that our metric is goal-oriented metric, while these simple metrics focus on the language properties. Then, the logic flow leads to the discussion about the existing goal-oriented metrics, e.g., self-play, to evaluate the ability of agent to complete the task as below. }
%\Bo{The term ``static'' and ``dynamics'' is not very clear to me. Is this a well-known concept in dialog?}

%\cite{liu2016not,ghandeharioun2019approximating} have shown that 

To overcome the limitations of the aforementioned static evaluation methods, another line of work has proposed to model the interactive process of a conversation as a Markov decision process~(MDP)~\citep{moller2006memo,li2016deep,yu2016strategy,shah2018bootstrapping,ghandeharioun2019approximating,jaques2019way}. Accordingly, automatic evaluation of dialog systems can be formulated as an off-policy evaluation (OPE) problem, where a human subject is the so-called ``environment"  in the reinforcement learning~(RL) literature. For instance, \citet{wei2018airdialogue} propose a model-based approach for goal-oriented dialog systems. They first learn an environment/human model from the experience data consisting of human response, and then evaluate a dialog agent/policy 
by executing the policy within the learned environment. This procedure is known as ``self-play evaluation''. Such a model-based approach requires accurate estimation of an environment/human when both input and output are in a \emph{combinatorially} large space, i.e., the trained model needs to be able to mimic complex human behavior of generating meaningful sentences from huge vocabulary. Unfortunately, such a requirement is far beyond the current capability of model-based reinforcement learning algorithms. As a result, evaluations that rely on accurate modeling of the environment is often unreliable. A similar model-based approach is proposed \citep{ghandeharioun2019approximating} to evaluate open-domain chit-chat dialog systems. In addition to modeling human behavior, they also model the reward function (for mimicking the complex mechanism behind human ratings) based on handcrafted features, which makes evaluation even more unreliable.

%\Bo{We may emphasize more on the combinatorial nature of the model-based metric.}

%the automatic evaluation, however, the self-play requires the learned environment to be very accurate. 

%We instead consider the interaction process of dialog systems by reinforcement learning (RL) formulation. Automatic dialog evaluation is essentially an Off-Policy Evaluation (OPE) problem. 
%%%%%%%%%%%%%%%%%%%%%%%%%%%% Model-based no reliable beyound techniqucal limit
%Some research has been done in model-based dialog OPE . Specifically, they consider learning an environment model, including the human model and the reward function, using experience data and ad-hoc hand-crafted features (e.g., number of question marks, emoji's). The evaluation of the target model is conducted via self-play with the learned environment. Such a self-play approach requires the learned environment to be very accurate. However, mimicing human is beyond the current technical limit and thus the resulting estimation is highly biased. 

%%%%%%%%%%%%%%%%% no Model-based, Model-free. Impossible to mimic human
%%%%%%%%%%%% existing model-free can not handle varying horizon;;(fix or infinite)
%%%%%%%%%%%%% full coverage; we only have limited data and large language space. 
%%%%%%%%%%%%% explain padding. 
%%%%%%%%%%%%% limited coverage by BERT ; pre-train on large text; has prior knowledge; explain more
In this paper, we propose a general OPE framework named {\ours} (\underline{E}valuati\underline{N}g d\underline{I}alo\underline{G} syste\underline{M}s \underline{A}utomatically) for estimating human evaluation score (i.e., how a human would rate a dialog system). 
%\Bo{How to evaluate the chit-chat system? The reward in chit-chat system is not clear in practice.}
Different from the aforementioned model-based approaches, which rely on complex modeling of human behavior given combinatorially large vocabulary, {\ours} takes advantage of recent advances in model-free OPE and avoids direct modeling of dynamic transitions and reward functions in a complex environment. %By doing so, {\ours} significantly alleviates the difficulty in modeling the complicated environment. 
%In addition, {\ours} works for both goal-oriented and chit-chat systems without the need of modeling reward function.
Moreover, {\ours} overcomes several limitations of existing OPE methods in order to evaluate dialog systems: {\bf{(I)}} Existing OPE methods only apply to infinite or fixed horizon settings (where horizon length corresponds to number of turns in a conversation), while conversations, on the other hand, often have {varying} horizon lengths; {\bf{(II)}} Existing OPE methods require experience data to sufficiently cover states and actions a target policy might visit. Due to limited experience data and the combinatorial nature of languages, such a requirement can hardly be satisfied in dialog evaluation; {\bf{(III)}} Certain OPE methods rely on accurate estimation of the behavior policies used to collect the experience data. 
Unfortunately, such behavior policies are humans or complex dialog systems, and estimating their probabilistic model is essentially a challenging imitation learning problem. \footnote{Note that even though some of the model-free OPE estimators still require modeling behavior policies, they are still significantly easier than model-based OPE, which has to model the underlying dialog environment.} 
%\end{enumerate}

To address {\bf{(I)}}, we propose a pseudo state padding method, which augments each conversation into infinitely many turns and yet preserves the original policy value; to address {\bf{(II)}}, we leverage pre-trained language models \cite{devlin2018bert}, which essentially transfer knowledge from open-domain data to alleviate the coverage issue; to address {\bf{(III)}}, we adopt a stationary distribution correction estimation approach \cite{nachum2019dualdice}, which directly models the state-action density ratio between the experience data and the target policy \cite{liu2018breaking}, and is therefore agnostic to the behavior policy. We summarize {\ours} in comparison to existing works in Table \ref{tab:comparison}.\footnote{We only present a compact table due to space limit. More details can be found in Appendix \ref{app:dialog_cmp}.}

%\Bo{how about we delete this paragraph and only discuss the ratio issue in Section  \ref{new-protocol}?}
%Even though ENIGMA has the aforementioned advantages over existing methods, there still exists an information theoretic limit that any off-policy evaluation methods cannot overcome. Specifically, existing literature has shown that no off-policy evaluation methods can perform well when the state-action density ratio between the target policy and experience data is too large \cite{wang2020statistical,xie2019towards}. To avoid this circumstance, we consider an scenario, where the behavior policy for collecting experience data not deviate too much from the target policy. It is worth noting that existing works on dialog system evaluation also explicitly or implicitly enforce such a requirement \cite{lowe2017towards,ghandeharioun2019approximating,see2019what}. More detailed discussions can be found at Section \ref{new-protocol}.

We conduct thorough experiments on evaluating goal-oriented (AirDialog, \citet{wei2018airdialogue}) and chit-chat (ConvAI2, \citet{dinan2020second}) dialog systems to demonstrate the superiority of {\ours}. Specifically, we follow the experimental settings similar to \citet{ghandeharioun2019approximating,see2019what} (See details in Section \ref{sec:exp}), and show {\ours} significantly outperforms the existing static evaluation and self-play evaluation methods in both domains.

%We further present thorough experiments to demonstrate the superiority of {\ours} for evaluating both task-oriented and chit-chat dialog systems. Specifically, for evaluating a specific chit-chat system, {\ours} shows strong Pearson's correlation $\rho_{\rm p}= 0.97$ and Spearman's rank correlation $\rho_{\rm s}= 0.92$ with human evaluation, while the previous state-of-the-art (SOTA) performance in \citet{ghandeharioun2019approximating} is only $\rho_{\rm p}=0.60$ and $\rho_{\rm s}=0.34$; For evaluating a task-oriented system, {\ours} achieves $\rho_{\rm p} = 0.96$ and $\rho_{\rm s}=0.94$, while the previous SOTA performance in \citet{wei2018airdialogue} is only $\rho_{\rm p}=0.73$ and $\rho_{\rm s}=0.63$.

%We further present thorough experiments to demonstrate the superiority of {\ours} for evaluating both task-oriented and chit-chat dialog systems. Specifically, for evaluating a specific chit-chat system, {\ours} shows strong correlation $\rho= 0.97$ with human evaluation, while the previous state-of-the-art (SOTA) performance in \citet{ghandeharioun2019approximating} is only $\rho=0.60$; For evaluating a task-oriented system, {\ours} achieves $\rho=0.96$, while the previous SOTA performance in \citet{wei2018airdialogue} is only $\rho=0.73$.

The rest of this paper is organized as follows: Section~\ref{sec:background} introduces the background of dialog systems and model-free OPE; Section~\ref{sec:method} presents the ENIGMA framework for automatically evaluating dialog systems; Section~\ref{sec:exp} presents the experimental results and detailed analysis; more discussions on {\ours} and related work are presented in Section~\ref{sec:discussions}.% presents detailed discussions.

\begin{table*}[tb!]
% \vspace{-0.1in}
\centering

\caption{Comparison of existing works on Automatic Evaluation of Dialog Systems.}
% \vspace{-0.125in}
\begin{tabular}{c | c | c | c | c }
\toprule\hline
	Method & Criterion & Dynamic
	& Model-Free &  Behavior-Policy \\
	\hline
	 BLEU, PPL & Language Quality & No  & N/A & Human \\  
	  \citet{lowe2017towards} & Language Quality & No  & N/A & Model \& Human \\  
	 \citet{wei2018airdialogue} & Task Completion & Yes  & No &  Human \\
	 \citet{ghandeharioun2019approximating} &Language Score & Yes  & No  & Model \\
	{\ours} & Both & Yes  & Yes  & Model \\
\hline\bottomrule
\end{tabular}
% \vspace{-0.225in}
\label{tab:comparison}
\end{table*}

%single turn eval model human behavior model reward 

% \vspace{-0.125in}
\section{Background}
\label{sec:background}
% \vspace{-0.075in}

% \jhm{we consider dialog is a sequence. dialog is a sequential decision making-> which is essentially a MDP model (a classic sequenctial makeing model); 
% Specifically, at each time step, state is contex until t? action is a response to the context? the policy is the mapping fron s to a? at time step 0 is $\mu_0$. transition kernel is complex and unknown.

% R is a score. For language quality,. 2 goals: language quality/task completion score. Ending state has a reward. Sparse reward.  

% Environment is runnable, calculate empericial reward. which requires human interaction in dialog. with no human, we consider ope based on experience data, which are collected from behavior. 

% Self-play estimate P, "interactive environment" ==> "error accumulated"

% Eval is to calculate expected reward: , varying horizon reward/ equation for (varying horizon) different from tradition rl settings. existing literature handles 

% }

%We provide a brief introduction to the technical background and notations that are needed for deriving our~\ours.

%\vspace{-0.15in}
%\subsection{Dialog Generation as Markov Decision Process}
%\vspace{-0.1in}

\noindent $\bullet$ {\bf Dialog Generation as Markov Decision Process}. A conversation is generated through interactions alternating between an agent $\pi$ (i.e., a dialog system) and an environment $\cE$ (i.e., a human). We denote the dialog as $h=\{e_0,a_1,e_1,...,a_T\}$, where $a_i$ and $e_i$ are sentences generated by $\pi$ and $\cE$ respectively, and $T$ is the number of turns in the conversation. Dialog can be naturally described as a Markov decision process (MDP)~\citep{puterman1995markov} $\cM = \langle \cS, \cA, P, R, \mu_0 \rangle$. Specifically, at the $t$-th turn, state $s_t \in \cS$ captures the previous conversation history $s_t=\{e_0,a_1,e_1,...,a_{t-1},e_{t-1}\}$. An action $a_t \in \cA$ is an agent's response given this context. Conversation can then be represented by the last state and action, i.e., $h=\{s_T,a_T\}$. An agent $\pi$ is essentially a policy that maps $\cS$ to $\cP(\cA)$, where $\cP(\cdot)$ denotes the set of probability measures over the action space. A transition kernel $P(\cdot|s_t, a_t)$ returns $s_{t+1}$ as the state at turn $t+1$, and an environment $\cE$ generates a reward $r_t=R(s_t,a_t)\in [0,1]$. Note that $s_{t+1}$ essentially concatenates $s_{t}$ and $a_t$ with $e_t$, where $e_t$ is a response from the environment (i.e., human) at
\noindent the $t$-th turn. The initial state $s_1 = \{e_0\}$ is randomly sampled from some distribution $\mu_0$. An illustrative example of the dialog on booking a flight ticket is shown in Figure~\ref{fig:dialog}  \citep{wei2018airdialogue}.% , $s_0$ is the ticket database as . %is the pre-specified context, which can be knowledge base, intent, and persona. 

\begin{figure}[htb!]
%  	\vspace{0.05in}
  \centering
  \includegraphics[width=0.7\textwidth]{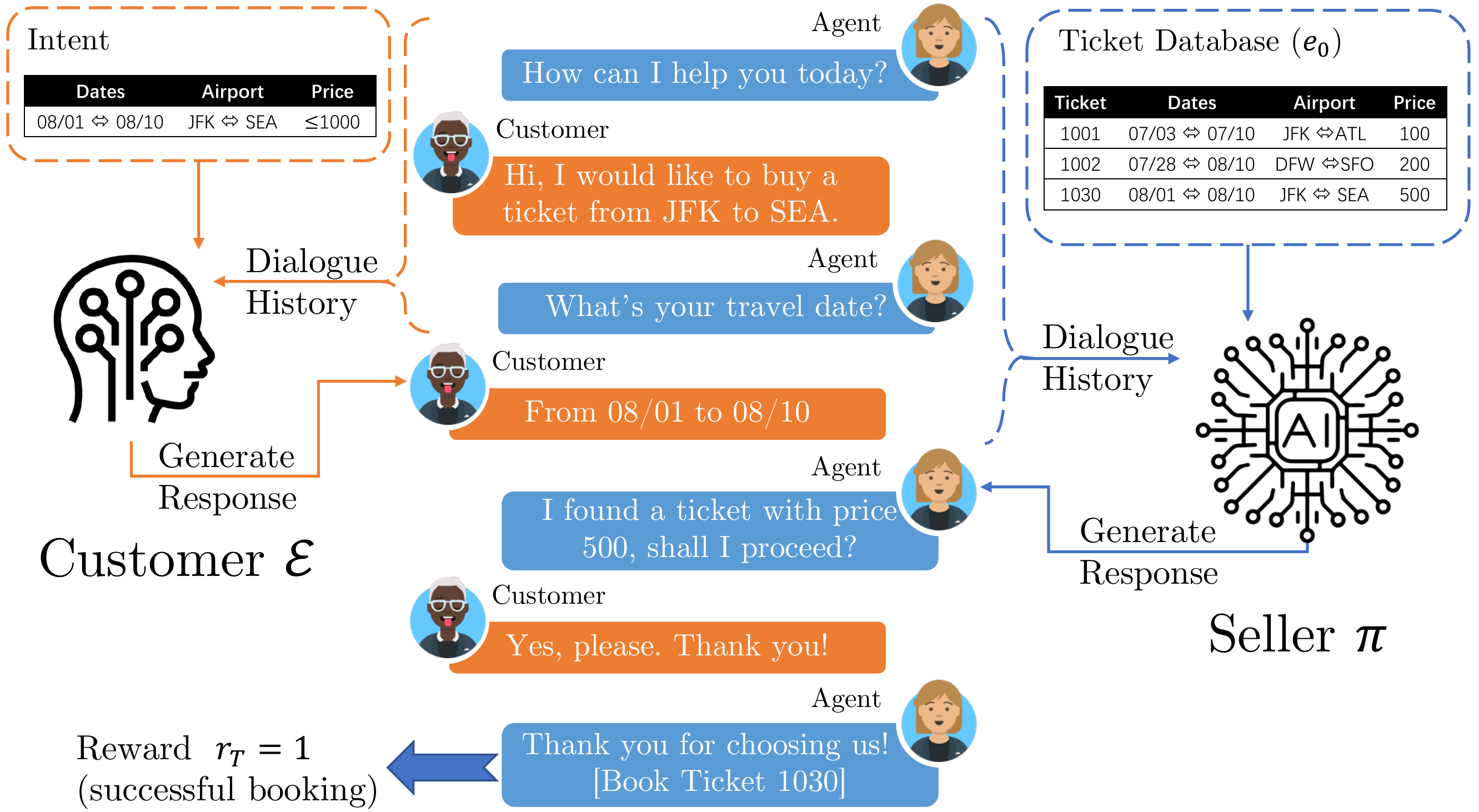}
%   \vspace{-0.075in}
 \caption{Dialog for booking a flight ticket (Airdialog). }
 \label{fig:dialog}
%   \vspace{-0.2in}
\end{figure}

Note that the reward $r_t=R(s_t,a_t)$ generated by the environment depends on the purpose of the dialog system: for open-domain chit-chat dialog, $R$ measures language quality; for goal-oriented agents, $R$ measures task completion scores.  
%\SY{Sentence too long - hard to parse. Maybe something like "$R$ is either given by humans for assessing language quality or by metrics comparing desired and generated actions."}
In particular, we follow the \textit{sparse reward} setting, where each conversation is only evaluated at the ending state, i.e., $r_t = 0$ for $t < T$ \citep{wei2018airdialogue}.

%\vspace{-0.15in}
%\subsection{Automatic Dialog Evaluation as Off-Policy Evaluation}
%\vspace{-0.1in}
%\label{sec:background_ope}
%% only intro infinite here
\vskip-3pt
\noindent $\bullet$ {\bf Automatic Dialog Evaluation as Off-Policy Evaluation}. Dialog evaluation can be naturally viewed as computing the expected reward of the above MDP defined as 
\begin{align}
	\rho(\pi) = \EE_{h \sim \mu_0,\pi,\cE} [R(s_T,a_T)],
	\label{eq:dialog-value}
\end{align}
where $h=\{s_T,a_T\}$ is sampled from the initial distribution $\mu_0$ and the interaction between $\pi$ and $\cE$. When the environment (i.e., human) is accessible, $\rho(\pi)$ can be directly estimated by interaction with the environment, which is known as \textit{on-policy evaluation}~\citep{sutton2018reinforcement}. 
% \citep{metropolis1949monte}. 

In dialog systems, however, interaction with human is expensive or prohibitive in practice, so human-free automatic evaluation is desired. \textit{Off-policy evaluation} (OPE)~\citep{precup2000eligibility} is an appealing choice when access to the environment is limited or unavailable. In particular, OPE can estimate $\rho(\pi)$ based solely on pre-collected tuples $\cbr{\rbr{s, a, r, s'}_i}_{i=1}^N$ from behavior policies that are different from $\pi$. 

OPE has been considered as one of the most fundamental problems in RL. A straightforward approach is to first directly learn an environment model ($R$ and $P$) from experience data and then estimate $\rho(\pi)$ by executing the policy within the learned environment. Such \textit{model-based} OPE exactly corresponds to the so-called ``self-play evaluation'' in the dialog system literature \citep{wei2018airdialogue,ghandeharioun2019approximating}.
Unfortunately, it is notoriously difficult to specify a proper model for highly complicated environments such as a dialog environment (i.e., a human), where the state and action spaces are combinatorially large due to huge vocabulary size and complex transitions. As a result, the estimation error of the environment accumulates as interaction proceeds, and model-based self-play evaluation of dialog systems often becomes unreliable \citep{voloshin2019empirical}.

To address the challenge above, many \textit{model-free} OPE methods that avoid direct modeling of the environment have been proposed. Model-free OPE can be categorized into \textit{behavior-aware} and \textit{behavior-agnostic} methods.
Specifically, behavior-aware methods rely on either knowing or accurately estimating the probabilistic model of behavior policies used for collecting the experience data (e.g., inverse propensity scoring,~\citet{horvitz1952generalization}). Unfortunately, behavior policies are often unknown in practice. Estimating their probabilistic models is also quite challenging, as it requires modeling human behaviors or complex dialog systems.
Behavior-agnostic methods, on the other hand, do not require explicit knowledge or direct modeling of behavior policies, and are therefore more favorable when experience data is collected by multiple (potentially unknown) behavior policies.

Unfortunately, most of the existing model-free behavior-agnostic OPE methods focus on either infinite-horizon~\citep{nachum2019dualdice,zhang2020gendice,yang2020off} or fixed-horizon settings~\citep{yin2020asymptotically,duan2020minimax}, and cannot be applied to evaluating dialog systems whose horizon (number of turns) vary between conversations. While LSTDQ~\citep{lagoudakis2003least} can be adopted to handle varying horizons, it has been shown to not work well under the sparse reward setting \cite{lagoudakis2003least,mataric1994reward}.

\section{{\ours}}
\label{sec:method}

We present the {\ours} framework for automatically evaluating dialog systems using experience data. In particular, {\ours} is model-free and agnostic to behavior policies for generating the experience data. {\ours} has three major ingredients: 
{\bf (1)} pseudo-state padding for converting a dialog into an infinite-horizon MDP,
{\bf (2)} distribution-correction estimation (DICE, \citet{nachum2019dualdice}) with post-normalization for estimating the value of the target policy based on experience data,
and 
{\bf (3)} function approximation and representation learning with pre-trained language models. %t agents, where we improve the estimation by post-normalization, and 

% As the collected experiment usually have limited data size, very diverse behavior policies, we propose the ENIGMA framework to effectively use. 

% There are some problems regarding the experience data: \\
% \noindent~1) Limited resource: collecting a large experience data for off-policy evaluation is usually not plausible. For addressing the bottleneck of data size, we leverage the pre-trained language models as discussed in Section~\ref{sec:method-bert}. We also study the effect of the number of the experience dialogs in the experiments. \\
% \noindent~2) Diversity of behavior policies: the behavior policies used for collecting the data might be very complicated and the probabilistic model might not be available. For handling the experience data collected from diver behavior policies, we leverage the behavior-agnostic OPE estimator -- DICE as discussed in Section~\ref{sec:method-dice}.

% \vspace{-0.125in}
\subsection{Pseudo-State Padding}
% \vspace{-0.075in}
%\jhm{varying ==> fixhorion. We first define pseudo state. define jumps. at the end restart. For example, for a seq with T, after T jump to pseudo T+1. we have infinite-horizon MDP. }

As mentioned in Section 2, existing model-free behavior-agnostic OPE methods cannot handle varying horizon lengths in conversations under the sparse reward setting. %as they assume infinite- or fixed-horizons
To address this issue, we design a special padding scheme, so that the policy value can be estimated by OPE methods from the resulting padded MDP. We first pad conversation sequences with pseudo states, which leads to a padded MDP with a fixed horizon length $T_{\rm max}$. We then convert such a fixed horizon MDP into infinite horizon by augmentation, i.e., we repeatedly concatenate the ending state of the fixed horizon MDP to its initial state. 

More specifically, as illustrated in Figures \ref{fig:aug_mdp}, the policy takes a deterministic action at all pseudo states, i.e., $\pi(a={\rm NextPad}|s={\rm Pad}_{k})=1$. The transition kernel of the new process can be defined as 
\begin{center}
\begin{tcolorbox}[width = 0.7\columnwidth]
\vspace{-0.25in}
\begin{align*}
    &{\rm Conversation~Transition:} 
    \\ 
    &P(s'=s\cup a \cup e|s,a,{\rm incomplete~conv.})=\cE(e|s,a), \\
    &{\rm Jump~into~Pseudo~States:} 
    \\
    &P(s'={\rm Pad}_{T+1}|s, a, {\rm complete~conv.~with~}T{\rm ~turns})\hspace{-1mm}=\hspace{-1mm}1,\\
    &{\rm Jump~between~Pseudo~States:} \\
    & P(s'={\rm Pad}_{k+1}|s={\rm Pad}_{k},a={\rm NextPad},k < T_{\rm max})\hspace{-1mm}=\hspace{-1mm}1,\\
    &{\rm Jump~out~of~Pseudo~States:} \\
    & P(s'|s={\rm Pad}_{T_{\rm max}},a={\rm NextPad})=\mu_0\rbr{s'}.
\end{align*}
% \vspace{-0.275in}
\end{tcolorbox}
\end{center}

This new process is still a valid MDP, as its transition kernel satisfies the Markov property. For notational simplicity, we refer to such an augmented MDP with infinite horizon as ``the augmented MDP''. 
%\Bo{do we need to emphasize this augmentation is only conceptually, and the we do not need to explicit execute it and only need to add $s_0\sim\mu$ to the final tuples?}

%\vspace{0.05in}

Accordingly, the policy value of $\pi$ for the augmented MDP can be defined as %\SY{$d^\pi$ not defined.}
\begin{align}
	\rho_A(\pi) \textstyle=\lim_{N \rightarrow \infty}&{\EE}_{(h_1,h_2,...,h_N) \sim \mu_0,\pi,\cE} \textstyle[\frac{1}{N T_{\rm max}} \sum_{i=1}^{N}\sum_{t=1}^{T_{\rm max}} R(s_{t}^{(i)},a_{t}^{(i)})],
	%= \mathop{\EE}_{(h_1,h_2,...) \sim \mu_0,\pi,\cE} [\frac{1}{N T_{\rm max}} \sum_{i=1}^{N} r_{T_i}^{(i)}],
	\label{eq:aug-dialog-value}
\end{align}
where $h_i$'s are padded conversations sampled from interactions between $\pi$ and $\cE$. Since there is only one non-zero reward for every $T_{\rm max}$ steps, rewards in the augmented MDP are also sparse.

\begin{figure}[htb!]
	\centering
%  	\vspace{-0.15in}
	\includegraphics[height=1.8in]{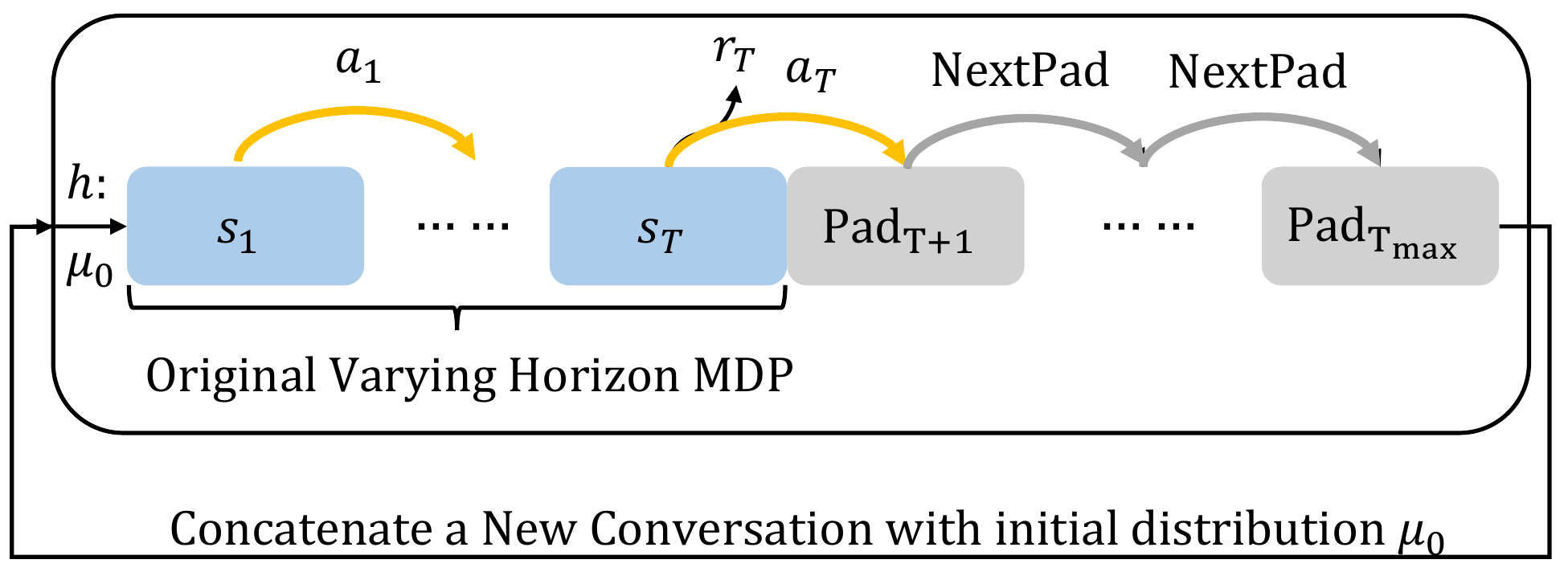}
%  	\vspace{-0.15in}
	\caption{Augmented MDP with Infinite Horizon. }
% 	\vspace{-0.125in}    
	\label{fig:aug_mdp}
\end{figure}

We justify such a padding scheme in the following theorem showing that the augmented MDP has a unique stationary distribution, and guarantees any policy $\pi$ to have a finite policy value. 
Moreover, the policy value of $\pi$ under the augmented MDP is proportional to its counterpart under the original MDP without augmentation. Due to space limit, we defer the proof to Appendix~\ref{app:thm}. 

\begin{theorem}
	\label{thm:pad}
	The augmented MDP with infinite horizon satisfies the following properties:
	%\begin{itemize}[topsep=0pt,parsep=2pt,partopsep=0pt, leftmargin=*]
		%\item the new environment is an MDP, as the pseudo state padding satisfies the Markov property;
% 		\vskip-3pt
	\noindent $\bullet$ It has a unique stationary state-action visitation distribution $d^\pi(s,a)$;
	
% 	\vskip-3pt
		\noindent $\bullet$   For the state-action pair $(s_t,a_t)$ in a conversation $h$ with padded pseudo states, we have
        \begin{align}
        %d^\pi(s,a) := \EE_{\mu_0, \pi}[ \lim_{N \rightarrow \infty}\frac{1}{N T_{\rm max}} \sum_{i=1}^{N}\sum_{t=1}^{T_{\rm max}} \ind(s_{t}^{(i)}=s,a_{t}^{(i)}=a)];
        d^\pi(s_t,a_t) = &\textstyle\frac{1}{T_{\rm max}}\sum_{\{(s_k,a_k)\}_{k=1}^{t-1}}[ \mu_0(s_1)\pi(a_1|s_1)\notag\\
        &\hspace{-0.3in}P(s_2|a_1,s_1)\cdots P(s_t|a_{t-1},s_{t-1})\pi(a_t|s_t)],
        \label{eq:stationary_dist}
        \end{align}
        where $\{(s_k,a_k)\}_{k=1}^{t-1}$ are the state-action pairs in the same conversation as $(s_t,a_t)$;
        %with padded pseudo states.
		%iii) the density of the ending state-action pair $(s_T,a_T)$ of a complete dialog $h$ in \eqref{eq:policy-value} determines the density of $h$ in \eqref{eq:dialog-value} by$d(h) = T_{\rm max} d(s_{T},a_{T})$;
		
% 		\vskip-3pt
		\noindent $\bullet$  The policy value can be computed by sampling from $d^\pi(s,a)$, and we have
        \begin{align}
		    \rho_{A}(\pi) = {\EE}_{(s,a) \sim d^\pi(s,a)} [R(s,a)] = \rho(\pi)/T_{\rm max}.
	    \label{eq:aug-dialog-value-2}
        \end{align}
\end{theorem}

\begin{remark}
Some OPE methods, e.g., LSTDQ \cite{lagoudakis2003least}, can handle fixed horizons, therefore only applying the fixed-horizon padding would suffice. DICE estimators \citep{nachum2019dualdice}, on the other hand, can only handle infinite horizons, therefore the infinite-horizon augmentation is necessary.
\end{remark}
\begin{remark}
 Note that in practice, we do not actually need to concatenate infinitely many conversations for computing $\rho_A(\pi)$. As suggested by \eqref{eq:aug-dialog-value-2}, $\rho_A(\pi)$ can be computed based on $d^\pi(s_t,a_t)$ defined in \eqref{eq:stationary_dist}, which is the product of only finite terms.
\end{remark}
% \Bo{This remark is not necesasry. In the original OPE for infinite horizon MDP, they only need the tuples. }

% \vspace{-0.1in}
\subsection{Model-Free Behavior-Agnostic DICE Estimator}
\label{sec:method-dice}
% \vspace{-0.05in}
%\borevise{We recast the dialogue evaluation as an OPE problem with sparse reward on behavior-agnostic data. The regularized DICE \citep{yang2020off} is a natural choice.}\SY{Since directly sampling from $d^\pi$ is expensive or infeasible, DICE leverages the following change of variable in expressing the policy value}
With the proposed augmentation, we obtain an infinite horizon MDP from which the policy value of the original MDP can be recovered. We then apply DICE \citep{nachum2019dualdice,yang2020off} to estimate $\rho_A(\pi)$ based on pre-collected experience data $\cD = \cbr{\rbr{s, a, r, s'}_i}_{i=1}^N$ without interacting with $\cE$ (i.e., a human), where $(s,a) \sim d^\cD$ are samples from some unknown distribution $d^\cD$. 
We slightly abuse the notations and use $(s,a,r,s') \sim d^\cD$ as a shorthand for $(s,a) \sim d^\cD,r=R(s,a),s' \sim P(\cdot|s,a)$, 
which simulates sampling form the dataset $\cD$.

DICE is a model-free policy evaluation method (without explicitly modeling $\cE$) and does not require knowledge of behavior policies for generating the experience data, which provides a more reliable estimation of $\rho_A(\pi)$ than other OPE methods. Specifically, DICE decomposes $\rho_A(\pi)$ into:
\begin{align}
	\rho_A(\pi) = \EE_{ (s,a,r) \sim d^\cD } [\zeta(s,a) r],
	\label{eq:dice_goal}
\end{align}
where $\zeta(s,a) := d^\pi(s,a)/d^\cD(s,a)$ is the \textit{ distribution correction ratio}. Then DICE estimates $\zeta$ by solving the following regularized minimax optimization problem:
\begin{align}\label{eq:dice_minmax_obj}
	\max_{\zeta \geq 0} \min_{\nu,\lambda} &L_D(\zeta, \nu, \lambda) 
	=  \EE_{ (s,a,r,s') \sim d^\cD,a' \sim \pi(s')}[ \zeta(s,a)\notag\\
	&\cdot(\nu(s',a') - \nu(s,a)) ]+ \EE_{ (s,a) \sim d^\cD} [ \lambda(\zeta(s,a)\notag\\
	& - 1) ]  - \alpha_\zeta \cdot \EE_{(s,a) \sim d^\cD}[f(\zeta(s,a))].
\end{align}
%\begin{align}
%	\max_{\zeta \geq 0} \min_{Q,\lambda} L_D(\zeta, Q, \lambda) 
%	:= &  \mathop{\EE_{ (s,a,r,s') \sim d^\cD}}_{a' \sim \pi(s')} [ \zeta(s,a) (\alpha_R r + Q(s',a') - Q(s,a)) ] \notag\\
%	&  + \EE_{ (s,a) \sim d^\cD} [ \lambda(\zeta(s,a) - 1) ]- \alpha_\zeta \cdot \EE_{(s,a) \sim d^\cD}[f(\zeta(s,a))]. 
%	\label{eq:dice_minmax}
%\end{align}
where $\nu(s,a)$'s are auxiliary variables, $f$ is a convex regularizer (e.g., $f(x)=x^2$), and $\alpha_\zeta$ is a tuning parameter. Due to the space limit, we omit the details of deriving the DICE estimator. Please refer to \citet{yang2020off} for more technical details.

% \vskip-3pt
\noindent $\bullet$ \textbf{Post-Normalization}. Note that \eqref{eq:dice_minmax_obj} handles the constraint $\EE_{(s,a) \sim d^\cD} \zeta(s,a) = 1$ by Lagrange multipliers $\lambda$, which cannot guarantee that the constraint is \emph{exactly} satisfied when solving \eqref{eq:dice_minmax_obj} using alternating SGD-type algorithms \citep{dai2017boosting,chen2018landscape}. To address this issue, we propose a post-normalization step that explicitly enforces the constraint:
% \vspace{-1mm}
\begin{align}
    \textstyle
	\rho_{n}(\pi) = \sum_{(s,a,r) \sim d^\cD} {\zeta}(s,a) r \Big/ \sum_{(s,a) \sim d^\cD} {\zeta}(s,a).
	\label{eq:dice_est_sn}
\end{align}
As we will see in our experiments in Section~\ref{sec:exp}, the post-normalization step is crucial for DICE to attain good estimation accuracy in practice; without the post-normalization, we observe potential divergence in terms of policy value estimation.

% Comparing with other OPE methods, regularized DICE is model-free and agnostic to behavior policies. Moreover, regularized DICE is more computational stability.
%\begin{remark}
% \vskip-3pt
\noindent $\bullet$ {\bf Why do we prefer DICE?} Deep Q-learning and its variants are another popular model-free and behavior-agnostic approach to off-policy evaluation. However, due to the sparse rewards in dialogs, fitting the state-action value function (i.e., the $Q$-function) in deep Q-learning is notoriously difficult~\citep{mataric1994reward}. We observe in Section~\ref{sec:exp} that deep Q-learning is computationally unstable.

In contrast, DICE only needs to estimate the density correction ratio $\zeta$, which is decoupled from the rewards associated with the policy value as shown from \eqref{eq:dice_minmax_obj}. This significantly alleviates the computational challenge incurred by sparse rewards. Moreover, DICE also applies the post-normalization, additional regularization (i.e., $\EE_{(s,a) \sim d^\cD}[f(\zeta(s,a))]$), and constraints on $\zeta$  (i.e., $\zeta \geq 0$ and $\EE_{ (s,a) \sim d^\cD} [\zeta(s,a)]=1$), all of which further stabilize training. These features allow DICE achieve better estimation performance than deep Q-learning in dialog systems evaluation.

% \Bo{lower bound, the necessary of extra dataset, observation: policy learned from the offline data behaves similarly. }

Recent progresses in OPE based on density ratio estimation are remarkable~\citep{liu2018breaking,nachum2019dualdice,xie2019towards,uehara2019minimax}, however, there exists a statistical limit in off-policy evaluation. Specifically, the Cramer-Rao lower bound of the MSE has been established in~\citet{jiang2016doubly}, which is proportional to the square of the density ratio. This implies that we can only obtain accurate estimation of policy value only if the ratio $\zeta$ is not too large. While the ratio-based minimax algorithms should have achieved the asymptotic lower bound~\citep{kallus2019efficiently,yin2020asymptotically}, even better estimation results can be obtained when behavior and target policies are more similar. 
We thus introduce an experience data collection protocol in Section~\ref{new-protocol} which satisfies the bounded ratio requirement and ensures the success of OPE methods.

\subsection{Function Approximation with RoBERTa}
\label{sec:method-bert}
% \vspace{-0.075in}

%Despite the apparent advantages of DICE estimators, directly applying DICE to dialog evaluation will fall short due to the large combinatorial state-action space, since DICE requires good coverage (hence an enormous amount of pre-collected data) to ensure reliable estimation.

Despite the apparent advantages of DICE estimators, directly training DICE from scratch will fall short due to the bounded ratio requirement being quickly broken in the large combinatorial state-action space in dialog.

We alleviate this issue by learning reliable representations from an enormous amount of pre-collected data. We resort to the domain transfer learning technique, also known as language model pre-trained and fine-tuning \cite{devlin2018bert}. For example, RoBERTa\citep{liu2019roberta} is an extremely large bidirectional transformer model \citep{vaswani2017attention} pre-trained using huge amounts of open-domain text data in a self-supervised/unsupervised manner. RoBERTa is particularly attractive to the dialog evaluation task due to the following merits: (1) the pre-training process does not require any labelled data; (2) the pre-trained models are publicly available; (3) the massive model sizes (usually with hundreds of millions or billions of parameters) allow these models to effectively capture rich semantic and syntactic information of natural language (rather than enumerating the original combinatorial language space).

%particularly attractive to this task due to the following merits: First, they are very large bidirectional self-attention networks~\citep{vaswani2017attention} pretrained with huge amounts of unlabeled data; Second, due to their massive sizes (usually hundreds of millions or billions of parameters), they have strong expressive power to capture general semantics and syntactic information effectivel

 %so that the coverage requirement of DICE will hold in the embedding space.

To transfer the knowledge from the pre-trained RoBERTa model to dialog evaluation, we parameterize $\zeta$ and $\nu$ as follows: We keep the pre-trained RoBERTa encoder layer and replace the original mask language modeling head by a two-layer fully connected network with a scalar output. For simplicity, we denote the corresponding parametric forms of $\zeta$ and $\nu$ as RoBERTa-$\zeta$ and RoBERTa-$\nu$, respectively. Note that we only need RoBERTa-$\zeta$ and RoBERTa-$\nu$ to share the same encoder, as illustrated in Figure~\ref{fig:func_approx}. We then use RoBERTa-$\zeta$ and RoBERTa-$\nu$ as the initial solution to solve \eqref{eq:dice_minmax_obj}, which is also known as the fine-tuning step~\citep{devlin2018bert,liu2019roberta}.

\begin{figure}[htb!]
%  	\vspace{-0.1in}
	\centering
	\includegraphics[width=0.6\textwidth]{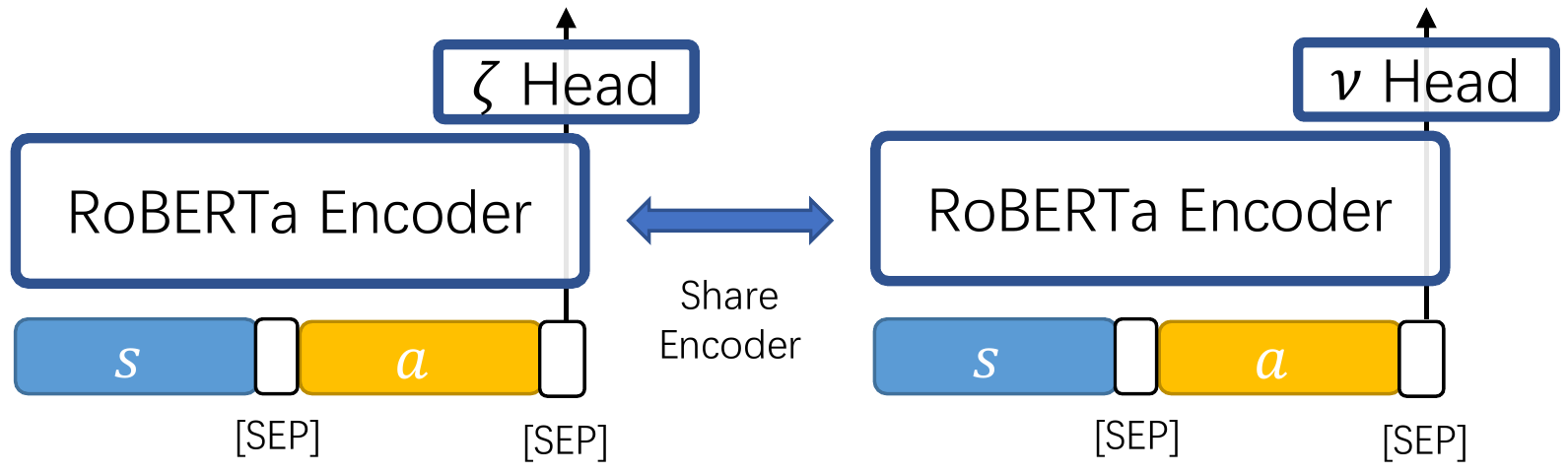}
%  	\vspace{-0.15in}
	\caption{Function Approximation with RoBERTa.}
% 	\vspace{-0.15in}
	\label{fig:func_approx}
\end{figure}

With a properly designed mask, the self-attention mechanism in the bi-direction transformer architecture allows us to efficiently compute $\zeta(s,a)$ and $\nu(s,a)$ for all state-action pairs in the same dialog simultaneously. Due to the space limit, we defer the mask design details to Appendix~\ref{app:bert}.

\subsection{Summary}
\label{sec:method-procedure}

We summarize ENIGMA  using SGD with batch-size 1 in Algorithm~\ref{algo:summary}. We defer the details of {\ours} with mini-batch SGD to Appendix~\ref{app:algo} (Algorithm~\ref{algo:main}).

\begin{algorithm}[htb]
	\caption{ENIGMA}\label{algo:summary}
	\begin{algorithmic}[1]
		\INPUT Experience conversations $\cD = \{(h_i= \{e^{(i)}_0,a^{(i)}_1,e^{(i)}_1,...,a^{(i)}_{T_i} \}, r^{(i)})\}_{i=1}^{N}$, Target Policy $\pi$, Padding Length $T_{\rm max}$, Regularization function $f$, DICE hyper-parameters $\alpha_\zeta$, $\alpha_R$ 
		\OUTPUT Performance Estimation $\hat \rho_{n}(\pi)$
		\PARAMETER  $\zeta=\{$RoBERTa-$\zeta, [\zeta_{{\rm pad},t}]_{t=1}^{T_{\rm max}}\}$, $\nu=\{$RoBERTa-$\nu, [\nu_{{\rm pad},t}]]_{t=1}^{T_{\rm max}}\}$, $\lambda$
% 		\item[\textbf{\textit{Generate OPE Data}}]
% 		\For{ $({h}_i, r^{(i)}) $ in $ \cD$}
% 		\For{$t$ in $1, \cdots, T_i$}
% 		\State $\tilde{a}_t^{(i)} \sim \pi(\{e^{(i)}_0,a^{(i)}_1,e^{(i)}_1,...,e^{(i)}_{t-1} \})$ // Sample Action From Target Policy
% 		\EndFor
% 		\EndFor
% 		\State $\tilde{\cD}:= \{(\tilde{h}_i, r^{(i)})\}_{i=1}^{N} = \textrm{Padding}(\cD, T_{\max})$
% 		\State $\tilde{\cD} = \{(\tilde{h}_i=e^{(i)}_0,a^{(i)}_1,e^{(i)}_1,...,a^{(i)}_{T_i} \}, r^{(i)}) \}_{i=1}^{N}$
% 		\item[\textbf{\textit{Estimate $\zeta$ by Regularized DICE}}]
		\While{ Sample $({h}, r)$ in $\cD$ }
% 		\For{ $(\tilde{h}_i, r^{(i)})$ in $\cB$}
% 		\State // infinite-horizon concatenation
		\State $\zeta_{0} = \zeta_{{\rm pad},T_{\rm max}}$,~~ $\nu_{0} = \nu'_{0} = \nu_{{\rm pad},T_{\rm max}}$ %// infinite-horizon concatenation
		\For{$t$ in $1, \cdots, T$}
		\State $\zeta_{t} = $ RoBERTa-$\zeta(e_0,a_1,...,e_{t-1},a_{t})$
		\State $\nu_{t} = $ RoBERTa-$\nu(e_0,a_1,...,e_{t-1},a_{t})$
		\State $\tilde{a}_{t} \sim  \pi(e_0,a_1,...,e_{t-1})$
		\State $\nu'_{t} = $ RoBERTa-$\nu(e_0,a_1,...,e_{t-1},{\color{red}\tilde{a}_{t}})$
		\EndFor
		\For{$t$ in $T + 1, \cdots, T_{\rm max}$}
% 		\State // add padding state
		\State $\zeta_{t} = \zeta_{{\rm pad},t}$,~~ $\nu_{t} = \nu'_{t} = \nu_{{\rm pad},t}$ 
		\EndFor
		\State $L_D(\zeta, \nu, \lambda) = \frac{1}{T_{\rm max}} [ \sum_{t=0}^{ T_{\rm max} -1 } [\zeta_t (\nu'_{t+1} - \nu_t ) + \lambda (\zeta_t-1)  -\alpha_\zeta f(\zeta_t) ]]$
% 		\EndFor
% 		\State $L_D(\zeta, \nu, \lambda) = \frac{1}{|\cB|}\sum_{i \in \cB} \ell_i$ 
		\State SGD update based on $\frac{\partial L_D}{\partial \nu}$,$\frac{\partial L_D}{\partial \lambda}$,$\frac{\partial -L_D}{\partial \zeta}$.
		\EndWhile
% 		\item[\textbf{\textit{Policy Value Estimation }}]
		
		\For{ $({h}_i, r^{(i)})$ in $\cD$}
		\State $\zeta_i = $ RoBERTa-$\zeta({h}_i)$
		\EndFor
		\State \textbf{Return} $\hat{\rho}_{n}(\pi) = \sum_{i} \zeta_i r^{(i)} \big/  \sum_{i} \zeta_i$
	\end{algorithmic}
\end{algorithm}

\section{Experiments}
\label{sec:exp}
% \vspace{-0.075in}

We empirically evaluate {\ours} on two dialog datasets: AirDialog \citep{wei2018airdialogue} for goal-oriented tasks and ConvAI2 \citep{dinan2020second} for open-domain chit-chat respectively. See details of experimental setup in Appendix~\ref{app:exp-setup}.
\footnote{We release our source code for ENIGMA algorithm on Github: https://sites.google.com/view/eancs/shared-task.}

% \vspace{-0.125in}
\subsection{Policy Training Data and Experience Data}\label{new-protocol}
% \vspace{-0.075in}

%\jhm{from a information perspective. We not only need to eval good but also needd to eval}

As mentioned in Section \ref{sec:method-dice}, there exists an information theoretic limit for all off-policy evaluation methods: no method can perform well when the state-action density ratio between the target and behavior policy is too large. To avoid such a circumstance, we need to ensure that the experience data collected by a behavior policy do not deviate too much from data induced by the target policy. Unfortunately, both datasets used in our experiments do not satisfy such a requirement. AirDialog, for example, consists of dialog between humans, which are near-perfect golden samples as human agents almost always successfully book tickets for customers. Dialog system agents, on the other hand, have many failure modes (i.e., the target policy/agent does not book the correct ticket for a human customer). Hence, directly using human dialog as the behavior data to evaluate dialog agents is subject to limitations.

%\sy{This part is confusing. We don't actually have a behavior policy that collects experience data, correct? If so maybe we should avoid using the term "behavior policy" and stick with "experience data" instead. From my understanding, AirDialog uses PhD student-agent dialog as experience data and ConvAI2 uses a small percentage ($50\%$ or $10\%$) of target data to be evaluated as experience data. Is this correct? If so maybe we can clarify that here instead of having a separate "Experience Data" bullet point in the "Goal-Oriented System: AirDialog" section and nothing about experience data in the "Open-Domain Chit-Chat Systems: ConvAI2" section?}
In order to properly evaluate an imperfect target policy in the presence of the information theoretic limit, we refer to \citet{lowe2017towards,ghandeharioun2019approximating,see2019what}, and collect experience data using behavior policies similar to the target policy. To avoid confusion, we call data collected by the behavior policy ``experience data'' and data used to train an agent ``policy training data''. More details are elaborated below for each dataset.

% \vskip-2pt
It is worth noting that existing work on dialog systems evaluation also enforces similar requirements. For example, \citet{lowe2017towards} show higher Pearson correlation coefficient ($0.37$) between automatic metrics and human ratings when behavior policies contain the target policy. When the target policy is excluded from behavior policies, however, the correlation is only $0.13$, even lower than the meaningless correlation between dialog lengths and human ratings ($0.27$). Another example is \citet{ghandeharioun2019approximating}, where the studied agents are similar to each other in their hierarchical architectures, hyperparameters, and training data.

% \vskip-2pt
We compare the experience data used in this paper with existing works in Table~\ref{tab:cmp_prior}.

\begin{table}[htb!]
% \vspace{-0.155in}
    \centering
    \caption{Number of Dialogues/Agents in Experience Data. $^\dagger$: 4,104 dialog turns and the number of dialogs is not avaliable.}
    % \vspace{-0.125in}
    \begin{tabular}{c|cc}
    \toprule
    \hline
        Experience Data & Agents & Dialogs \\
        \hline
       \citet{lowe2017towards}                  & 4 & N/A $^\dagger$ \\
       \citet{ghandeharioun2019approximating}  &  12 & 512 \\
       Ours-Airdialog & 24 & 2400 \\
       Ours-ConvAI2 & 29 & 2616 \\ \hline
       \bottomrule
    \end{tabular}
    \label{tab:cmp_prior}
    % \vspace{-0.125in}
\end{table}

\subsection{Goal-Oriented Systems}
\label{sec:exp-air}
% \vspace{-0.075in}

We first test \ours~ for evaluating goal-oriented dialog systems on a flight ticket booking task.

% \vskip-3pt
$\bullet$ \textbf{Policy Training Data}. We use the {\bf \emph{AirDialog}} dataset\footnote{\url{https://github.com/google/airdialogue}} for policy training \cite{wei2018airdialogue}. It contains 402,038 pieces of dialog from human sellers and human customers collaborating on buying flight tickets. We use different proportions of the dataset and different hyperparameters to train 24 seller agents using behavior cloning (See Appendix~\ref{app:air-trans} for details) \footnote{We also demonstrate that {\ours} can be applied to rule based agent in Appendix~\ref{app:rule-air}.}. 

% \vskip-3pt
$\bullet$ \textbf{Experience Data}. We invite 20 people to evaluate the 24 seller agents.
Specifically, each of the 20 human customers interacts with a seller agent 5 times to generate 100 pieces of dialog, and gives each piece an evaluation score between 0 and 1. The final score an agent receives is the average of the 100 scores. We consider three types of scores: flight score, status score, and overall reward used in \citet{wei2018airdialogue}.

%Specifically, each seller agent interacts with all 20 human customers for 5 times, and generate 100 pieces of dialog. For each piece, its corresponding human customer generates a human evaluation score between 0 and 1. The final human evaluation score for each agent are the average of the scores over its 100 generated pieces (5 from each human customers). We consider three types of scores, including flight score, status score and overall reward.

\begin{table*}[htb!]
% \vspace{-0.1in}
\scriptsize
\caption{The correlation between two metrics. Each column is a task completion score obtained by interacting human customers (``Selected Agents'' denotes only evaluating agents with reasonably good performance).}% Each row is an automatic metric. }
\label{tab:air_r2}
% \vspace{-0.175in}
\begin{center}
\begin{tabular}{c|l|cccc|cccc}
\toprule \hline 
\multicolumn{1}{c|}{\multirow{2}{*}{\bf Setting}}&\multicolumn{1}{c|}{\multirow{2}{*}{\bf Method}} & \multicolumn{4}{c|}{\bf Pearson Correlation}& \multicolumn{4}{c}{\bf Spearman's Rank Correlation } \\
\cline{3-10}
 &  &\multicolumn{1}{c}{\bf Flight Score } &\multicolumn{1}{c}{\bf Status Score } &\multicolumn{1}{c}{\bf Reward } &\multicolumn{1}{c|}{\bf Average } &\multicolumn{1}{c}{\bf Flight Score } &\multicolumn{1}{c}{\bf Status Score } &\multicolumn{1}{c}{\bf Reward } &\multicolumn{1}{c}{\bf Average }
\\ \hline 
& BLEU        & ~0.1450 & -0.1907 & -0.0709 &-0.0389 & ~0.0370 & -0.1453 & -0.1472 &-0.0852\\
All & PPL         & -0.1598 & ~0.1325 & ~0.0195 &-0.0026 & -0.1817 & ~0.0090 & -0.0039 &-0.0649\\
Agents & SPE   & ~0.6450 & ~0.7926 & ~0.7482 &~0.7286 & ~0.3539 & ~0.8004 & ~0.7400 &~0.6314\\
& {\ours}         & ~\textbf{0.9255} & ~\textbf{0.9854} & ~\textbf{0.9672} & ~\textbf{0.9593} & ~\textbf{0.8948} & ~\textbf{0.9839} & ~\textbf{0.9435} &~\textbf{0.9407}\\
\hline 
& BLEU        & -0.0621 & -0.1442 & ~0.2944 &~0.0294 & -0.1273 & -0.2208 & ~0.1793 &-0.1758\\
Selected & PPL         & -0.0197 & -0.1775 & ~0.0460 &-0.0504 & -0.1146 & -0.4652 & -0.0404 &-0.2067\\
Agents & SPE   & ~0.0970 & ~0.5203 & ~0.4777 &~0.3650 & ~0.1368 & ~0.5304 & ~0.4943 &~0.3872\\
& {\ours}         & ~\textbf{0.8640} & ~\textbf{0.9031} & ~\textbf{0.8952} &~\textbf{0.8874} & ~\textbf{0.8496} & ~\textbf{0.9414} & ~\textbf{0.8782} &~\textbf{0.8686} \\
% \hline
% \multirow{3}{*}{M-M} 
% & BLEU        & ~0.1981 & -0.0067 & ~0.0980 & ~0.1525 & ~0.0009 & ~0.0924\\
% & ppl         & -0.1584 & -0.0610 & -0.1209 & -0.2475 & -0.1060 & -0.1178\\
% & OPE         & ~\textbf{0.9687} & ~\textbf{0.9947} & ~\textbf{0.9874} & ~\textbf{0.8800} & ~\textbf{0.9872} & ~\textbf{0.9574}\\
\hline 
\bottomrule
\end{tabular}
\end{center}
% \vspace{-0.05in}
\end{table*}

We evaluate {\ours}, BLEU/PPL \citep{papineni2002bleu} and Self-Play Evaluation (SPE) based on the correlation between estimated reward and true reward. The results are summarized in Table~\ref{tab:air_r2}. {\ours} uses the experience data of the other 23 agents to evaluate each agent (i.e., leave-one-bot-out). Note that SPE~\citep{wei2018airdialogue} needs to train a customer agent in addition to the seller agent being evaluated. For a fair comparison, we train the SPE customer agent on both experience data and policy training data (See Appendix~\ref{app:air-trans} for details). Our empirical observations are as follows:
 %For fair comparison, the Self-Play customer agent adopts the pre-trained RoBERTa encoder and is further fine-tuned on both experience data and policy training data. For automatic evaluation of each agent using {\ours}, we use the experience data of the other 23 agents (i.e., leave-one-bot-out).
 
 %As can be seen from . {\ours} significantly outperforms BLEU/PPL and Self-Play.

% \vskip-3pt
\noindent $\bullet$ {\bf {\ours} vs. BLEU/PPL}. {\ours} significantly outperforms BLEU/PPL. As mentioned earlier, BLEU and PPL are well-known metrics for evaluating language quality. For goal-oriented systems whose goal is to complete a specific task, however, BLEU and PPL scores show little correlation with task completion scores.
%\Bo{As a reader who is not familiar with NLP, I am not very clear how these two metrics are computed. For example, where we can have the reference for BLEU? Is this a common knowledge in NLP?}
%\noindent $\bullet$ {\bf {\ours} vs. BLEU/PPL}. {\ours} significantly outperforms BLEU/PPL. This is because as mentioned earlier, BLEU and PPL are well-known metrics for evaluating the language quality of dialog systems. However, for goal-oriented systems, whose goal is completing the tasks, they show negligible correlations with the final task rewards.

% \vskip-3pt
\noindent $\bullet$ {\bf {\ours} vs. SPE}. {\ours} significantly outperforms SPE. To better understand their performance, we also 
present the regression plots between estimated and true rewards in Figure~\ref{fig:air_ope_vs_selfpaly}. %As can be seen, 
Both {\ours} and SPE can easily identify agents with extremely poor rewards. However, for selected good agents whose flight score, status score, and overall reward are better than $0.5$, $0.7$, and $0.65$ respectively, SPE performs worse than {\ours} by a much larger margin (especially for flight score). Additional regression plots are shown in Appendix~\ref{app:exp_air}.

%We observe that bad agents with extremely low rewards are easily identified by both {\ours} and Self-Play, so we focus our evaluation on reasonably good agents. Specifically, we only consider agents whose flight score, status score, and overall reward are better than $0.5$, $0.7$, and $0.65$, respectively. Table~\ref{tab:air_r2} shows correlations between estimated rewards and true rewards, and regression plots are shown in Appendix~\ref{app:exp_air}. As we can see, when considering only reasonably good agents, the correlation to human evaluation is weaker for both \ours~and Self-Play. This demonstrates that it is more challenging to distinguish models with similar performance. It is worth noticing that the performance degradation is less sever for \ours.

%  \Bo{I am confused here: As we can see, both \ours~and Self-Play perform worse in this setting indicating that it is indeed a more challenging setting as expected. It's worth noticing that the performance degradation is less sever for \ours.}

% The corresponding $R^2$ is shown in Table~\ref{tab:air_r2}.  As can be seen Self-Play also correlates to human-evaluation. But OPE performs significantly better. 

\begin{figure*}[htb!]
% \vspace{-0.1in}
\begin{center}
\begin{subfigure}{0.49\textwidth}
\includegraphics[width=\textwidth]{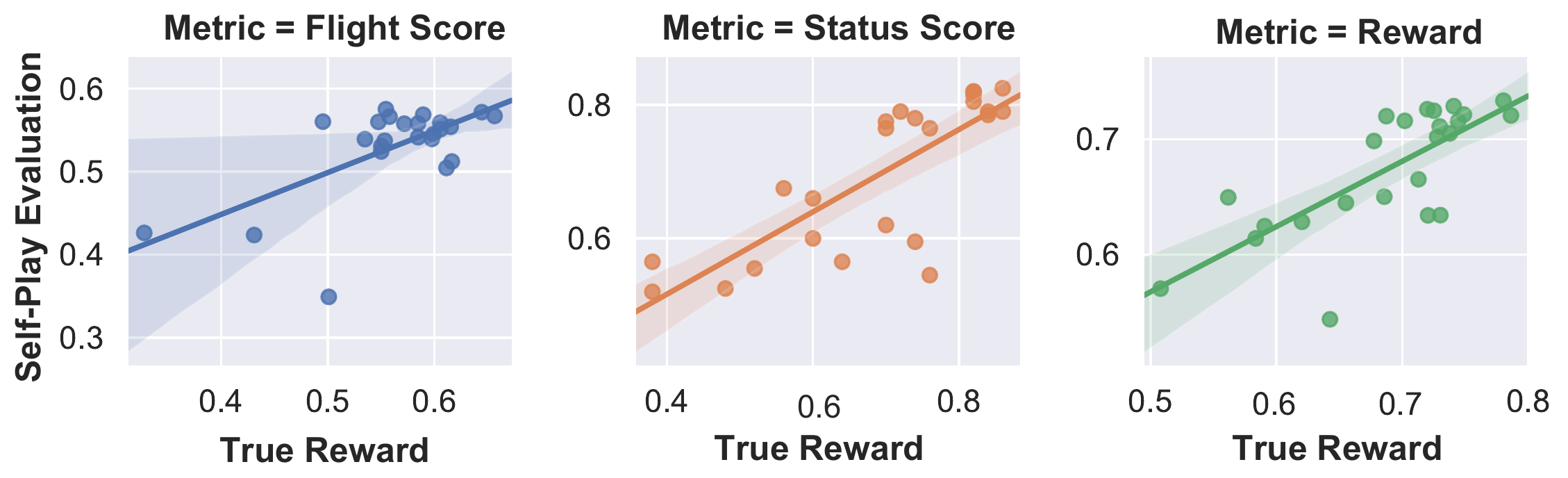}
% \vspace{-0.25in}
\caption{SPE vs. Human Evaluation}
\end{subfigure}
\begin{subfigure}{0.49\textwidth}
\includegraphics[width=\textwidth]{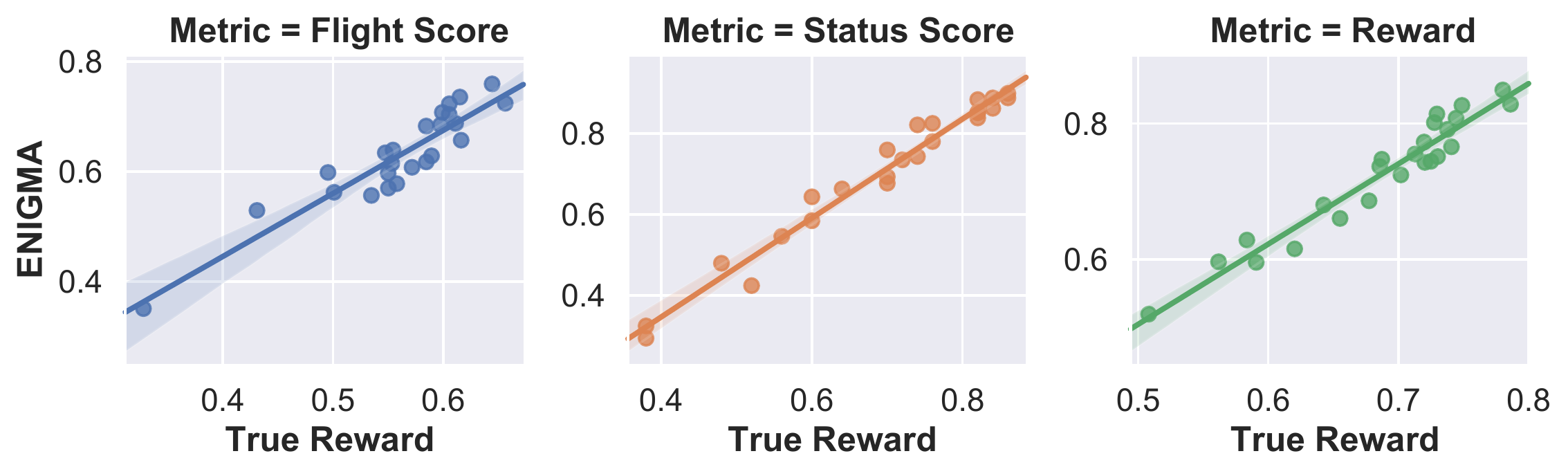}
% \vspace{-0.25in}
\caption{{\ours} vs. Human Evaluation}
\end{subfigure}
\end{center}
% \vspace{-0.175in}
\caption{Regression Plots.  The x-axis is the average reward obtained by chatting with human. The y-axis is the reward estimated by SPE / {\ours}. Different colors denote different types of rewards (flight score, status score, and overall reward).  The solid line is obtained by linear regression and the shaded region indicates $95\%$ confidence interval.}
\label{fig:air_ope_vs_selfpaly}
% \vspace{-0.225in}
\end{figure*}

% Nevertheless, they are still used as one of the main metrics. 
% we compared {\ours} with BLEU and PPl 
% Results indicated that our method strongly outperform those two metrics in terms of correlation with human evaluation.

%\vskip-2pt
%These results clearly demonstrate that \ours~correlates significantly stronger to human scores than that of the Self-Play, BLEU and PPL, and thus, can be used for large-scale goal-oriented dialog systems automatic evaluation. 

% We compare estimations using {\ours} (average of the last 100 epochs) with the actual rewards from the environments under both the Model-Model and the Model-Human settings. Results are illustrated in Figure~\ref{fig:air_ope_vs_auto}. We see that our method correlates linearly with the actual rewards for both of the settings while the traditional metrics such as BLEU score or PPL  

% In addition, we compare OPE estimation and traditional automatic metrics including BLEU scores and perplexity. The scatter plots are shown in Figure~\ref{fig:air_ope_vs_auto}. We also present the corresponding $R^2$ in Table~\ref{tab:air_r2}. As can be seen, the proposed OPE method strongly correlates to the rewards obtained by chatting with environments ($R^2 \geq 0.9$). 

\begin{figure}[htb!]
	%\vspace{-0.15in}
	\begin{center}
		\includegraphics[width=0.6\textwidth]{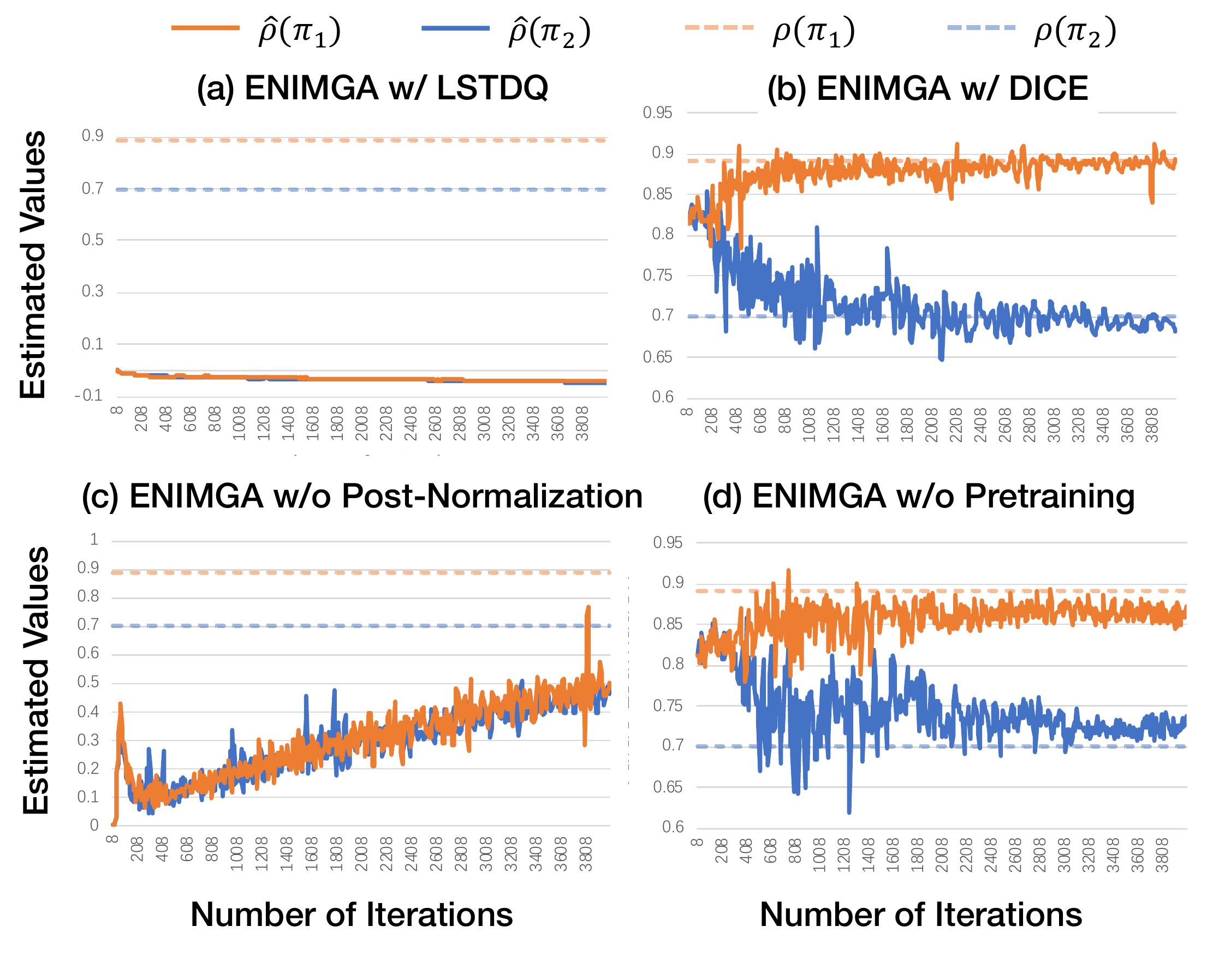}
	\end{center}
% 	\vspace{-0.2in}
	\caption{Value estimation using different methods for two target agents ($\pi_1$ and $\pi_2$) vs. \# of iterations. Dotted lines denote the true rewards.}
% 	\vspace{-0.15in}
	\label{fig:prelim_exp}
\end{figure} 

\begin{figure}[!htb]
%  	\vspace{-0.1in}
	\begin{center}
		\includegraphics[width=0.6\textwidth]{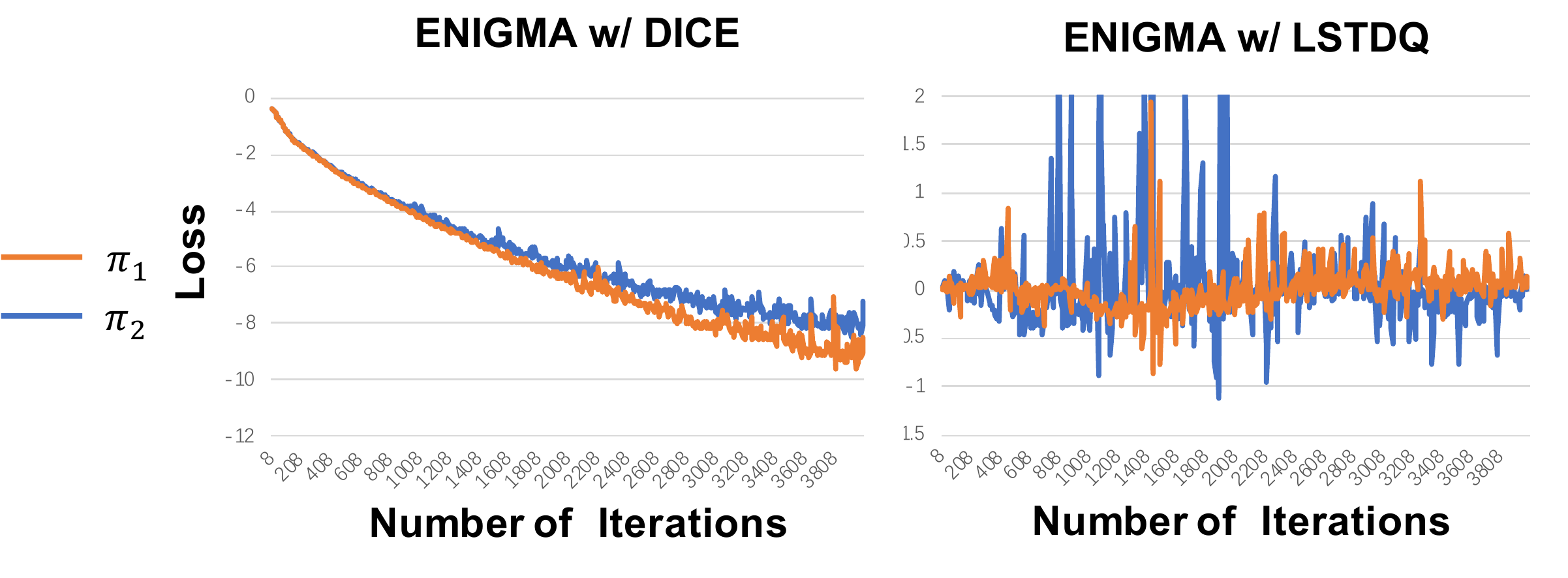}
	\end{center}
% 	\vspace{-0.25in}
	\caption{Training Objectives vs. Number of Iterations for two target agents.}
% 	\vspace{-0.3in}
	\label{fig:prelim_loss_exp}
\end{figure} 

% \vskip-3pt
\noindent $\bullet$ {\bf Ablation Study}. We select 2 out of the 24 agents to illustrate the importance of each component in {\ours}.

% \vskip-5pt
\noindent $\star$ {\bf \emph{DICE vs. LSTDQ}}. Figure~\ref{fig:prelim_exp}(a) and Figure~\ref{fig:prelim_exp}(b) show the estimated values of LSTDQ (only fitting the $Q$-function) and DICE respectively: estimates of LSTDQ are stuck at 0 whereas estimates of DICE approach the true rewards (dotted lines) as training progresses. Figure~\ref{fig:prelim_loss_exp} additionally shows that the training objectives of LSTDQ oscillates as DICE stably converges.
%Figure~\ref{fig:prelim_exp}(a) shows the performance of LSTDQ (only fitting $Q$-function) --- the estimated values are stuck at 0. In contrast, Figure~\ref{fig:prelim_exp}(b) shows that the values estimated by the DICE objective approach the true rewards (dotted lines) as training progresses. We further investigate the training objectives of DICE and LSTDQ in Figure \ref{fig:prelim_loss_exp}, and find that LSTDQ's estimates are oscillating, whereas DICE's estimates are more stable.

% \vskip-5pt
\noindent $\star$ {\bf \emph{Post-normalization}}. Figure~\ref{fig:prelim_exp}(c) shows the performance of {\ours} without post-normalization: The estimated values fail to approach the true rewards.

% \vskip-5pt
\noindent $\star$ {\bf \emph{Pretrained Encoder}}. Figure~\ref{fig:prelim_exp}(d) shows the performance of {\ours} without the pretrained encoder: The estimated values can approach the true rewards, but are less stable and less accurate than the counterpart with the pretrained encoder.

%, i.e., DICE estimator vs. LSTDQ, post normalization and pre-trained language models, for evaluating two different target agents. As demonstrated in Figure~\ref{fig:prelim_exp}(a), fitting $Q$-function in LSTDQ is quite difficult and the resulted estimator is not accurate; while \ours~converges to policy value accurately in Figure~\ref{fig:prelim_exp}(b). The importance of post normalization is empirically justified in Figure~\ref{fig:prelim_exp}(c). Without the post normalization, the reward estimation does not converge. By using pretrained encoder, the estimator is more stable and more accurate as it provides language representation with rich semantic knowledge and makes estimation easier as shown in~Figure~\ref{fig:prelim_exp}(d). We provide the embedding visualization analysis in Appendix~\ref{app:embed_vis}, which further illustrates the role of pretrained language embedding. 
%We further contrast the training loss of DICE estimator and LSTDQ estimator in Figure~\ref{fig:prelim_loss_exp} for demonstrating the stability of DICE learning. 
% As we can see, LSTDQ failed to estimate policy. 
% We further examine their optimization loss during mini-max optimization in Figure~\ref{fig:prelim_loss_exp}. 
%As we can see, the loss in LSTDQ keep oscillating. 
% We observe optimization difficulty in using LSTDQ, where the loss keep oscillating around 0. 
%On the other hand, DICE is stably decreasing. 
% in the early stage but become unstable in the later stage, where we use learning rate decay to prevent algorithm from crash in practice.

\begin{figure*}[!htb]
% \vspace{-0.1in}
\begin{center}
\begin{subfigure}{0.48\textwidth}
\includegraphics[width=\textwidth]{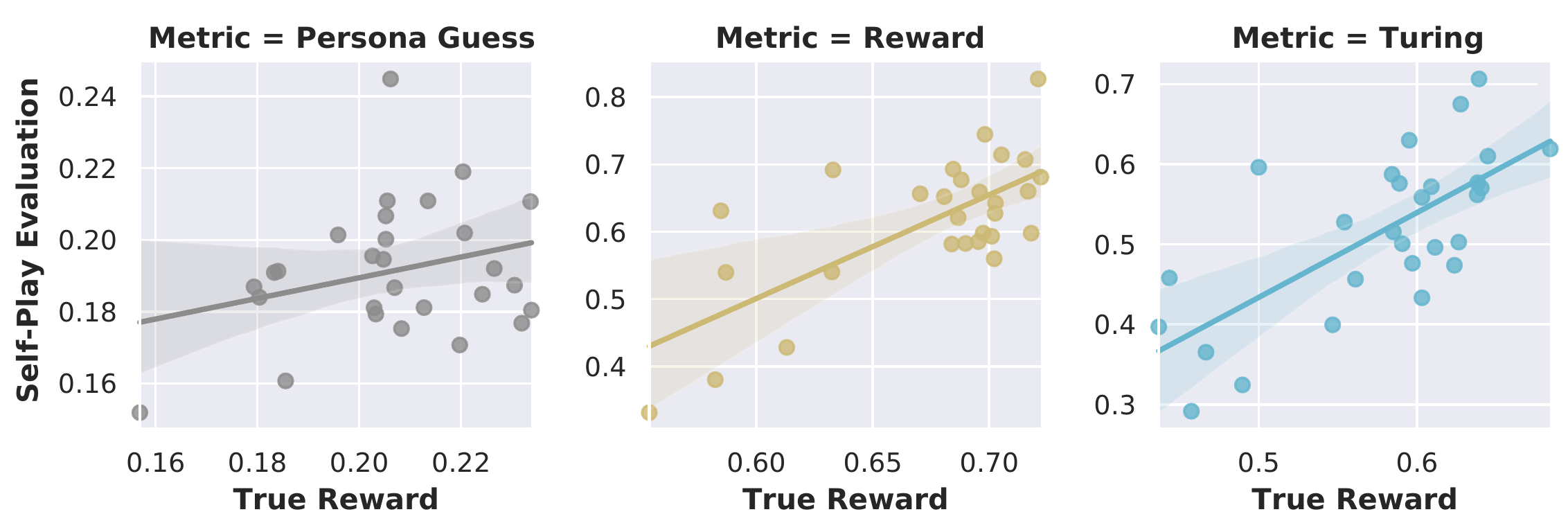}
% \vspace{-0.2in}
\caption{SPE vs. Human Evaluation}
\end{subfigure}
\begin{subfigure}{0.48\textwidth}
\includegraphics[width=\textwidth]{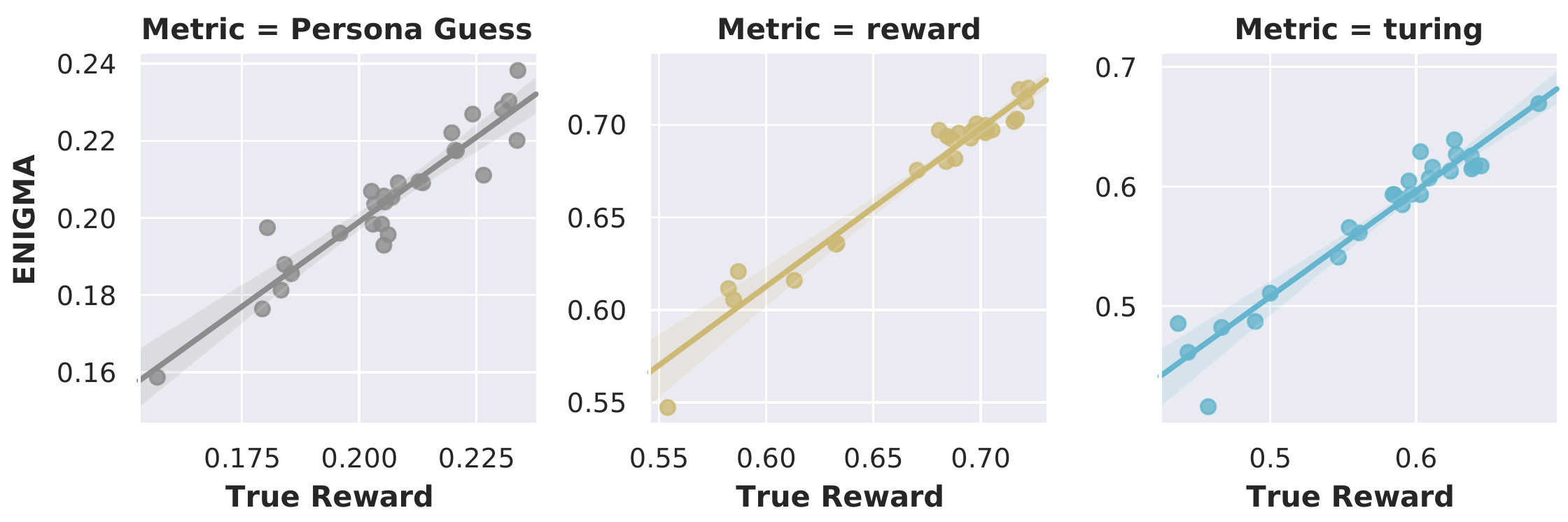}
% \vspace{-0.2in}
\caption{{\ours} vs. Human Evaluation}
\end{subfigure}
\end{center}
% \vspace{-0.175in}
\caption{Regression Plots. Only three metrics are presented. Please refer to Appendix~\ref{app:exp_convai2} for all plots.}
\label{fig:convai2_opevsselfplay}
%\vspace{-0.025in}
\end{figure*}

% \vspace{-0.125in}
\subsection{Open-Domain Chit-chat Systems}
% \vspace{-0.075in}

We now test \ours~ for evaluating open-domain chit-chat dialog systems.

% \vskip-3pt
$\bullet$ \textbf{Policy Training Data}. We use 29 pre-trained agents\footnote{\url{https://github.com/facebookresearch/ParlAI/tree/master/projects/controllable_dialogue}} provided by \citet{see2019what}. These agents are trained using behavior cloning on the {\bf \emph{ConvAI2}} dataset\footnote{\url{https://parl.ai/projects/convai2/}} \cite{zhang2018personalizing,dinan2020second}. The dataset contains 8,939 pieces of dialog, where participants are instructed to chat naturally using given personas.

% \vskip-3pt
$\bullet$ \textbf{Experience Data}. We use the experience dataset provided by \citet{see2019what}. The dataset contains 3,316 agent-human evaluation logs and 10 different language quality metrics for each log.

\begin{table*}[htb!]
% \vspace{-0.15in}
\scriptsize
\caption{The correlation between automatic metrics and language score obtained by interacting with human. We only present the average/min/max correlations to all 10 different language metrics in this table. For detailed numbers, please refer to Appendix~\ref{app:exp_convai2}, Figure~\ref{fig:heatmap_convai2}.}
\label{tab:enigma_r2}
% \vspace{-0.175in}
\begin{center}
\begin{tabular}{c|c|ccc|ccc}
\toprule \hline 
\multicolumn{1}{c|}{\multirow{2}{*}{\bf Method}} &\multicolumn{1}{c|}{\multirow{2}{*}{\bf Experience Data}}& \multicolumn{3}{c|}{\bf Pearson Correlation}& \multicolumn{3}{c}{\bf Spearman's Rank Correlation } \\
\cline{3-8}
 &  &\multicolumn{1}{c}{\bf Average} &\multicolumn{1}{c}{\bf Min } &\multicolumn{1}{c|}{\bf Max } &\multicolumn{1}{c}{\bf Average } &\multicolumn{1}{c}{\bf Min } &\multicolumn{1}{c}{\bf Max } 
\\ \hline 
Best of 8 HCDFs & Human-Human & ~0.6045&	0.4468&	0.9352&	0.3384&	0.1724&	0.7526\\
8 HCDFs + Regression & Human-Human & ~0.5387&	-0.0348&	0.7519&	0.4740&	0.2784&	0.7880\\
BLEU & Human-Human & 0.4127 &	0.0671 &	0.6785&	0.2965 &	0.0482 &	0.7236\\
BLEURT & Human-Human & 0.4513	& 0.1557	& 0.6572&	0.4389	&0.0055	&0.6864\\
BERTscore F-1 & Human-Human & 0.5365	&0.0609	&0.8385&	0.5293	&0.2852	&0.7044\\
SPE & Human-Model& ~0.5907&	0.0962&	0.8820&	0.4350&	0.1363&	0.6405 \\
SPE & Human-Model (Challenging) &~0.3559&	-0.1679&	0.6900&	0.1429&	-0.0777&	0.3216 \\
{\ours} & Human-Model & {\bf 0.9666}&	{\bf 0.9415}&	{\bf 0.9792}&	{\bf 0.9167}&	{\bf 0.8717}&	{\bf 0.9485} \\
{\ours} & Human-Model (50\% data)                              &~0.9126&	0.8506&	0.9585&	0.7790&	0.6651&	0.8647 \\
{\ours} & Human-Model (10\% data)                               &~0.7327&	0.4544&	0.9266&	0.5214&	0.3651&	0.6492 \\
{\ours} & Human-Model (Challenging)                        &~0.6505&	0.5394&	0.7762&	0.5190&	0.3168&	0.6672 \\
\hline 
\bottomrule
\end{tabular}
\end{center}
% \vspace{-0.3in}
\end{table*}

We follow the setups from Section 4.2 to evaluate, {\ours}, SPE, BLEU, BLEURT \citep{sellam2020bleurt}, BERTscore \citep{zhang2019bertscore}, and 8 Hand-Crafted Dialog Features (HCDFs) based on Pearson and Spearman's rank correlations between the estimated rewards and the true rewards. Here the true rewards are human evaluation scores under 10 different language quality metrics. More details of HCDFs and language quality metrics can be found in \citet{see2019what}. The average, minimum and maximum of the 10 correlations (under different language quality metrics) of each method are summarized in Table \ref{tab:enigma_r2}. Moreover, we also consider using 8 HCDFs to fit the true rewards using linear regression, and the results are also included in Table \ref{tab:enigma_r2}.

%\Bo{I am lost here: why the \ours need the extra reward prediction? Does not the experience data already have it for each dialog?}
Note that since we are considering a chit-chat dialog system, SPE does not train an additional agent but asks two identical target agents to chat with each other. However, SPE needs to train an additional model to predict the reward of each dialog. Specifically, we fine-tune the pre-trained RoBERTa encoder with an output layer over the experience data (an additional sigmoid function is applied to ensure an output between 0 and 1). For automatic evaluation of each agent using {\ours}, we use the experience data of the other 28 agents (i.e., leave-one-bot-out).

% \vskip-3pt
\noindent $\bullet$ {\bf {\ours} vs. SPE/BLEU/BLEURT/BERTscore/HCDFs}. {\ours} significantly outperforms SPE, BLEU, BLEURT, BERTscore and HCDFs in both Pearson and Spearman's rank correlations. Moreover, we compare the correlations between estimated rewards and human evaluation scores under each language quality metric. Due to space limit, we only show the plots of {\ours} and SPE under 3 out of 10 language quality metrics in Figure~\ref{fig:convai2_opevsselfplay}. Additional plots and detailed results can be found in Appendix~\ref{app:exp_convai2}. We see that {\ours} outperforms SPE and HCDFs under all language equality metrics.

% \vskip-3pt
\noindent $\bullet$ {\bf Sample Efficiency of {\ours}}. To demonstrate that {\ours} is sample efficient, we test {\ours} on randomly sub-sampled (10\% and 50\%) experience data. We found that even using only 10\% of the experience data, {\ours} still outperforms SPE and HCDFs.

% \vskip-3pt
\noindent $\bullet$ {\bf Evaluation under Challenging Experience Data}. To make the evaluation more challenging, we further test {\ours} by excluding the experience data obtained by the behavior policies similar to the target policy (see more details in Appendix~\ref{app:exp_convai2}). We see that even with such challenging experience data, {\ours} still outperforms SPE with trained on full data and HCDFs under almost all language quality metrics.

\section{Discussions}
\label{sec:discussions}
% \vspace{-0.075in}

% discuss rl in dialog, criticize eval is not good and training is not good. 
% discuss off-policy learning

%The existing research on automatic evaluation of dialog system can be categorized into static evaluation vs. dynamic evaluation. Most of the research fall into static evaluation which is evaluating language quality of single-turn response or task-completion score of fixed complete dialog; while few literature emphasizes on dynamic property in the interactive environment in the evaluation. We remark in both static and dynamic evaluation, the algorithms relies on the assumption that data have sufficient coverage explicitly or implicitly. 
Existing research on automatic evaluation of dialog systems can be categorized into static vs. dynamic evaluation. Most of existing research falls into static evaluation with a focus on language quality of single-turn response or on task-completion given fixed dialog, while few literature emphasizes dynamic properties of an interactive environment. Different from static evaluation, dynamic evaluation considers the sequential interaction between a human and an agent, and thus it is more challenging.
%, where reinforcement learning is a natural choice~\citep{ghandeharioun2019approximating}.

We note that in both static and dynamic evaluations, algorithms rely on the assumption of sufficient data coverage (explicitly or implicitly) to ensure reliable evaluation.
For example, in static evaluation, BLEU score requires all reasonably good responses to be exactly covered by the experience data. More recently, \citet{lowe2017towards} show that their method only works when the behavior policies include the target policy.
% ; similarly, \citet{tao2017ruber} achieve decent performances on single-turn language quality evaluation, which relies on the alignment of single-turn responses in the latent space. 
Dynamic evaluation also assumes the sufficient coverage. We emphasize that it is the information-theoretic limit of all OPE methods~\citep{jiang2016doubly}, which requires the experience data to cover sufficient target policy behaviors to ensure accurate estimation. Therefore, we 
suggest the broader research community to release human-model interaction evaluation data to further promote research in automatic dialog systems evaluation.

% \vspace{-0.15in}
\section{Conclusion}
\label{sec:conclusion}
% \vspace{-0.075in}
%We develop {\ours}, a model-free reinforcement learning framework, for evaluating dialog systems using experience data, by adopting the current state-of-the-art off-policy evaluation methods in reinforcement learning to the need of dialog systems. Compared with existing language-quality-metric-based methods and model-based reinforcement learning methods, {\ours} naturally takes the interactive and dynamic nature of the conversation into consideration, and alleviates from the difficulty of modeling complex human conversational behaviors. Thorough experiments demonstrate that {\ours} significantly outperforms existing methods in terms of the correlation with human evaluation scores. A potential future direction is to extend {\ours} from off-policy evaluation to off-policy learning, which aims to learn a conversational agent based on the experience data~\citep{nachum2019algaedice,kallus2020statistically}.

We develop a model-free OPE framework, {\ours}, for evaluating dialog systems using experience data. By adopting the current state-of-the-art OPE method in reinforcement learning, we tackle several challenges in modeling dialog systems. Different from existing single-turn language quality metrics and model-based reinforcement learning methods, {\ours} naturally takes into consideration the interactive and dynamic nature of conversations, while avoiding the difficulty of modeling complex human conversational behaviors. Our thorough experimental results demonstrate that {\ours} significantly outperforms existing methods in terms of correlation with human evaluation scores. One potential future direction is to extend {\ours} from off-policy evaluation to off-policy improvement, which aims to learn a dialog system based on experience data~\citep{nachum2019algaedice,kallus2020statistically}.

\bibliographystyle{ims}
\bibliography{dialogue_ope}

\clearpage

%!TEX root = DialogOPE.tex

\appendix
\section{{\ours}}
\label{app:dialog-dice}

\subsection{Supporting Theorem}
\label{app:thm}

\setcounter{theorem}{0}
\begin{theorem}
	The augmented MDP with infinite horizon satisfies the following properties:
	\begin{itemize}
		\item It has a unique stationary state-action visitation distribution $d^\pi(s,a)$;
		\item For the station-action pair $(s_t,a_t)$ in a conversation $h$ with padded pseudo states, we have
        \begin{align}
        d^\pi(s_t,a_t) = \frac{1}{T_{\rm max}}\sum_{\{(s_k,a_k)\}_{k=1}^{t-1}}[ \mu_0(s_1)\pi(a_1|s_1)P(s_2|a_1,s_1)\cdots P(s_t|a_{t-1},s_{t-1})\pi(a_t|s_t)],
        \label{eq:stationary_dist_app}
        \end{align}
        where $\{(s_k,a_k)\}_{k=1}^{t-1}$ are the state-action pairs in the same conversation as $(s_t,a_t)$; 
		\item The policy value can be computed by sampling from $d^\pi(s,a)$, and we have
        \begin{align}
		    \rho_{A}(\pi) = {\EE}_{(s,a) \sim d^\pi(s,a)} [R(s,a)] = \rho(\pi)/T_{\rm max}.
	    \label{eq:aug-dialog-value-2_app}
        \end{align}
	\end{itemize}
\end{theorem}
\begin{proof} 

% \newline

\textbf{First}, we prove that the augmented MDP  has a unique stationary state-action visitation distribution shown in \eqref{eq:stationary_dist_app}.

    As the augmented MDP is periodic with period $T_{\rm max}$, the uniqueness and stationary distribution can not be immediately obtained by ergodicity of the MDP (the first two points of the Theorem).
    
    To obtain the stationary state-action visitation distribution, we essentially need to solve the following equations:
    \begin{align}
        d^\pi(s,a) = \sum_{(s',a')} d^\pi(s',a')P(s|s',a')\pi(a|s),~~{\rm for~all}~~ (s,a)
        \label{eq:stationary_transition}
    \end{align}
    with $d^\pi(s,a)$ is  a probability measure on the state-action space, i.e., $\sum_{(s,a)} d^\pi(s,a)=1$.
    
    We first group the state-action pairs by their dialog turns $t$. More specifically, we define $\cS_t := \{s_t: s_t {\rm ~contains~} t {\rm ~dialog~turns}\}$, $\cA_t := \{a_t: a_t {\rm ~is~the~response~at~the~} t {\rm -th~dialog~turn}\}$ and $ \cQ_t =  \cS_{t} \times \cA_{t}$. We have the state space is the direct sum of state groups $\cS_0\oplus\cS_1\cdots \oplus \cS_{T_{\rm max}}=\cS$ and the action space is the union of all action groups $\bigcup_{t=1}^{T_{\rm max}}\cA_t=\cA$. We further have $\cQ_0\oplus\cQ_1\cdots \oplus \cQ_{T_{\rm max}}=\cS\times\cA=\cQ$. Notice that $t$ is the number of dialog turns in original MDP, not the time step for the augmented MDP. 
     
     We then consider the quantity $S_{t} = \sum_{(s_t,a_t) \in \cQ_t} d^\pi(s_{t},a_{t}) $, which sum over the LHS of the Eq.\eqref{eq:stationary_transition} for each group of $(s_t,a_t)$. We now expand the corresponding sum of the RHS of \eqref{eq:stationary_transition}:
     \begin{align*}
         &S_t = \sum_{(s_t,a_t) \in \cQ_t}d^\pi(s_{t},a_{t}) \\
         &= \sum_{(s_t,a_t) \in \cQ_t }\sum_{(s_{t-1},a_{t-1})  \in \cQ_{t-1}} d^\pi(s_{t-1},a_{t-1})P(s_{t}|s_{t-1},a_{t-1})\pi(a_{t}|s_{t}) \\
         &= \sum_{(s_{t-1},a_{t-1})  \in \cQ_{t-1}}\sum_{(s_t,a_t)\in \cQ_t} d^\pi(s_{t-1},a_{t-1})P(s_{t}|s_{t-1},a_{t-1})\pi(a_{t}|s_{t}) \\
         &= \sum_{(s_{t-1},a_{t-1})  \in \cQ_{t-1} } d^\pi(s_{t-1},a_{t-1}) \\
         &= S_{t-1} ~~(t>1).
     \end{align*}
     We have $S_1=S_2=\cdots=S_{T_{\rm max}}=S$. As $d^\pi(s,a)$ is a probability measure, we have the following unique solution for $S_t$'s
     \begin{align*} 
         S =  \frac{1}{T_{\rm max}} \sum_{t=1}^{T_{\rm max}} S_t = \frac{1}{T_{\rm max}} \sum_{t=1}^{T_{\rm max}} \sum_{(s_t,a_t)\in \cQ_t} d^\pi(s_{t},a_{t}) =  \frac{1}{T_{\rm max}} \sum_{(s,a)\in \cQ} d^\pi(s,a) = \frac{1}{T_{\rm max}} .
     \end{align*}
     As there is only one possibility for the last state-action pairs in all possible conversation that $s_{T_{\rm max}}={\rm Pad}_{T_{\rm max}},a_{T_{\rm max}}={\rm NextPad}$,  we have $d^\pi({\rm Pad}_{T_{\rm max}},{\rm NextPad}) = S_{T_{\rm max}} = \frac{1}{T_{\rm max}}$. We now consider $(s_1,a_1)$, which is the first state-action pair of a conversation. We have 
     \begin{align}
         &d^\pi(s_1,a_1) = \sum_{(s',a')} d^\pi(s',a')P(s_1|s',a')\pi(a_1|s_1) \notag\\
         &=d^\pi({\rm Pad}_{T_{\rm max}},{\rm NextPad}) P(s_1|{\rm Pad}_{T_{\rm max}},{\rm NextPad})\pi(a_1|s_1) = \frac{1}{T_{\rm max}} \mu_0(s_1)\pi(a_1|s_1),
         \label{eq:app-proof-init}
     \end{align}
     which is the unique solution. For any $(s_t,a_t)\in\cQ_t (t>1)$, the previous state-action pairs in the same conversation must be in $\cQ_{t-1}$. We have 
     \begin{align}
         &d^\pi(s_t,a_t) = \sum_{(s',a')} d^\pi(s',a')P(s_t|s',a')\pi(a_t|s_t) \notag\\
         &= \sum_{(s_{t-1},a_{t-1}) \in \cQ_{t-1}} d^\pi(s_{t-1},a_{t-1}) P(s_t|s_{t-1},a_{t-1})\pi(a_{t}|s_{t}).
         \label{eq:app-proof-induce}
     \end{align}
     
     Based on \eqref{eq:app-proof-init} and \eqref{eq:app-proof-induce}, we can obtain the unique solution for \eqref{eq:stationary_transition}:
    \begin{align*}
    &d^\pi(s_t,a_t) \\
    &= \sum_{(s_{t-1},a_{t-1}) \in \cQ_{t-1}} d^\pi(s_{t-1},a_{t-1}) P(s_t|s_{t-1},a_{t-1})\pi(a_{t-1}|s_{t-1}) \\
    & = \sum_{(s_{t-1},a_{t-1})  \in \cQ_{t-1}} \sum_{(s_{t-2},a_{t-2}) \in \cQ_{t-2}} d^\pi(s_{t-2},a_{t-2}) P(s_{t-1}|s_{t-2},a_{t-2})\pi(a_{t-1}|s_{t-1})  P(s_t|s_{t-1},a_{t-1})\pi(a_{t}|s_{t}) \\
    & \cdots \\
    & = \frac{1}{T_{\rm max}}\sum_{\{(s_k,a_k)\}_{k=1}^{t-1}}[ \mu_0(s_1)\pi(a_1|s_1)P(s_2|a_1,s_1)\cdots P(s_t|a_{t-1},s_{t-1})\pi(a_t|s_t)],
    \end{align*}
    where we omit the constraint of $\cQ_t$ as the transition kernel $P$ naturally satisfies the constraints.  
     
     Till now, we have shown that the augmented MDP  has a unique stationary state-action visitation distribution shown in \eqref{eq:stationary_dist_app} (the first two points of the Theorem).
     
     \textbf{Next}, we show that the policy value of the policy $\pi$ under the augmented MDP is proportional to its counterpart under the original MDP without the augmentation (the third point of the Theorem). 
     
    Recall that the expected reward of original MDP \eqref{eq:dialog-value} is defined as
    \begin{align*} 
	&\rho(\pi) = \EE_{h \sim \mu_0,\pi,\cE} [R(s_T,a_T)] = \sum_{T=1}^{T_{\rm max}}\sum_{h} Pr(h, h{\rm ~has~}T{\rm~turns})R(s_T,a_T)  \\
	&= \sum_{T=1}^{T_{\rm max}}\sum_{\{(s_k,a_k)\}_{k=1}^{T-1}} \mu_0(s_1)\pi(a_1|s_1)P(s_2|a_1,s_1)\cdots P(s_T|a_{T-1},s_{T-1})\pi(a_T|s_T) \\
	& \times \ind(a_T {\rm~ End~Conversation}) R(s_T,a_T),
    \end{align*}
    where $T$ is the number of turns in the original dialog before padding. Recall that, the MDP only obtain non-zero reward when the dialog ends, (i.e., when $a {\rm~ End~Conversation}$). On the other hand, 
    Due to the existence of unique stationary distribution, the policy value of $\pi$ for the augmented MDP \eqref{eq:aug-dialog-value} can written as:
    \begin{align*}
    	\textstyle\rho_A(\pi) & = {\EE}_{(s,a) \sim d^\pi(s,a)} [R(s,a)]  = {\EE}_{(s,a) \sim d^\pi(s,a)} [\ind(a {\rm~ End~Conversation})R(s,a)] \\
    	& = \frac{1}{T_{\rm max}} \sum_{t=1}^{T_{\rm max}} \sum_{\{(s_k,a_k)\}_{k=1}^t}  \mu_0(s_1)\pi(a_1|s_1)P(s_2|a_1,s_1)\cdots P(s_t|a_{t-1},s_{t-1})\pi(a_t|s_t) \\
    	& \times \ind(a_t {\rm~ End~Conversation})R(s_t,a_t) \\
    	& = \frac{1}{T_{\rm max}}\rho(\pi).
    \end{align*} 
	
\end{proof}

\textbf{Can we directly apply infinite-horizon augmentation without padding?} 
The answer is \textit{NO}. Here we use an example to illustrate the difference between $\rho_{\rm A}$ and $\rho$ and why we need to pad every dialog to have the same length for using OPE:

\begin{example}
Suppose you have two experience dialogs $a_0 \rightarrow \cdots \rightarrow a_{t_1}$ and $b_0 \rightarrow \cdots \rightarrow b_{t_2}$ with rewards $0$ and $1$ respectively. 
For the target policy, dialogs has per-episode density $0.2$ and $0.8$ respectively. The true value of such policy is $0\times 0.2 + 1 \times 0.8 = 0.8$. The corresponding per-state density of $[a_0,\cdots,a_{t_1}]$ is $\frac{0.2}{0.2 \times t_1 + 0.8 \times t_2}$ and the one for $[b_0,\cdots,b_{t_1}]$ is $\frac{0.8}{0.2 \times t_1 + 0.8 \times t_2}$. 
The value in the new augmented MDP is $\frac{0.2 * 0 + 0.8*1}{0.2 \times t_1 + 0.8 \times t_2}$, which depends on the dialog turns and can not be directly turned into policy value in the original MDP.
\end{example} 

\subsection{ Function Approximation with Pre-Trained Language Models }
\label{app:bert}

We can compute all state-action pairs for the same dialog in a parallel way as shown in Figure~\ref{fig:bert_par}. The input to the RoBERTa encoder consists of three parts, word tokens, position ids, and token types. 

\textit{Notation}: an experience dialog $h=\{e_0,a_1,e_1,...,a_T\}$, and the corresponding response generated by the target policy $\pi$, $\{a'_t = \pi(s_t)\}_{t=1}^T$. 

\textbf{Word Tokens}. The input token is the concatenation of responses $\{e_0,a_1,a'_1,e_1,...,e_{T-1},a_T,a'_T\}$.

\textbf{Position Ids}. The position ids is separately calculated for each response. For 
$e_i$, the position ids is from $l_{2i}=\sum_{j<i} {\rm len}(e_j)+ \sum_{j\leq i}{\rm len}(a_j)$ to $l_{2i+1} = l_{2i} + {\rm len}(e_i)$, where ${\rm len}(\cdot)$ denotes the number of tokens of a given response. For $a_i$, the position ids is from $l_{2i-1}$ to $l_{2i}$. For $a'_i$, the position ids is from $l_{2i-1}$ to $l'_{2i} = l_{2i-1}+{\rm len}(a'_i)$. 

\textbf{Token types}. For $e_i$'s, the token types are $0$ which
 denotes human responses. For $a_i$'s and $a'_i$'s, the token types are $1$ which denotes agent responses. 
 
\textbf{Attention Masking}. We need to modify the attention masks to prevent tokens from attending future responses. Specifically, the attention masks make sure:
\begin{enumerate}
    \item The tokens in each response can be mutually attended;
    \item $e_i$ attends to $\{e_0,a_1,e_1,a_2,...,e_{i-1},a_i\}$;
    \item $a_i$ attends to $\{e_0,a_1,e_1,a_2,...,a_{i-1},e_{i-1}\}$;
    \item $a'_i$ attends to $\{e_0,a_1,e_1,a_2,...,a_{i-1},e_{i-1}\}$;
\end{enumerate}

\begin{figure}[!htb]
  \centering
  \includegraphics[width=0.8\textwidth]{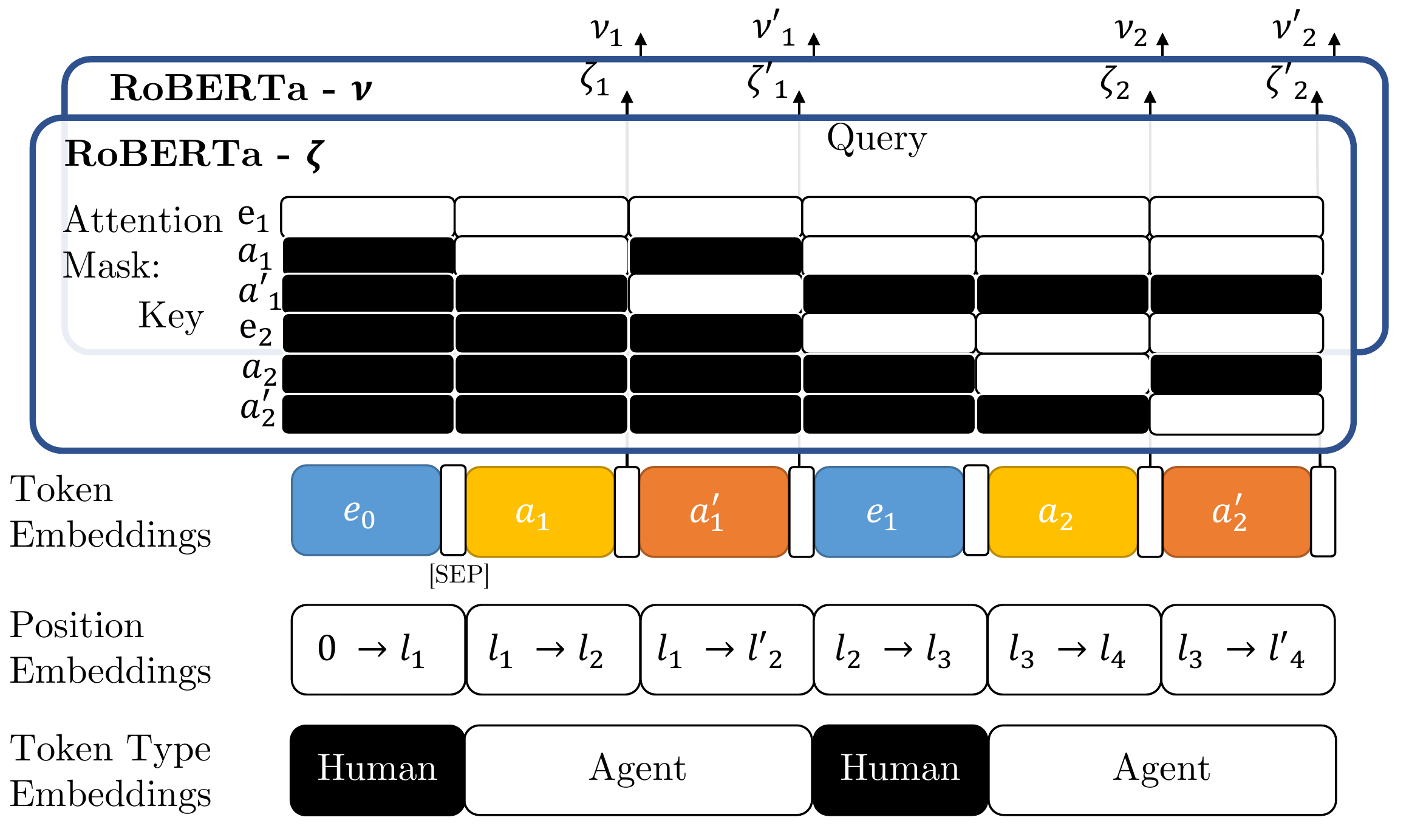}
  \caption{RoBERTa-$\zeta$ and RoBERTa-$\nu$}
  \label{fig:bert_par}
\end{figure}

% \clearpage
% \vspace{-0.2in}
\subsection{{\ours} with regularized DICE }
\label{app:algo}
% \vspace{-0.2in}

\begin{algorithm}[htb]
	\caption{Dialog OPE using regularized DICE}\label{algo:main}
	\begin{algorithmic}[1]
		\INPUT Experience Dialog with rewards $\cD = \{(h_i= \{e^{(i)}_0,a^{(i)}_1,e^{(i)}_1,...,a^{(i)}_{T_i} \}, r^{(i)})\}_{i=1}^{N}$, Target Policy $\pi$, Padding Length $T_{\rm max}$, Regularization function $f$, DICE hyper-parameters $\alpha_\zeta$, $\alpha_R$ 
		\OUTPUT Performance Estimation $\hat \rho_{n}(\pi)$
		\PARAMETER $\zeta = \{$ RoBERTa-$\zeta, [\zeta_{{\rm pad},t}]_{1\leq t\leq T_{\rm max}}\}$,$\nu = \{$ RoBERTa-$\nu, [\nu_{{\rm pad},t}]_{1\leq t\leq T_{\rm max}}\}$, $\lambda$
		\item[\textbf{\textit{Generate OPE Data}}]
		\For{ $({h}_i, r^{(i)}) $ in $ \cD$}
		\For{$t$ in $1, \cdots, T_i$}
		\State $\tilde{a}_t^{(i)} \sim \pi(\{e^{(i)}_0,a^{(i)}_1,e^{(i)}_1,...,e^{(i)}_{t-1} \})$ // Sample Action From Target Policy
		\EndFor
		\EndFor
		\State $\tilde{\cD} = \{(\tilde{h}_i=e^{(i)}_0,a^{(i)}_1,e^{(i)}_1,...,a^{(i)}_{T_i} \}, r^{(i)}) \}_{i=1}^{N}$
		\item[\textbf{\textit{Estimate $\zeta$ by Regularized DICE}}]
		\While{ Not Converged }
		\State Sample Mini-Batch $\cB \subset \tilde{\cD}$
		\For{ $(\tilde{h}_i, r^{(i)})$ in $\cB$}
		\State $\zeta^{(i)}_{0} = \zeta_{{\rm pad},T_{\rm max}}$,~~ $\nu^{(i)}_{0} = \nu'^{(i)}_{0} = \nu_{{\rm pad},T_{\rm max}}$ // infinite-horizon concatenation
		\For{$t$ in $1, \cdots, T_i$}
		\State $\zeta_{t}^{(i)} = $ RoBERTa-$\zeta(e_0^{(i)},a_1^{(i)},...,a_{t-1}^{(i)},e_{t-1}^{(i)},a_{t}^{(i)})$
		\State $\nu_{t}^{(i)} = $ RoBERTa-$\nu(e_0^{(i)},a_1^{(i)},...,a_{t-1}^{(i)},e_{t-1}^{(i)},a_{t}^{(i)})$
		\State $\nu'^{(i)}_{t} = $ RoBERTa-$\nu(e_0^{(i)},a_1^{(i)},...,a_{t-1}^{(i)},e_{t-1}^{(i)},{\color{red}\tilde{a}_{t}^{(i)}})$
		\EndFor
		\For{$t$ in $T_i + 1, \cdots, T_{\rm max}$}
		\State $\zeta_{t}^{(i)} = \zeta_{{\rm pad},t}$,~~ $\nu_{t}^{(i)} = \nu'^{(i)}_{t} = \nu_{{\rm pad},t}$ // add padding state
		\EndFor
		\State $\ell_i = \frac{1}{T_{\rm max}} [\alpha_R \zeta^{(i)}_{T_i}r^{(i)}_{T_i} + \sum_{t=0}^{ T_{\rm max} -1 } [\zeta_t (\nu'^{(i)}_{t+1} - \nu^{(i)}_t ) + \lambda (\zeta^{(i)}_t-1)  -\alpha_\zeta f(\zeta^{(i)}_t) ]]$
		\EndFor
		\State $L_D(\zeta, \nu, \lambda) = \frac{1}{|\cB|}\sum_{i \in \cB} \ell_i$ 
		\State SGD update based on $\frac{\partial L_D}{\partial \nu}$,$\frac{\partial L_D}{\partial \lambda}$,$\frac{\partial -L_D}{\partial \zeta}$ (Gradient Reversal).
		\EndWhile
		\item[\textbf{\textit{Estimate Average Reward with Post-Normalization }}]
		
		\For{ $({h}_i, r^{(i)})$ in $\cD$}
		\State $\zeta_i = $ RoBERTa-$\zeta(e_0^{(i)},a_1^{(i)},...,a_{T_i-1}^{(i)},e_{T_i-1}^{(i)},a_{T_i}^{(i)})$
		\EndFor
		\State \textbf{Return} $\hat{\rho}_{n}(\pi) = \sum_{({h}_i, r^{(i)}) \in \cD} \zeta_i r^{(i)} \big/  \sum_{({h}_i, r^{(i)}) \in \cD} \zeta_i$
	\end{algorithmic}
\end{algorithm}

\clearpage
\section{Experiment Set-Up}
\label{app:exp-setup}

In the following experiments, we share the RoBERTa encoder for RoBERTa-$\zeta$ and RoBERTa-$\nu$. On the top of RoBERTa-$\zeta$ and RoBERTa-$\nu$, it is a two-layer fully connected neural network equipped with GeLU activation \citep{hendrycks2016gaussian} and the same hidden dimension as RoBERTa. The RoBERTa encoder is initialized from RoBERTa-base checkpoint \citep{liu2019roberta}. We simply use reverse gradients for the mini-max updates. We set learning rate as $2\times 10^{-4}$ and use inverse square root learning rate decay. We impose the gradient norm clipping with the maximum norm $\norm{\cdot}_{2} \leq 10$. We use $100$ times larger learning rate for optimizing $\lambda$, $2$ times larger learning rate for RoBERTa-$\nu$. In \eqref{eq:dice_minmax_obj}, we set $\alpha_\zeta = 1$, $f(x)=x^2$ as suggested in \citet{yang2020off}. We maintain $\zeta \geq 0$ by adding a square activation at the end of RoBERTa-$\zeta$. The source code is built based on Transformers \citep{Wolf2019HuggingFacesTS}, AirDialog \citep{wei2018airdialogue}, and ParlAI \citep{miller2017parlai}. All experiments are conducted on a machine with $8 \times$ V100 GPUs on Google Cloud.

\section{Transformer-Based Agents for AirDialog}
\label{app:air-trans}

\textbf{Seller Agent Transformer Architecture}
There are four components for the encoder: ticket encoder, reservation encoder, dialog encoder, and task-specific heads (intent classification head and name classification head). 
All tickets and reservation are converted to natural languages. Noticing that, we always append a pseudo ticket in the ticket database representing ``no ticket found'' situation. The architecture is illustrated in Figure~\ref{fig:air_seller_transformer}.

\begin{figure}[!htb]
  \centering
  \includegraphics[width=0.8\textwidth]{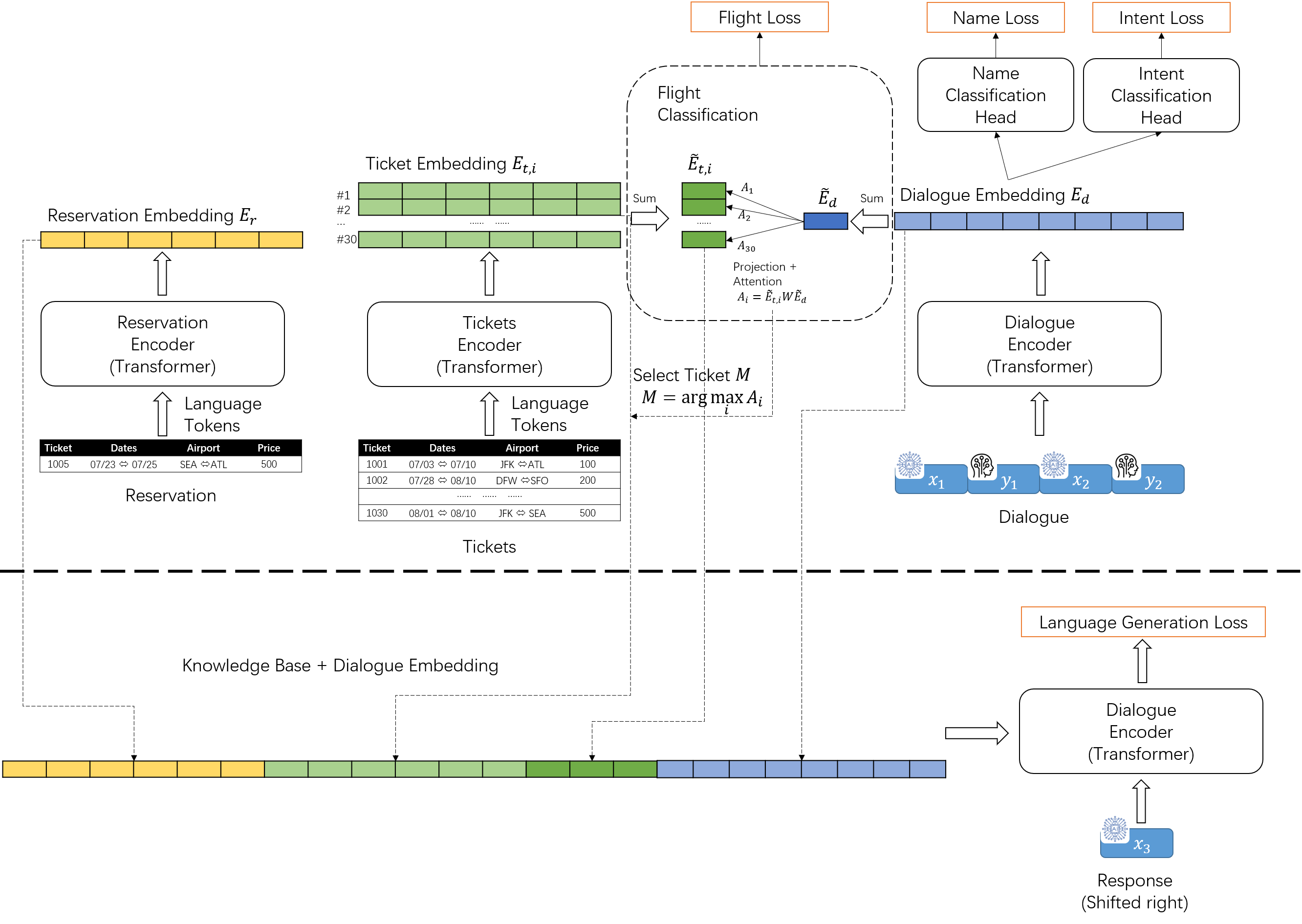}
  \caption{Transformer-based Seller Agent}
  \label{fig:air_seller_transformer}
\end{figure}

\textbf{Customer Agent Transformer Architecture}
There are two components for the encoder: intent encoder, reservation encoder.
All intents are converted to natural languages. The architecture is illustrated in Figure~\ref{fig:air_customer_transformer}.

\begin{figure}[!htb]
  \centering
  \includegraphics[width=0.5\textwidth]{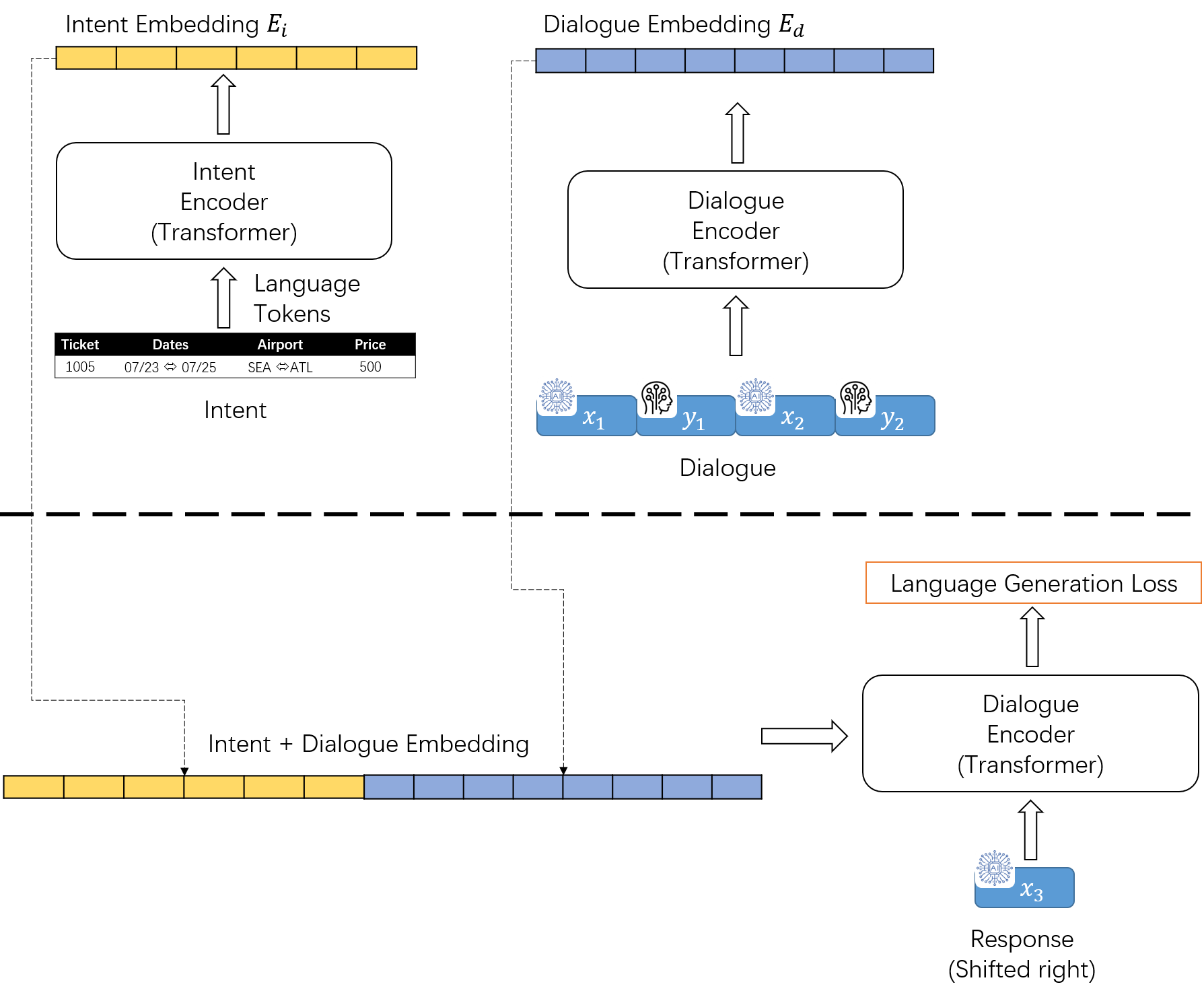}
  \caption{Transformer-based Customer Agent}
  \label{fig:air_customer_transformer}
\end{figure}

% \begin{table}[htb!]
% \caption{Examples for converting tickets, reservations, and intents into natural languages}
% \label{tab:convert_kb}
% \begin{center}
% \begin{tabular}{l|c}
% \textbf{Knowledge} & \textbf{Converted language} \\
% \hline
% Ticket &  \\
% Pseudo Ticket & \\
% Reservation & \\
% Reservation (None) & \\
% Intent & \\
% \end{tabular}
% \end{center}
% \end{table}

\textbf{Training Objective} Besides the language generation loss $\cL_{l}$, the training objective for seller consists of three parts: name loss, flight loss, intent loss:
\begin{align}
    \min_{\theta} \cL_s(\theta) =  \cL_l(\theta) +  \lambda_n \cL_{\textrm{name}}(\theta) +  \lambda_f \cL_{\textrm{flight}}(\theta)+  \lambda_i \cL_{\textrm{intent}}(\theta)
\end{align}

The customer agent is trained with normal language generation loss. 

\textbf{Benchmark} We compare the proposed model with the current AirDialog RNN baseline \citep{wei2018airdialogue}. As can be seen, the agent used in this paper are significantly stronger than  the baseline agent used in \citet{wei2018airdialogue}.

\begin{table}[htb!]
\caption{Benchmark of the proposed transformer based agent. `C' means customer, `S' means seller. Reward, Name, Flight, status are the task-specific scores obtained from self-play evaluation. }
\label{tab:benchmark_trans}
\begin{center}
\begin{tabular}{l|cccccccc}
\toprule \hline 
Model & BLEU (C) & BLEU (S) & PPL (C) & PPL (S) & Reward & Name & Flight & Status\\
\hline
RNN &  22.92 & 32.95 & - & - & 0.23 & 0.41 & 0.13 & 0.29 \\
Ours & 31.78 & 31.70 & 1.671 & 1.843 & 
0.702 & 1.00 & 0.547 & 0.761 \\
\hline \bottomrule 
\end{tabular}
\end{center}
\end{table}

\textbf{Hyper-Parameters} For training 24 seller agents used in Section~\ref{sec:exp}, we varies the size of training data (number of training dialogs) from $5K$ to the full size and varies $\lambda_i$ and $\lambda_f$ from $0.0001$ to $1$. For training the customer agent used in self-play evaluation, we use the full training data and tune the hyperparameters based on the BLEU score evaluated using the validation set. 

\textbf{Human Evaluation} The human evaluation is collected from 20 different Ph.D. students majored in Math/Stats/CS/IEOR. We provide detailed guidelines to the human evaluator that they have to speak to the agents with similar tone. Figure~\ref{fig:human_eval_screen} presents the screen shots of the human evaluation software. 

\begin{figure}
    \centering
    \includegraphics[width=\textwidth]{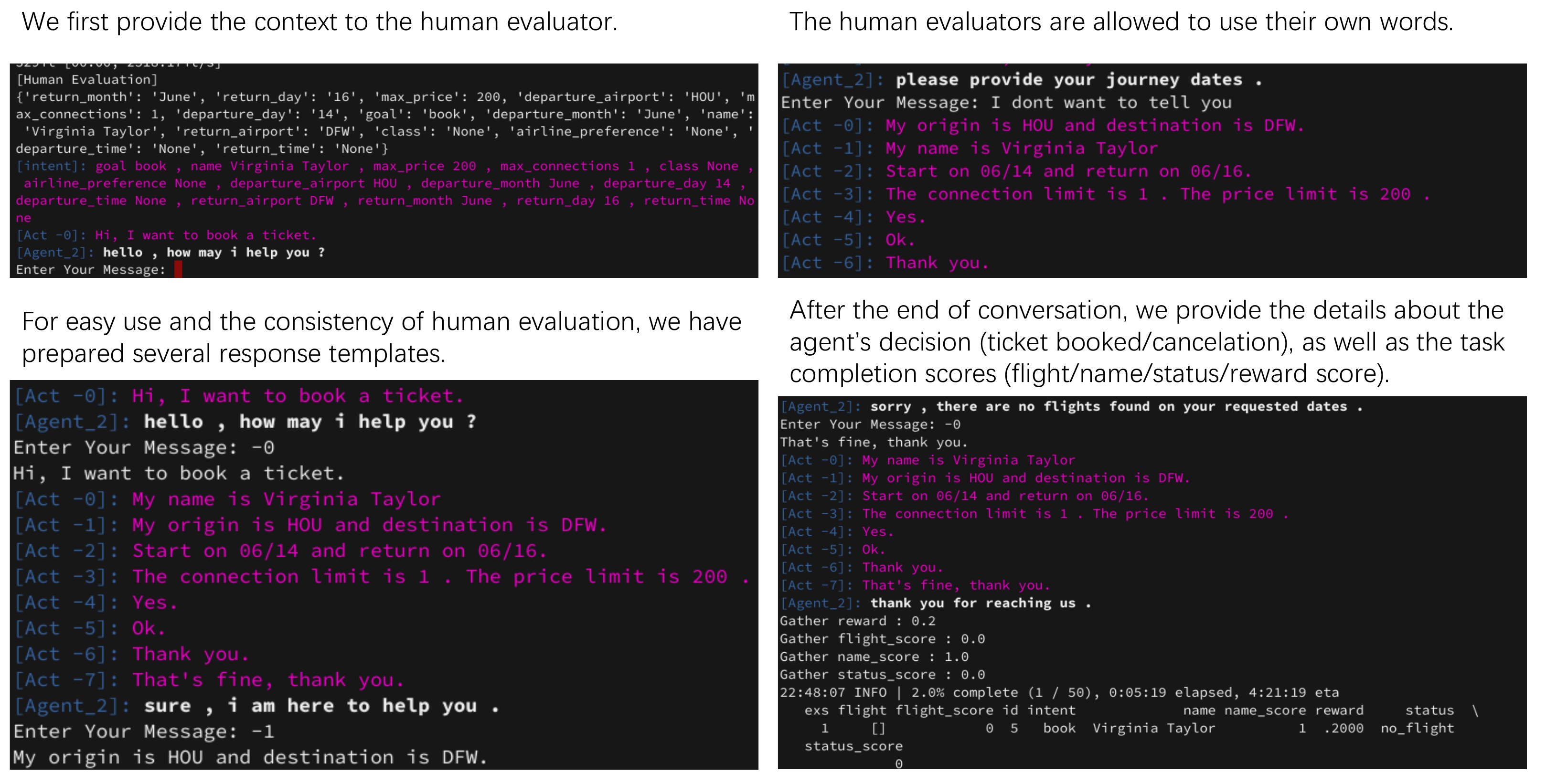}
    \caption{Screen Shots of Human Evaluation Software}
    \label{fig:human_eval_screen}
\end{figure}

%%%%%%%%%%%%%%%%%%%%%%%%%%%%%%%%%%%%%%%%%%%%%%%%%%%%%%%%%%%%%%%%%%%%%%%%%%%%%%%%%%%%%%%%%%%%%%%%%%%%
%%%%%%%%%%%%%%%%%%%%%%%%%%%%%%%%%%%%%%%%%%%%%%%%%%%%%%%%%%%%%%%%%%%%%%%%%%%%%%%%%%%%%%%%%%%%%%%%%%%%
%%%%%%%%%%%%%%%%%%%%%%%%%%%%%%%%%%%%%%%%%%%%%%%%%%%%%%%%%%%%%%%%%%%%%%%%%%%%%%%%%%%%%%%%%%%%%%%%%%%%
%%%%%%%%%%%%%%%%%%%%%%%%%%%%%%%%%%%%%%%%%%%%%%%%%%%%%%%%%%%%%%%%%%%%%%%%%%%%%%%%%%%%%%%%%%%%%%%%%%%%

\clearpage
\section{Additional Experiment}
\label{app}

\subsection{AirDialog}
\label{app:exp_air}

\textbf{Regression Plot}

We present the regression plot for the full setting in Figure~\ref{fig:air_ope_vs_auto} and for the selected agent in Figure~\ref{fig:air_ope_vs_auto_hard}.

\begin{figure}[!htb]
\begin{center}
% \begin{subfigure}{0.32\textwidth}
%   \centering
%   \includegraphics[width=\textwidth]{figure/air/model/bleu_vs_human.pdf}
%   \caption{BLEU vs. Reward (M-M)}
% \end{subfigure}
% \begin{subfigure}{0.32\textwidth}
%   \centering
%   \includegraphics[width=\textwidth]{figure/air/model/ppl_vs_human.pdf}
%   \caption{PPL vs. Reward  (M-M)}
% \end{subfigure}
% \begin{subfigure}{0.32\textwidth}
%   \centering
%   \includegraphics[width=\textwidth]{figure/air/model/ope_vs_human.pdf}
%   \caption{{\ours} vs. Reward  (M-M)}
% \end{subfigure}
\begin{subfigure}{0.49\textwidth}
  \centering
  \includegraphics[width=\textwidth]{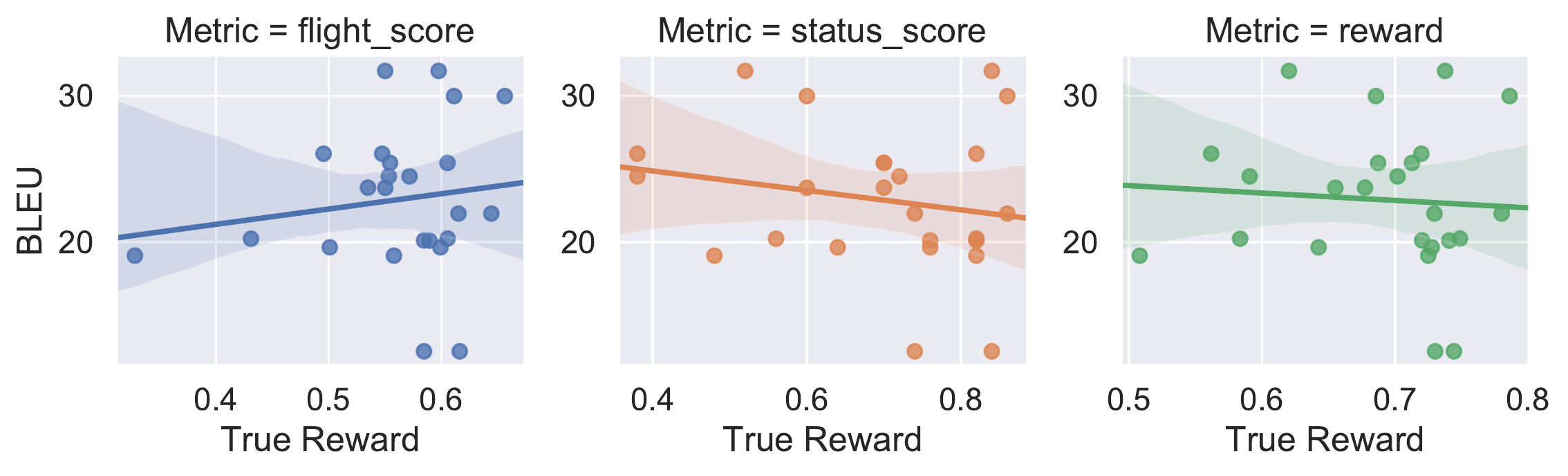}
  \caption{BLEU vs. Human Evaluation}
\end{subfigure}
\begin{subfigure}{0.49\textwidth}
  \centering
  \includegraphics[width=\textwidth]{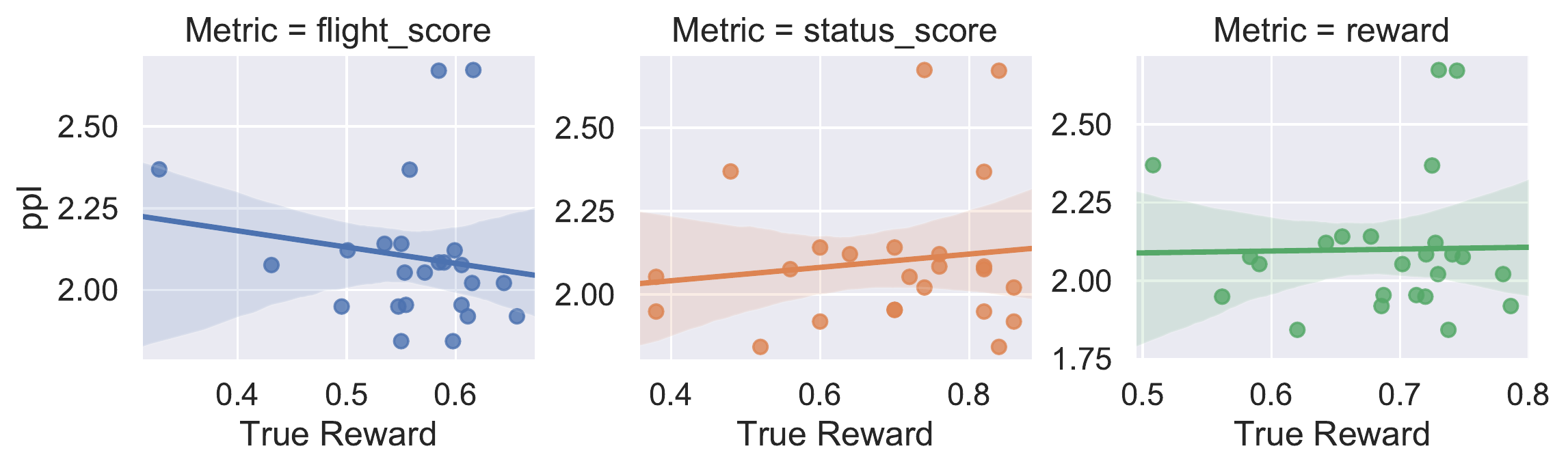}
  \caption{PPL vs. Human Evaluation}
\end{subfigure}
\begin{subfigure}{0.49\textwidth}
\includegraphics[width=\textwidth]{figure/air/human/selfplay_vs_human.pdf}
\vspace{-0.1in}
\caption{SPE vs. Human Evaluation}
\end{subfigure}
\begin{subfigure}{0.49\textwidth}
\includegraphics[width=\textwidth]{figure/air/human/ope_vs_human.pdf}
\vspace{-0.1in}
\caption{{\ours} vs. Human Evaluation}
\end{subfigure}

\end{center}
\caption{Regression Plot.  The x-axis is the average reward obtained by chatting with human. The y-axis is BLEU/PPL/the reward estimated by {\ours}. Different colors denotes different type of rewards (flight score, status score, and overall reward).  The solid line is obtained by linear regression and the shaded region indicates $95\%$ confidence interval. (see more in \textit{seaborn} packages). }
\label{fig:air_ope_vs_auto}
\end{figure}

\begin{figure}[!htb]
\begin{center}
\begin{subfigure}{0.49\textwidth}
  \centering
  \includegraphics[width=\textwidth]{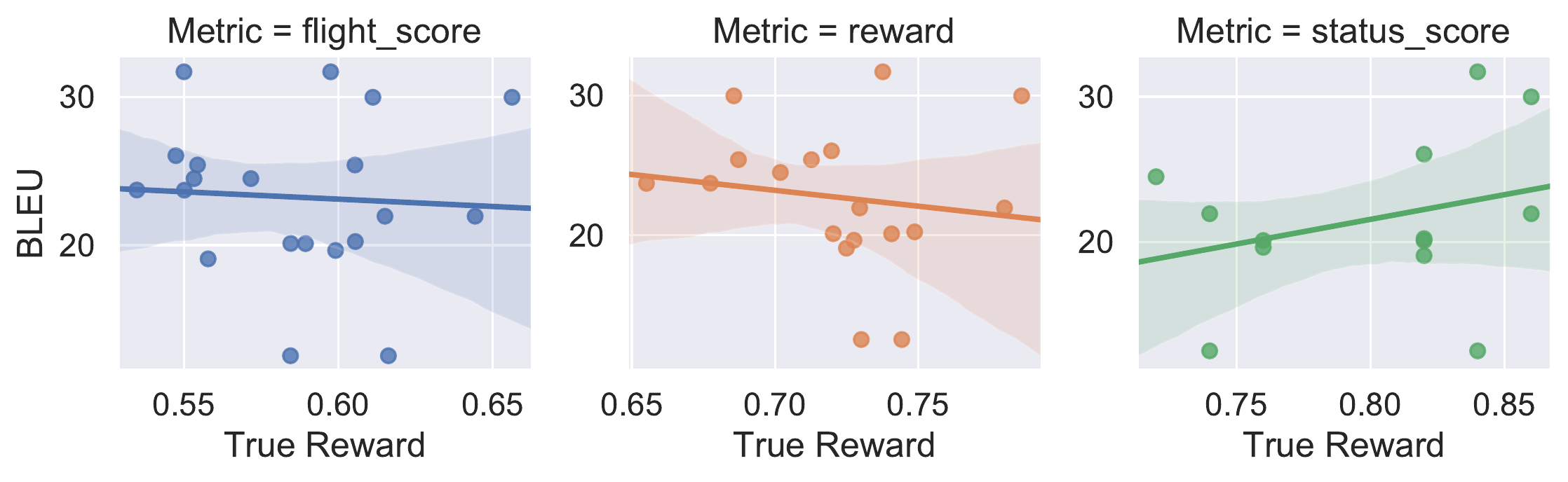}
  \caption{BLEU vs. Human Evaluation}
\end{subfigure}
\begin{subfigure}{0.49\textwidth}
  \centering
  \includegraphics[width=\textwidth]{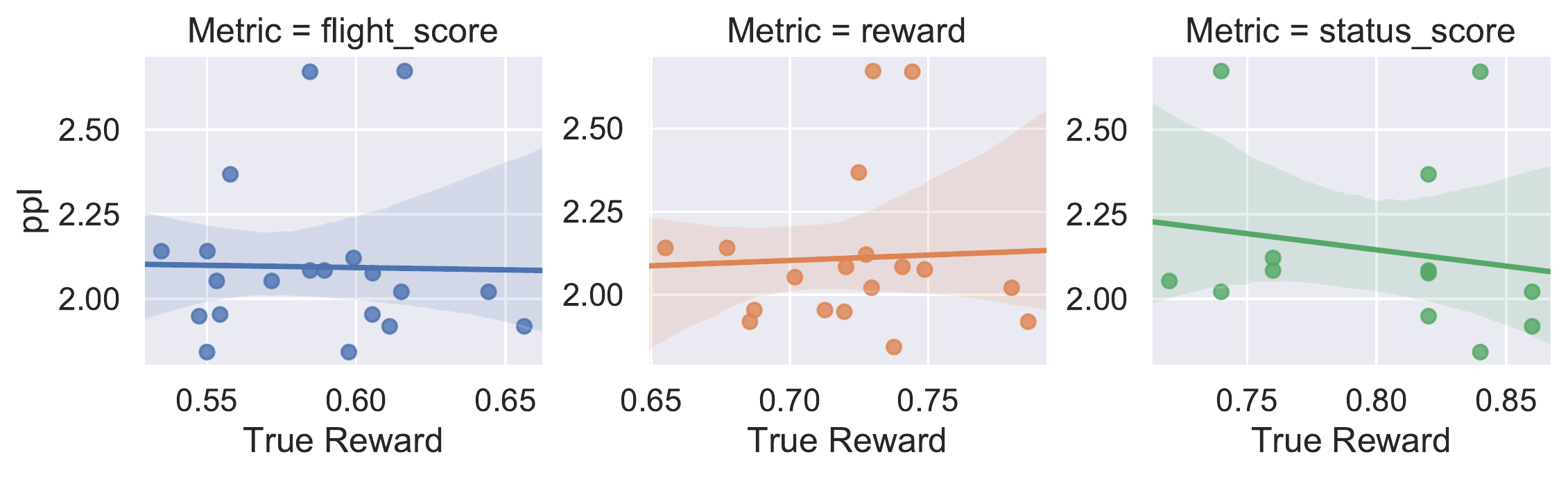}
  \caption{PPL vs. Human Evaluation}
\end{subfigure}
\begin{subfigure}{0.49\textwidth}
  \centering
  \includegraphics[width=\textwidth]{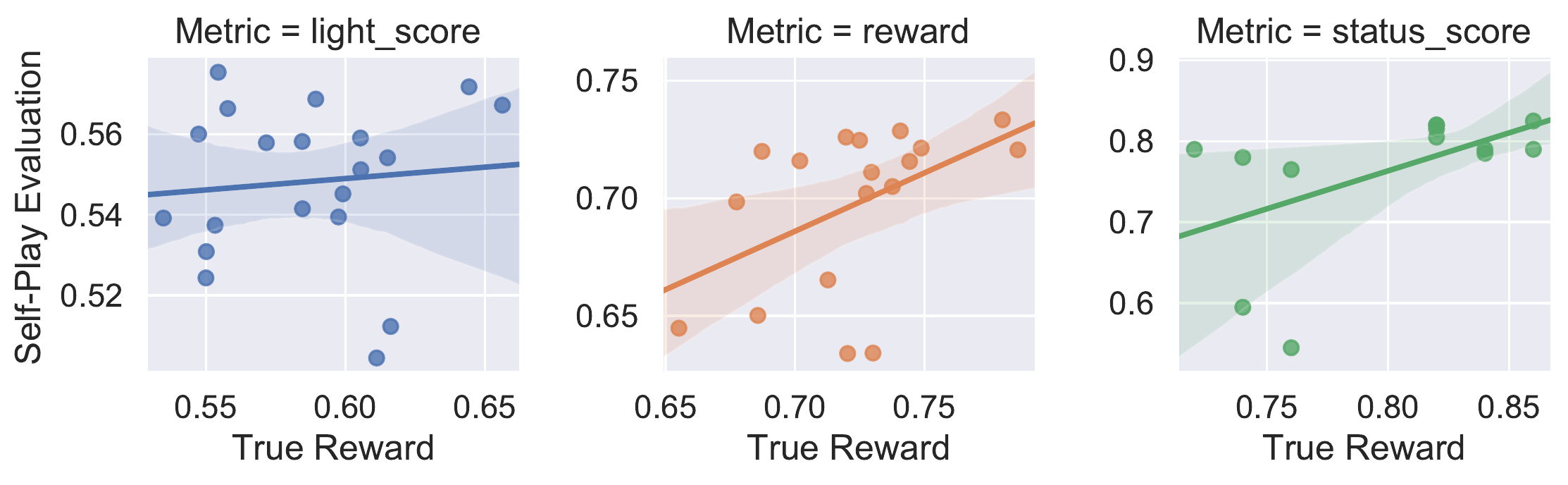}
  \caption{SPE vs. Human Evaluation}
\end{subfigure}
\begin{subfigure}{0.49\textwidth}
  \centering
  \includegraphics[width=\textwidth]{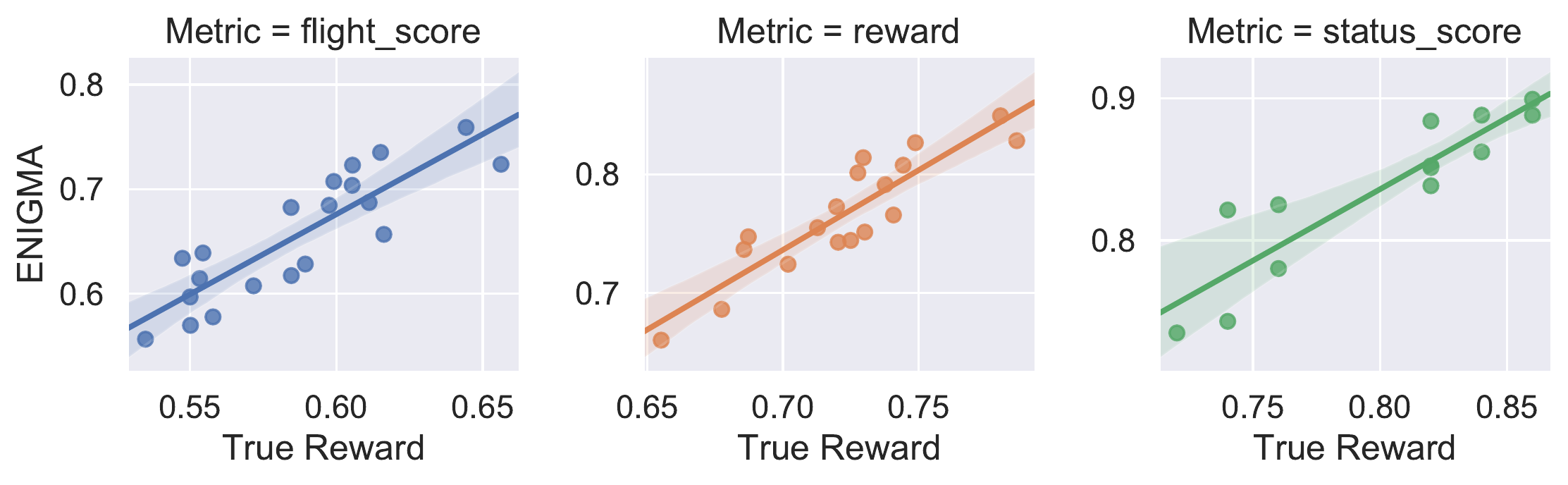}
  \caption{{\ours} vs. Human Evaluation }
\end{subfigure}

\end{center}
\caption{Regression Plot for ``selected agent'' Setting.  The x-axis is the average reward obtained by chatting with human. The y-axis is BLEU/PPL/the reward estimated by {\ours}/Self-Play Evaluation (SPE). Different colors denotes different type of rewards (flight score, status score, and overall reward).  The solid line is obtained by linear regression and the shaded region indicates $95\%$ confidence interval. (see more in \textit{seaborn} packages). }
\label{fig:air_ope_vs_auto_hard}
\end{figure}

\textbf{Training Curves}

We show the training curves of the {\ours} in Figure~\ref{fig:air_learncurve}. Here four models are presented, the best model (ranked $100\%$), model ranked as $50\%$, model ranked as $25\%$ and the worst model (ranked $0\%$). As can been seen the estimated reward estimation converges steadily to it's true values. 

\begin{figure}[h]
\begin{center}
% \begin{subfigure}{\textwidth}
%   \centering
%   \includegraphics[width=\textwidth]{figure/air/human/learning_curve.pdf}
%   \caption{Model-Human}
%   \label{fig:air_humancurve}
% \end{subfigure}
% \begin{subfigure}{\textwidth}
%   \centering
%   \includegraphics[width=\textwidth]{figure/air/model/learning_curve.pdf}
%   \caption{Model-Model}
%   \label{fig:air_modelcurve}
% \end{subfigure}
  \includegraphics[width=\textwidth]{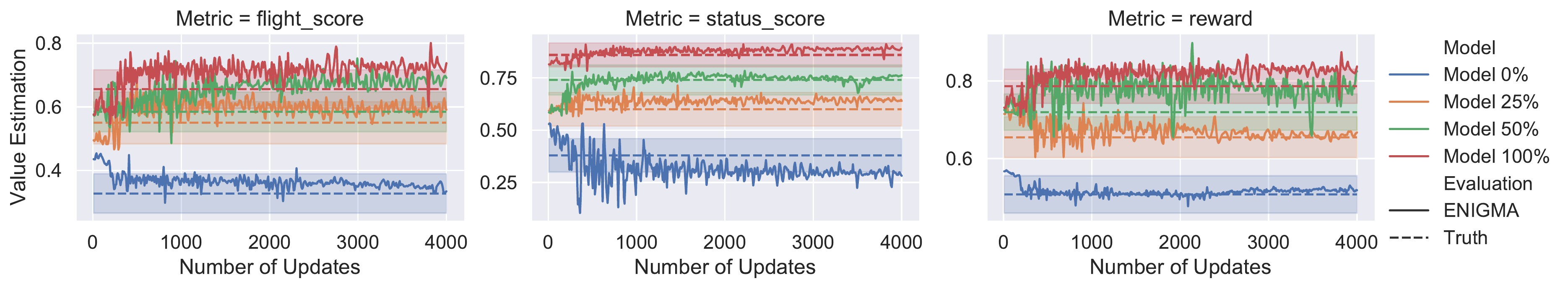}
\end{center}
\caption{Learning curve for AirDialog. The x-axis is the number of mini-max updates, while y-axis is the estimated values.  The straight line is the true reward, while the shaded region denotes the $90\%$ confidence interval. The true reward and the confidence interval is obtained via different evaluation chats between the agents and the environment (model/human). Different colors denotes different agents. }
\label{fig:air_learncurve}
\end{figure}

\textbf{Ablation Study}

Here we provide large figures (Figure~\ref{fig:prelim_exp_full} and Figure~\ref{fig:prelim_loss_exp_full}) for the ablation study mentioned in Section~\ref{sec:exp}.

\begin{figure*}[htb!]
% 	\vspace{-0.05in}
	\begin{center}
		\includegraphics[height=8cm]{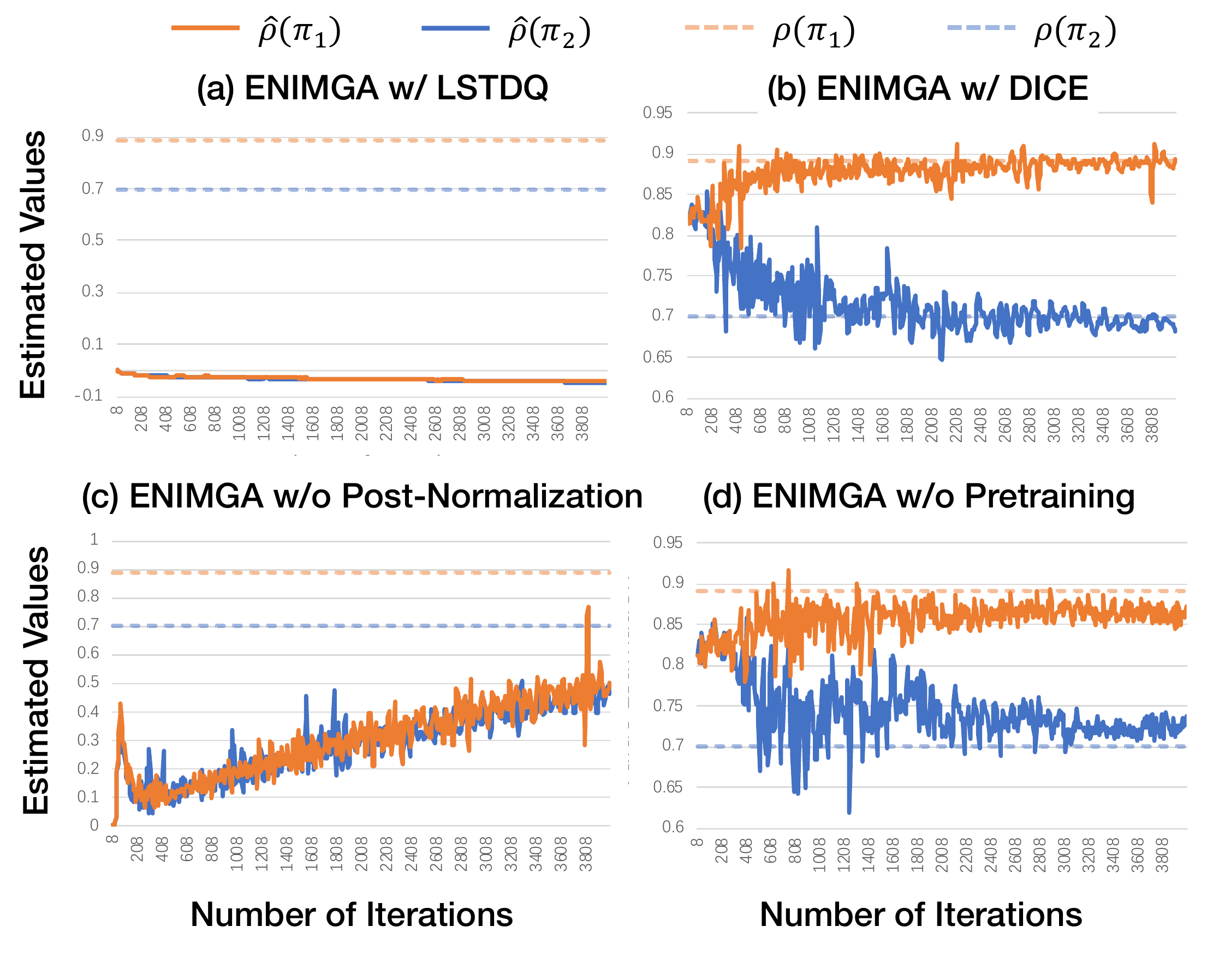}
	\end{center}
% 	\vspace{-0.25in}
	\caption{Reward estimation of two target agents ($\pi_1$ and $\pi_2$) vs. \# of iterations. Dotted lines represents true rewards. }
% 	\vspace{-0.25in}
	\label{fig:prelim_exp_full}
\end{figure*} 

\begin{figure}[!htb]
 	%\vspace{-0.05in}
	\begin{center}
		\includegraphics[width=0.7\textwidth]{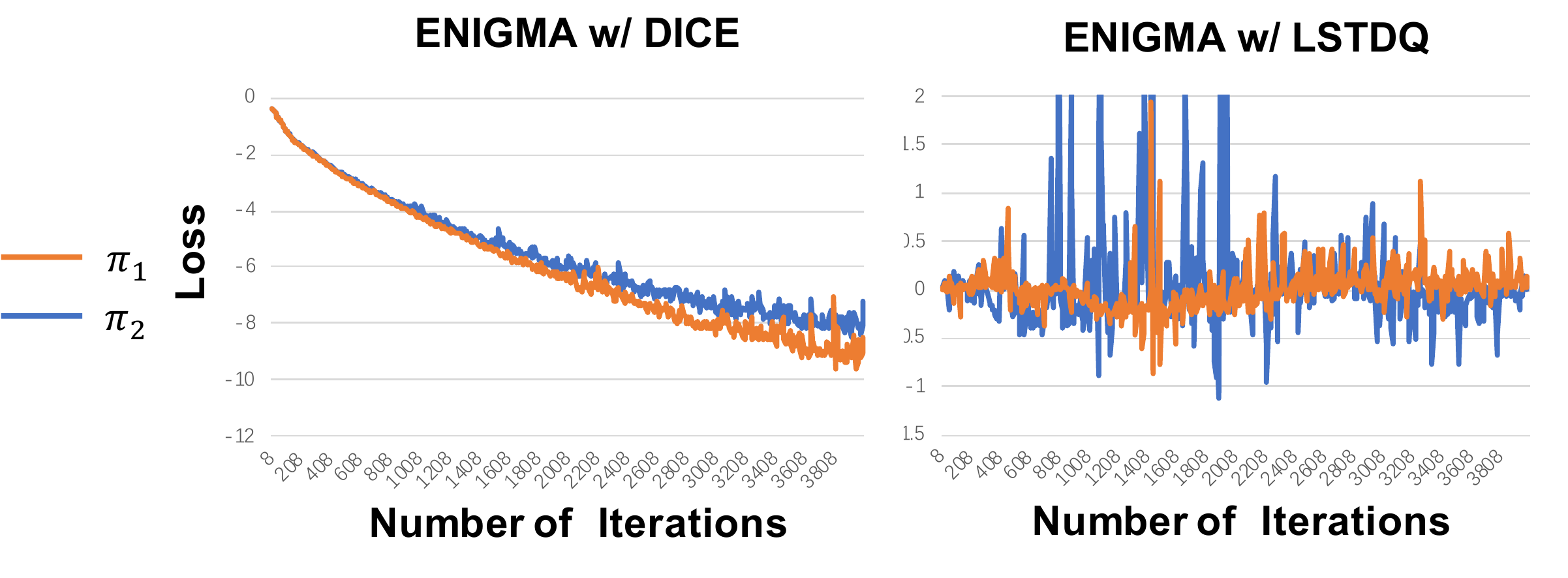}
	\end{center}
% 	\vspace{-0.2in}
	\caption{Loss value of two target agents during mini-max optimization. }
% 	\vspace{-0.3in}
	\label{fig:prelim_loss_exp_full}
\end{figure}

\subsection{Additional Results for Rule-Based Agents of AirDialog} \label{app:rule-air}

\noindent~$\bullet$ \textit{Rule-Rule (R-R)}: Both customer and seller agents are rule based. We fix the customer rule-based model and construct and evaluate 6 seller agents. The strongest agent can perfectly interpret the intent of rule-based customers. While the weaker agents interprets the intent with different levels of noise. The learning curve is presented in Figure~\ref{fig:air_rulecurve}. 

\begin{figure}[!htb]
  \centering
  \includegraphics[width=0.5\textwidth]{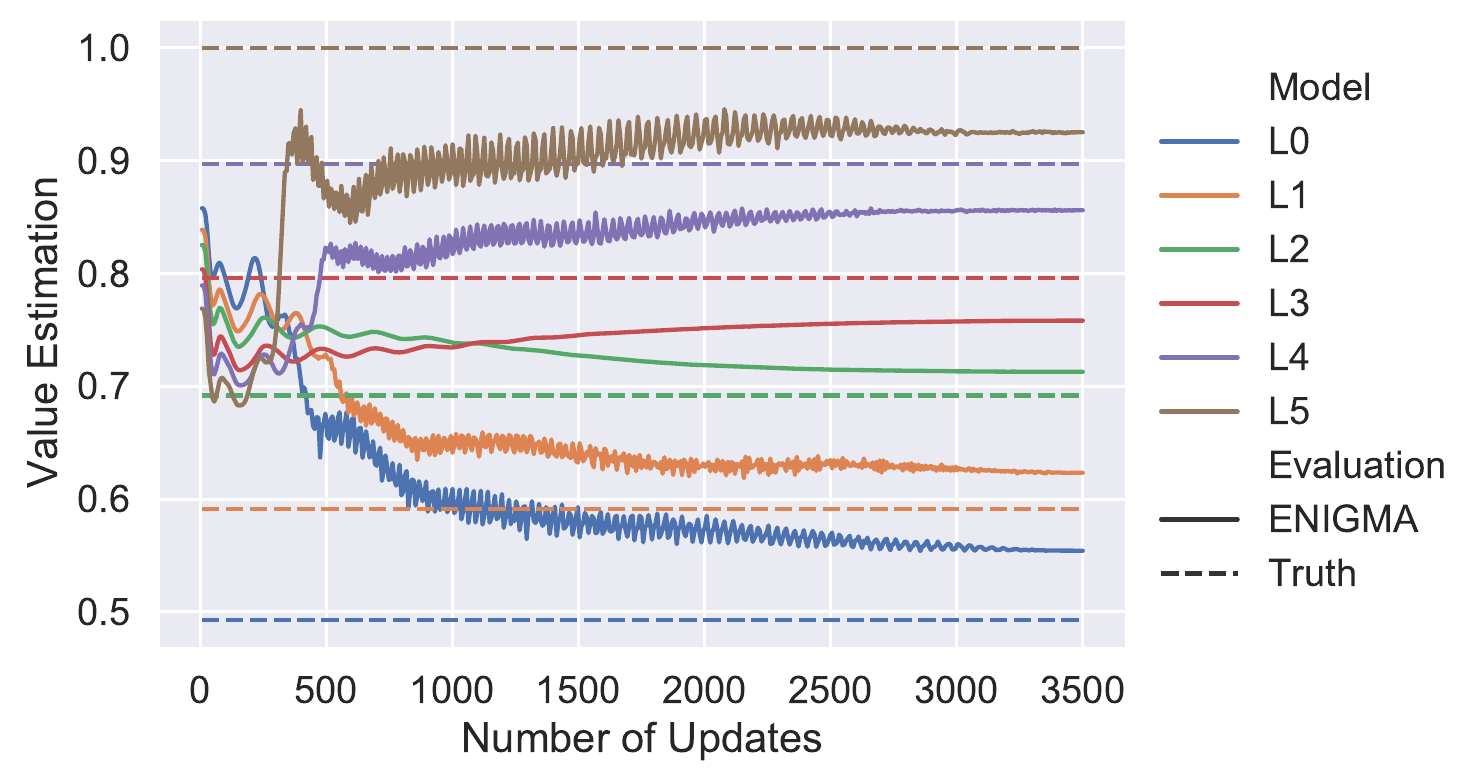}
  \caption{Learning Curve under Rule-Rule setting}
  \label{fig:air_rulecurve}
\end{figure}

\begin{table}[htb!]
\small
\caption{The correlation between two metrics. Each column is a task completion score obtained by interacting with the environments under R-R setting. Each row is an automatic metric. }
\label{tab:air_r2_r-r}
\vspace{-0.15in}
\begin{center}
\begin{tabular}{l|l|ccc|ccc}
\toprule \hline
\multicolumn{1}{c|}{\multirow{2}{*}{\bf Setting}}&\multicolumn{1}{c|}{\multirow{2}{*}{\bf Method}} & \multicolumn{3}{c|}{\bf Pearson Correlation}& \multicolumn{3}{c}{\bf Spearman's Rank Correlation } \\
\cline{3-8}
 &  &\multicolumn{1}{c}{\bf Flight Score } &\multicolumn{1}{c}{\bf Status Score } &\multicolumn{1}{c|}{\bf Reward } &\multicolumn{1}{c}{\bf Flight Score } &\multicolumn{1}{c}{\bf Status Score } &\multicolumn{1}{c}{\bf Reward } 
\\ \hline
\multirow{3}{*}{R-R} 
& BLEU        & 0.1981 & -0.0067 & 0.0980 & 0.1525 & 0.0009 & 0.0924\\
& ppl         & -0.1584 & -0.0610 & -0.1209 & -0.2475 & -0.1060 & -0.1178\\
& {\ours}         & \textbf{0.9687} & \textbf{0.9947} & \textbf{0.9874} & \textbf{0.8800} & \textbf{0.9872} & \textbf{0.9574}\\
\hline \bottomrule
\end{tabular}
\end{center}
\vspace{-0.15in}
\end{table}

\clearpage
\subsection{ConvAI2}
\label{app:exp_convai2}

\textbf{Training Curves.} 

Similar to the AirDialog dataset, we also show the training curves for the agents ranked at $100\%$, $50\%$, $25\%$, $0\%$ in Figure~\ref{fig:convai2_learncurve}. {\ours} also converges steadily to the true values within a resonable error.

\begin{figure}[h]
\begin{center}
\includegraphics[width=\textwidth]{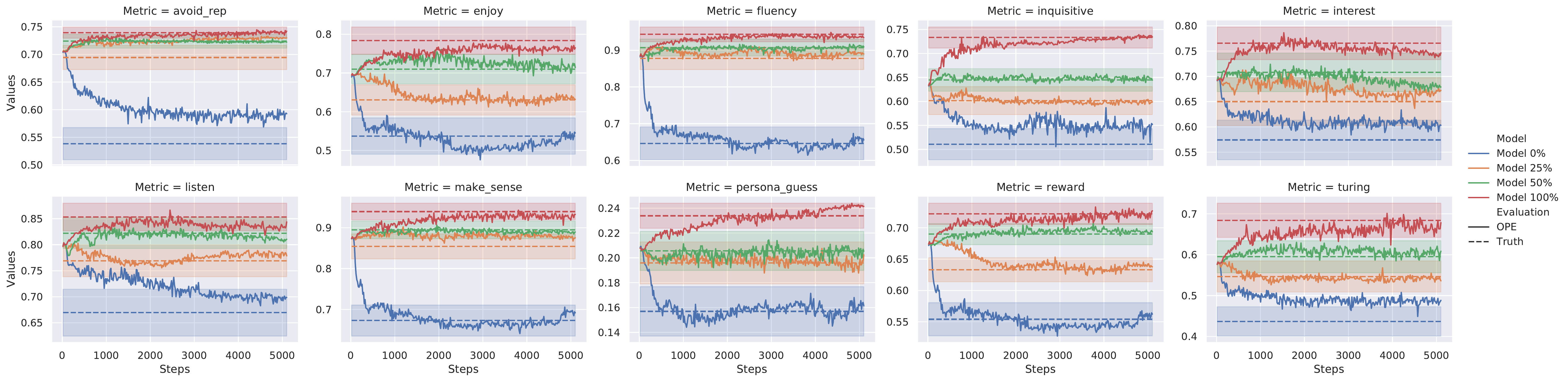}
\end{center}
\caption{Learning curve for ConvAI2. The x-axis is the number of mini-max updates, while y-axis is estimated values.  The straight line is the true reward, while the shaded area denotes the $95\%$ confidence interval. The true reward and the confidence interval is obtained via different evaluation chats between the agents and human. Different colors denotes different agents. }
\label{fig:convai2_learncurve}
\end{figure}

\textbf{Regression Plot.} 

We present the regression plot for the all 10 metrics in setting in Figure \ref{fig:convai2_opevshuman_full}. The corresponding corresponding correlation is presented in Table~\ref{tab:convai2_r2_full}. For comparison, we present the regression plot for self-play evaluation in Figure~\ref{fig:convai2_spvshuman_full}.

\begin{figure}[!h]
\begin{center}
\includegraphics[width=\textwidth]{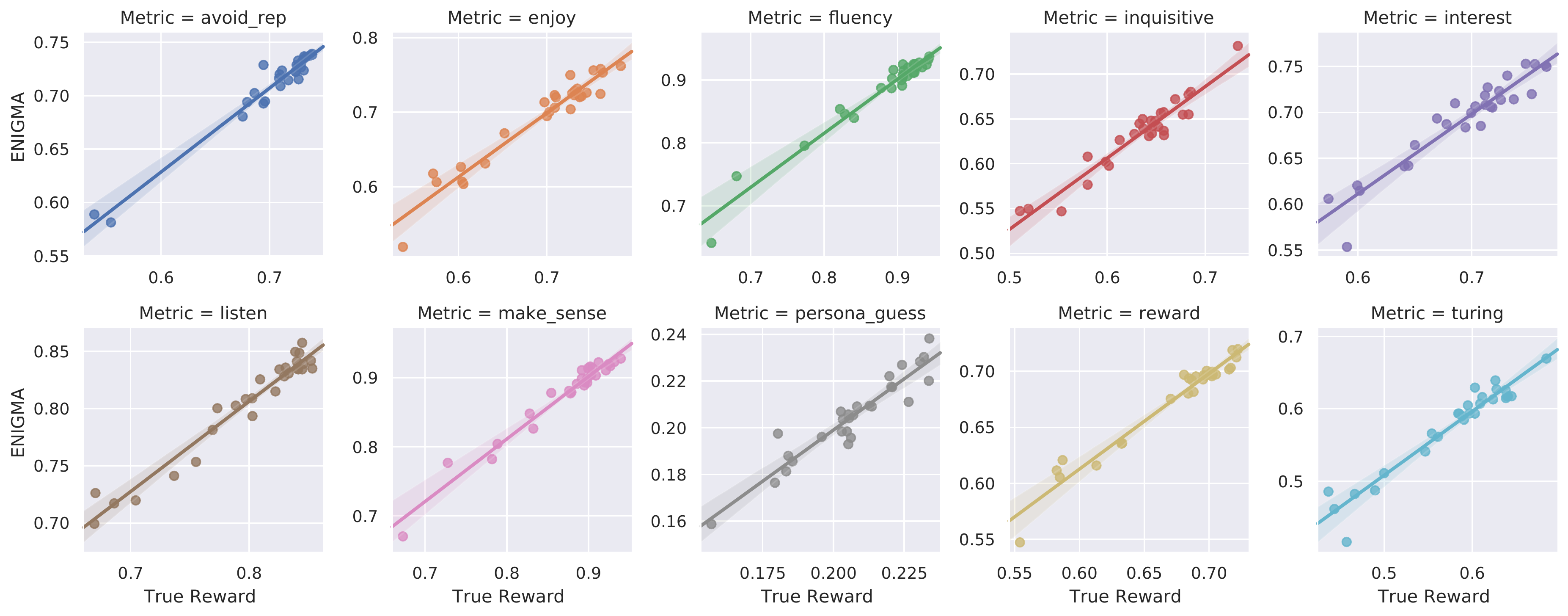}
\end{center}
\caption{ENIGMA vs. Human Evaluation for ConvAI2. The x-axis is the average reward obtained by chatting with human. The y-axis it the reward estimated by {\ours}. Different colors represent different language quality metrics. The solid line is obtained by simple linear regression.}
\label{fig:convai2_opevshuman_full}
\end{figure}

\begin{figure}[!h]
\begin{center}
\includegraphics[width=\textwidth]{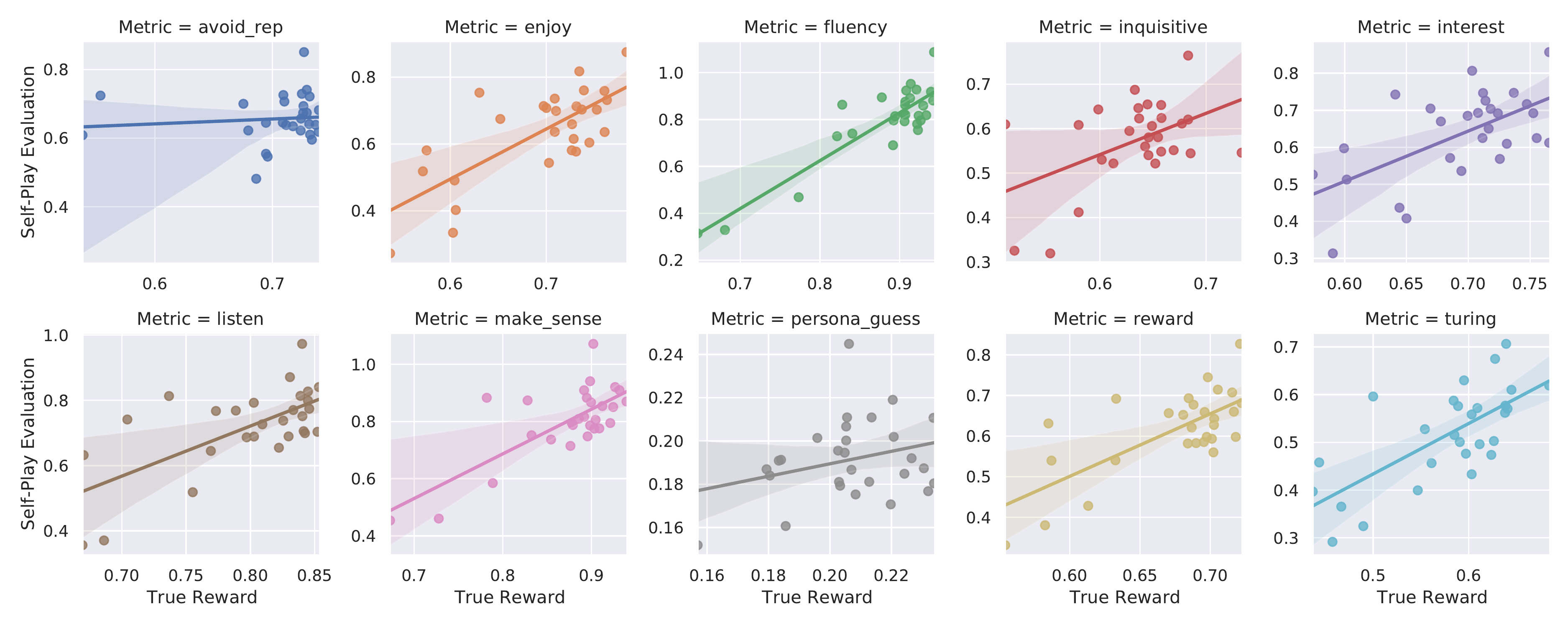}
\end{center}
\caption{Self-Play Evaluation vs. Human Evaluation for ConvAI2. The x-axis is the average reward obtained by chatting with human. The y-axis it the reward estimated by self-play evaluation Different colors represent different language quality metrics. The solid line is obtained by simple linear regression.}
\label{fig:convai2_spvshuman_full}
\end{figure}

\paragraph{Experience Data.} 

To analysis how many human-model evaluation dialogs are needed, we analysis {\ours} error under different sizes of the experience data. For ConvAI2, we compare the error for using $100\%$ data, $50\%$ data and $10\%$ data. As shown in Table~\ref{tab:convai2_r2_full} and Figure~\ref{fig:ope_error_size}, when we use half of the data, the error is similar to the one using full data. If we only use $10\%$ data, {\ours} becomes very inaccurate. OPE under low resource setting remains very challenging. 

In Figure~\ref{fig:ope_error_size}, we study the estimation error under different sizes of the experience data. As can be seen, when using $50\%$ data, the reward value estimation is very similar to the one of using full data. When using only $10\%$ data, the error is larger and {\ours} has lower correlation with the true reward. 

\begin{table}[htb!]
\caption{The correlation between different metrics and {\ours} estimation. Each column is each average language quality score obtained by chatting with human. Different rows represent different experience data {\ours} used. }
\label{tab:convai2_r2_full}
\begin{center}
\begin{tabular}{lccccc}
\toprule \hline
\multicolumn{6}{c}{\textbf{Pearson Correlation}}
\\
\hline
\multicolumn{1}{c}{\bf Setting} & \multicolumn{1}{c}{\bf Avoid Rep.}  &\multicolumn{1}{c}{\bf Enjoy } &\multicolumn{1}{c}{\bf Fluency } &\multicolumn{1}{c}{\bf Inquisitive } &\multicolumn{1}{c}{\bf Interest }
\\ \hline 
Full Data   & 0.9792 & 0.9661 & 0.9767 & 0.9584 & 0.9488 \\
$50\%$ Data & 0.9573 & 0.9046 & 0.9550 & 0.9237 & 0.8644\\
$10\%$ Data & 0.9266 & 0.6595 & 0.8910 & 0.8286 & 0.5052\\
Selected Data   & 0.6944 & 0.6759 & 0.7762 & 0.5605 & 0.5820\\
\hline
\multicolumn{1}{c}{\bf Setting} &\multicolumn{1}{c}{\bf Listen } &\multicolumn{1}{c}{\bf Make Sense }  &\multicolumn{1}{c}{\bf Persona } &\multicolumn{1}{c}{\bf Reward}  &\multicolumn{1}{c}{\bf Turing} 
\\ \hline 
Full Data   & 0.9754 & 0.9788 & 0.9415 & 0.9773 & 0.9637 \\
$50\%$ Data & 0.8971 & 0.9585 & 0.8770 & 0.9374 & 0.8506\\
$10\%$ Data & 0.7455 & 0.8100 & 0.4544 & 0.8240 & 0.6825\\
Selected Data   & 0.5520 & 0.7402 & 0.6879 & 0.6968 & 0.5394\\
\hline 
\hline
\multicolumn{6}{c}{\textbf{Spearman's rank correlation}}
\\
\hline
\multicolumn{1}{c}{\bf Setting} & \multicolumn{1}{c}{\bf Avoid Rep.}  &\multicolumn{1}{c}{\bf Enjoy } &\multicolumn{1}{c}{\bf Fluency } &\multicolumn{1}{c}{\bf Inquisitive } &\multicolumn{1}{c}{\bf Interest }
\\ \hline 
Full Data   & 0.8905 & 0.9070 & 0.9178 & 0.8717 & 0.9210 \\
$50\%$ Data & 0.7558 & 0.7980 & 0.6651 & 0.8482 & 0.7727 \\
$10\%$ Data & 0.4128 & 0.6147 & 0.6492 & 0.6335 & 0.4713 \\
Selected Data   & 0.5138 & 0.5561 & 0.4522 & 0.3168 & 0.6027 \\
\hline
\multicolumn{1}{c}{\bf Setting} &\multicolumn{1}{c}{\bf Listen } &\multicolumn{1}{c}{\bf Make Sense }  &\multicolumn{1}{c}{\bf Persona } &\multicolumn{1}{c}{\bf Reward}  &\multicolumn{1}{c}{\bf Turing} 
\\ \hline 
Full Data   & 0.9240 & 0.9448 & 0.9205 & 0.9485 & 0.9213 \\
$50\%$ Data & 0.7784 & 0.8647 & 0.8293 & 0.7750 & 0.7026 \\
$10\%$ Data & 0.3914 & 0.5096 & 0.3651 & 0.5774 & 0.5893 \\
Selected Data   & 0.4585 & 0.6126 & 0.5844 & 0.6672 & 0.4259 \\
\hline \bottomrule
\end{tabular}
\end{center}
\end{table}

\begin{figure}[!htb]
\begin{center}
\includegraphics[width=\textwidth]{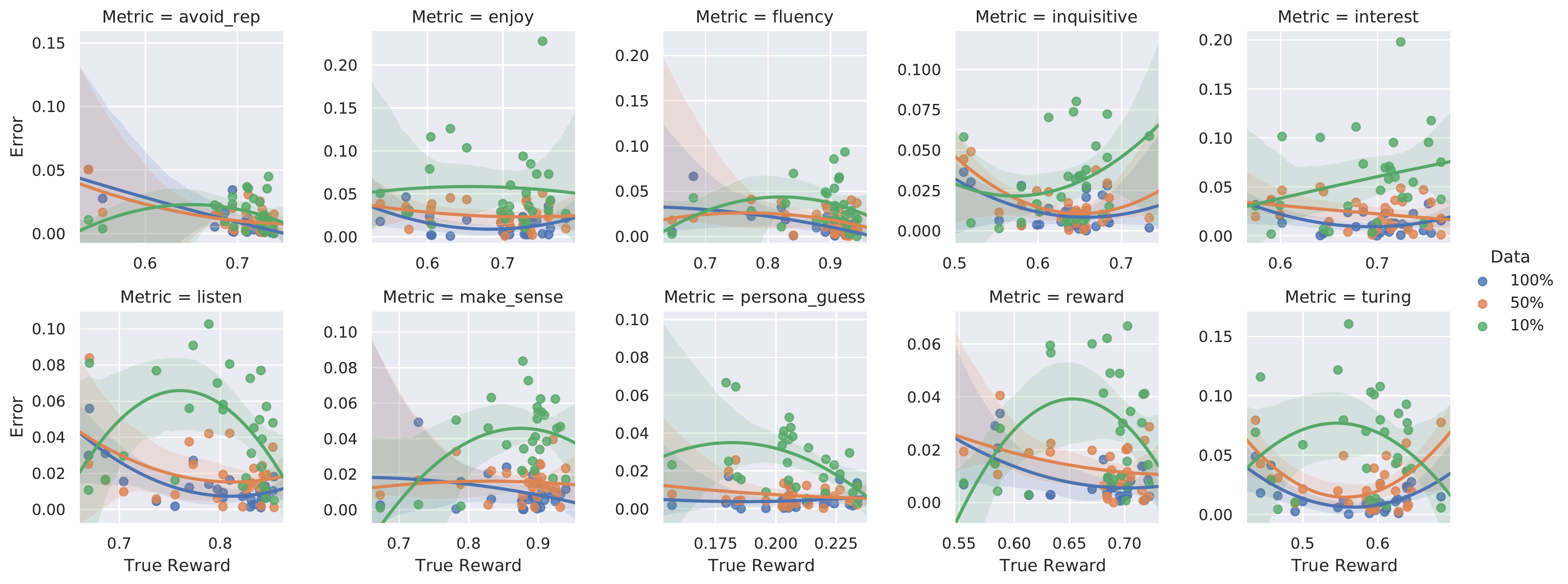}
\end{center}
\caption{Error Analysis on Convai2 under different data size. The x-axis is the true average reward. The y-axis is the {\ours} error. The solid line is the fitted quadratic function. Blue, orange, green colors represent $100\%$, $50\%$, $10\%$ datasets respectively. }
\label{fig:ope_error_size}
\end{figure}

% \clearpage

\textbf{A More Challenging Setting.}

Considering that some target agents are similar to the behavior policies with only slight difference in the way of decoding, they might yield very the similar dialog when the human acts in the same way. 
Specifically, in the data collection process, the target model might yield the responses that are very similar to the ones of the behavior policy for all turns in the dialog: ${\rm EditDistance}(a_t,a'_t) \leq 15 ~~ \forall 0\leq t \leq T$. For a more realistic setting, we consider removing these highly overlapped dialogs after the data collection process. This setting is very challenging that the target policy behavior is less covered by the experience data and ENIGMA can only hopefully generalize via pre-trained RoBERTa. The results are shown in Figure~\ref{fig:convai2_opevshuman_hard} and Table~\ref{tab:convai2_r2_full}. As can be seen, this setting remains challenging as the Pearson correlation is between $0.5$ and $0.8$. For comparison, we present the regression plot for self-play evaluation using this challenging subset of the experience data in Figure~\ref{fig:convai2_spvshuman_hard}.

\begin{figure}[!htb]
\begin{center}
\includegraphics[width=\textwidth]{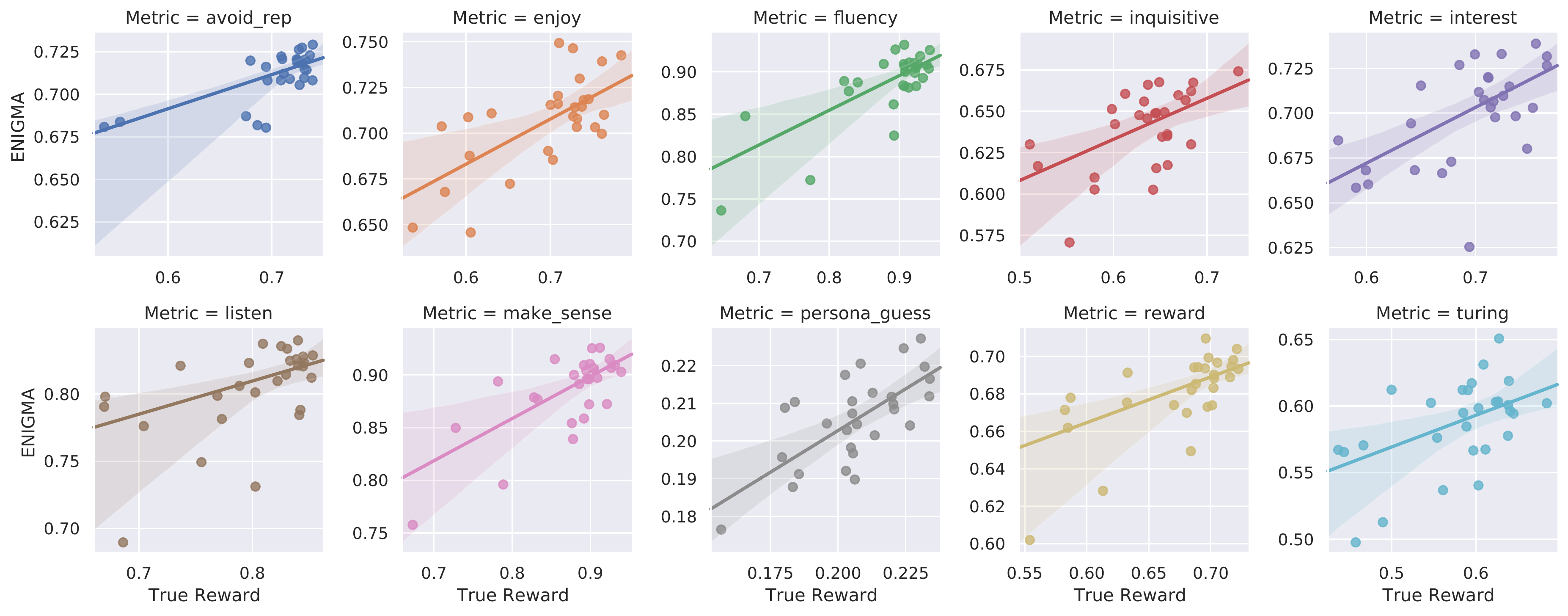}
\end{center}
\caption{{\ours} vs. Human Evaluation for ConvAI2 under the challenging setting. The x-axis is the average reward obtained by chatting with human. The y-axis it the reward estimated by {\ours}. Different colors represent different language quality metrics. The solid line is obtained by simple linear regression.}
\label{fig:convai2_opevshuman_hard}
\end{figure}

\begin{figure}[!htb]
\begin{center}
\includegraphics[width=\textwidth]{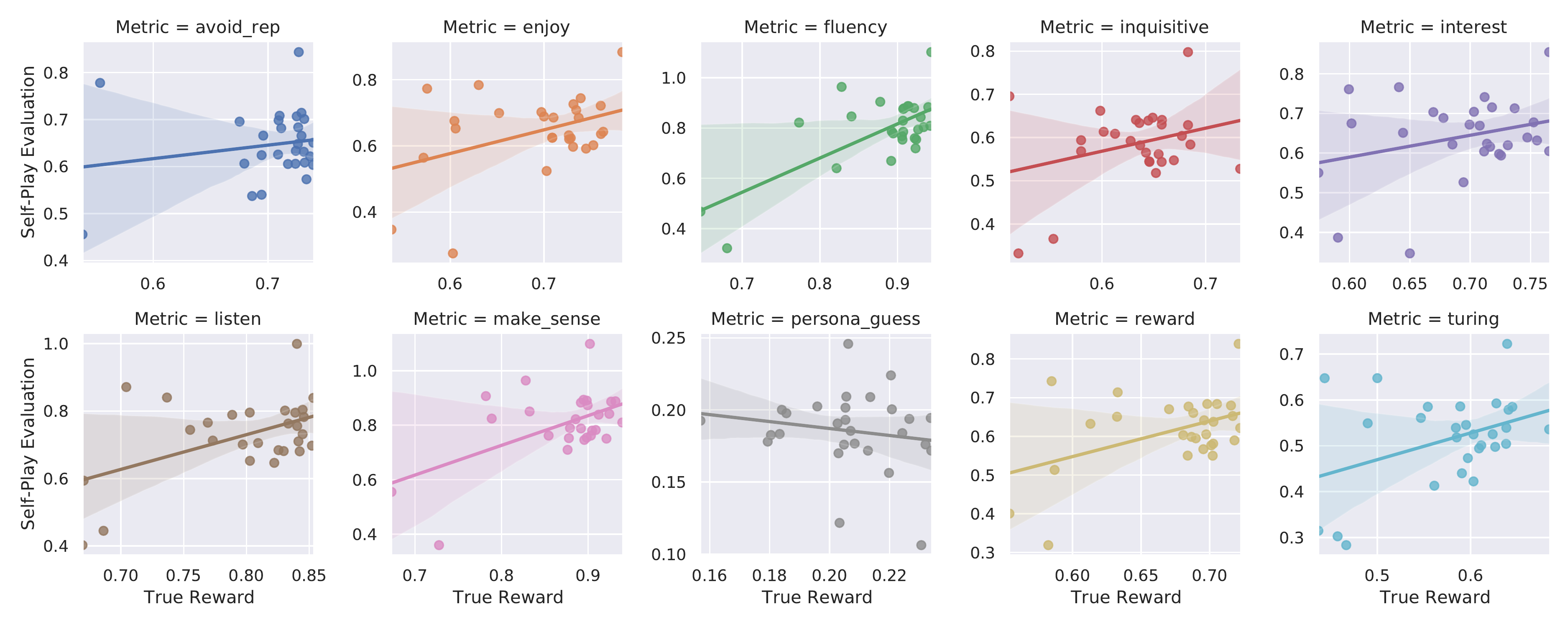}
\end{center}
\caption{Self-Play Evaluation vs. Human Evaluation for ConvAI2 under the challenging setting. The x-axis is the average reward obtained by chatting with human. The y-axis it the reward estimated by self-play evaluation. Different colors represent different language quality metrics. The solid line is obtained by simple linear regression.}
\label{fig:convai2_spvshuman_hard}
\end{figure}

\begin{figure}[!htb]
\begin{center}
\includegraphics[width=\textwidth]{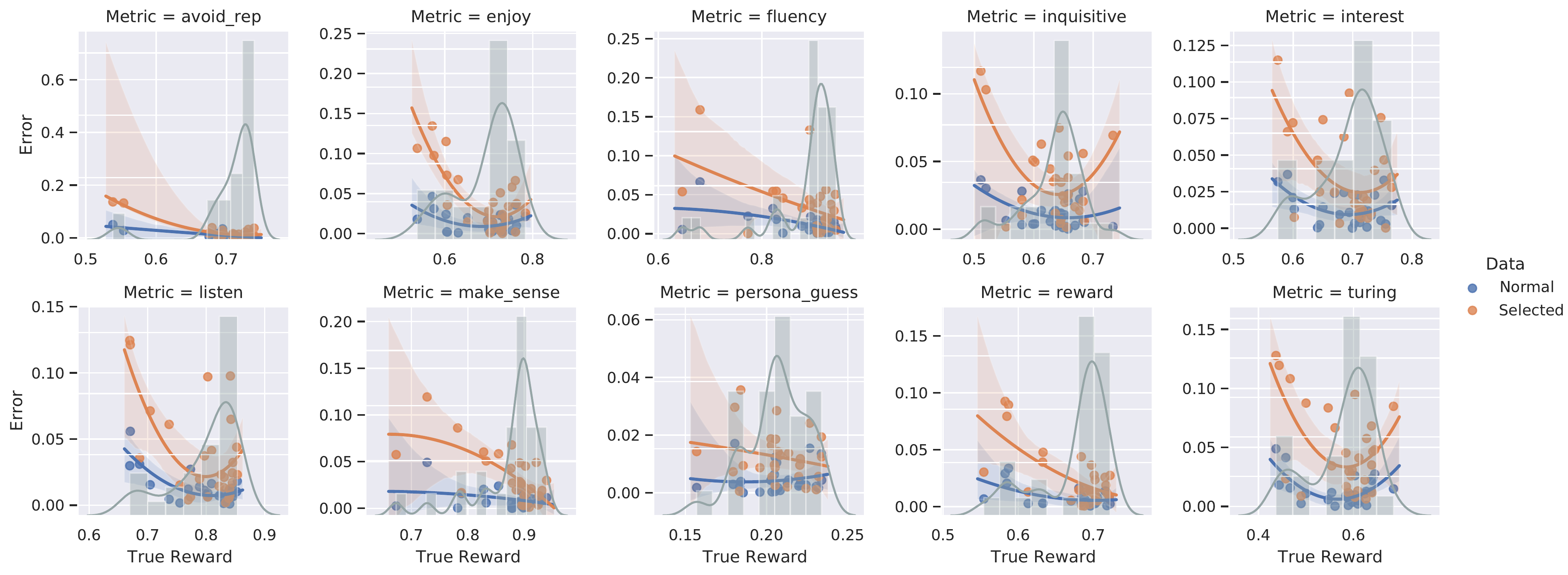}
\end{center}
\caption{{\ours} Error Comparison between using normal and selected challenging experience data on ConvAI2. The x-axis is the true average reward. The y-axis is the {\ours} error. The solid line is the fitted quadratic function. The histogram is the empirical distribution of the rewards of all the experience data. Orange represents challenging dataset, and blue represents normal dataset.}
\label{fig:convai2_error_hard}
\end{figure}

We remark that such experiments can also be done for AirDialog. However, due to the limitation that most agents are just learning template responses due to the goal-oriented nature, removing overlapped dialogs results in an extremely incomplete experience dataset. For example, most ``cancelation'' dialogs will be removed since they are very simple and basically the same for different agents. As a result {\ours} can not make a reasonable estimation due to the highly incomplete experience data.

Figure~\ref{fig:convai2_error_hard} compares the error of {\ours} between using the normal experience data and the selected challenging one. As can be seen, the error using the selected data is larger particularly for the agents with exceptionally low/high true reward. That indicates the problem of the lack of dialog coverage is exaggerated under the challenging setting, while the {\ours} estimation remains accurate when there is sufficient dialog coverage.

\textbf{Comparison to Automatic Hand-crafted Metrics.} 

We compare {\ours} with other automatic hand-crafted metrics proposed in \citet{see2019what}. For a more intuitive comparison, we use heat map and box plot to visualize the correlations between different automatic evaluation metrics and different human evaluation metrics. As can be seen in Figure~\ref{fig:heatmap_convai2} and Figure~\ref{fig:boxplot_convai2}, most hand-crafted metrics have relatively low correlation to human evaluation metrics. The only exception is the ``question marks'' automatic metrics for inquisitive human evaluation metric. Some hand-crafted metrics have high Pearson correlation to some human evaluation metrics, while the corresponding Spearman's rank correlation is low. The reason is that they can easily identify some extremely good/bad agents while they are less effective for identifying agents with similar performance. 

\textbf{Comparison to BLEU, BLEURT, and BERTscore.} 
We compare {\ours} with other automatic single-turn language quality metrics in Figure~\ref{fig:heatmap_convai2}: BLEU, BLEURT \citep{sellam2020bleurt}, and BERTscore \citep{zhang2019bertscore}. As can be seen, these metrics only have high correlation to certain human evaluation metrics and low correlation to other metrics. 
Note that, we do not compare the perplexity as the agents rely on complicated decoding methods \citep{see2019what} and perplexity does not take decoding into consideration.

\begin{figure}[!htb]
\begin{center}

\begin{subfigure}{\textwidth}
  \centering
  \includegraphics[width=\textwidth]{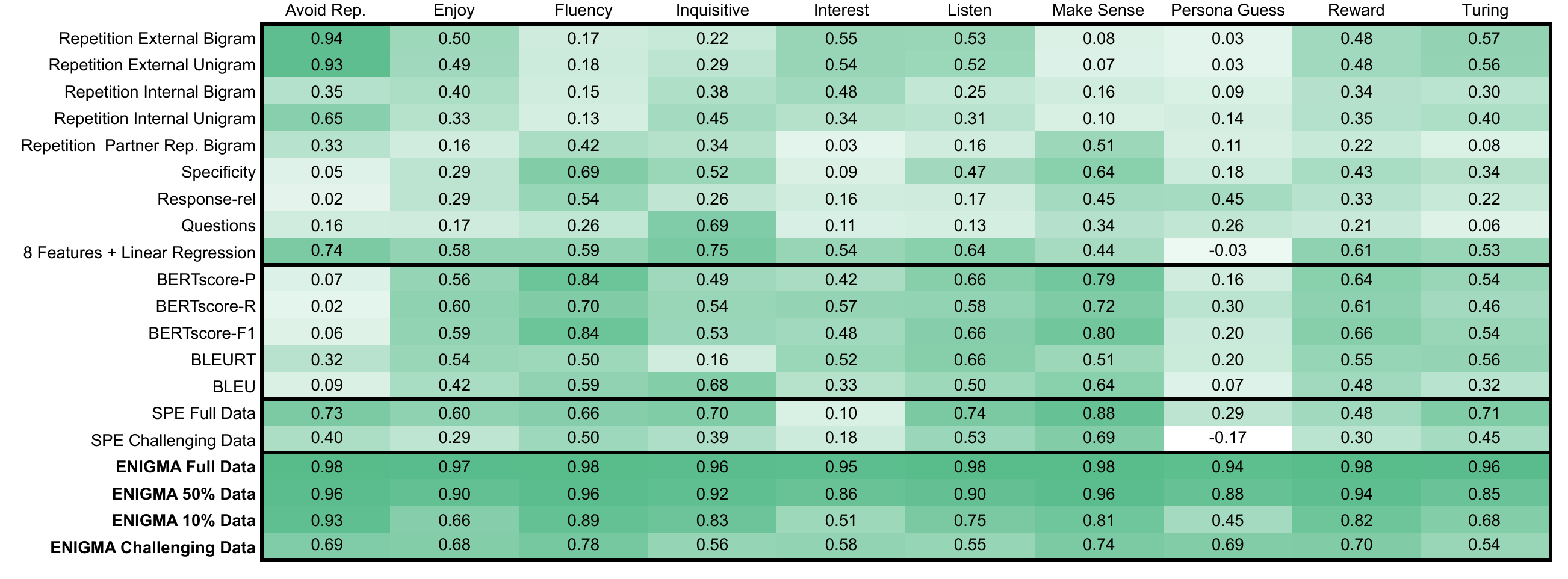}
  \caption{Pearson Correlation}
\end{subfigure}

\begin{subfigure}{\textwidth}
  \centering
  \includegraphics[width=\textwidth]{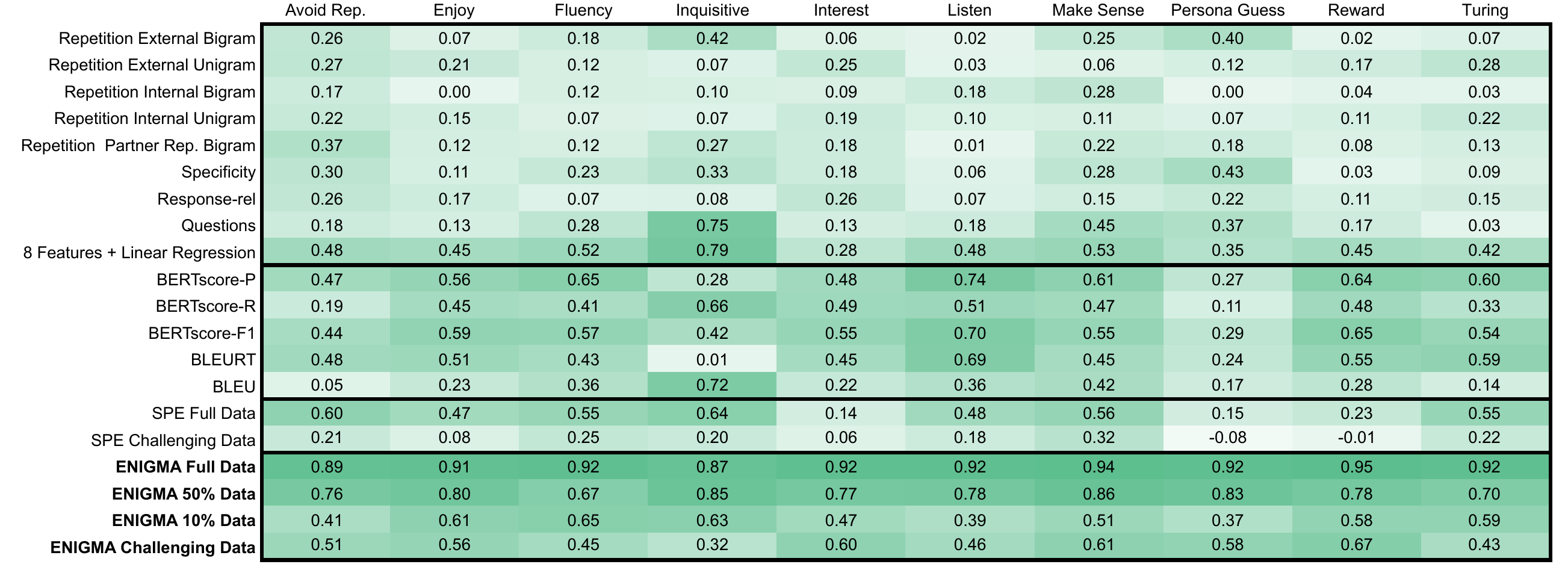}
  \caption{Spearman's Rank  Correlation}
\end{subfigure}

\end{center}
\caption{Heat map for correlation between different automatic evaluation metrics and different human evaluation metrics. Different rows represent different automatic metrics. Different column represent different human evaluation metrics.}
\label{fig:heatmap_convai2}
\end{figure}

\begin{figure}[!htb]
\begin{center}
\begin{subfigure}{0.7\textwidth}
  \centering
  \includegraphics[width=\textwidth]{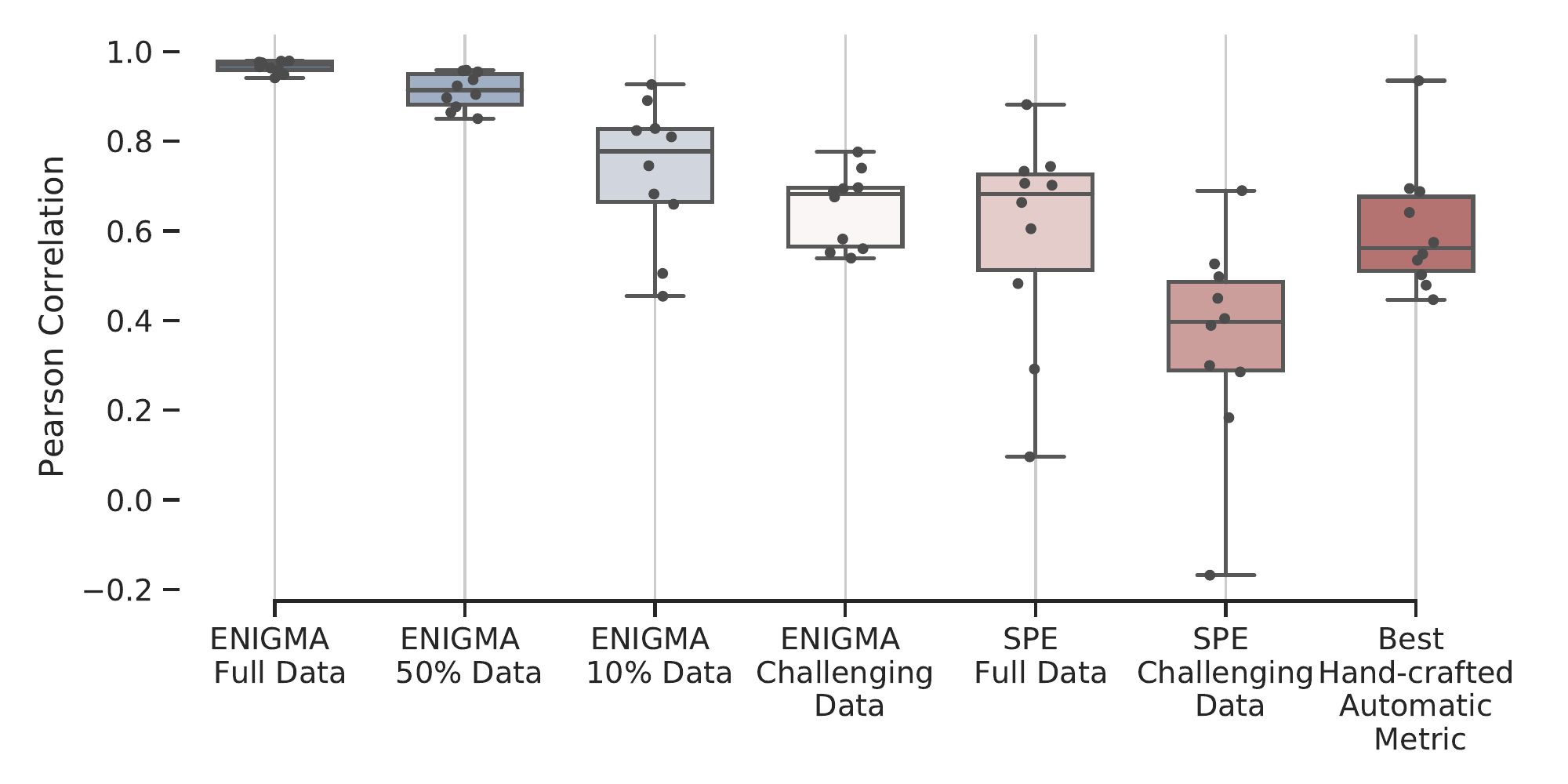}
  \caption{Pearson Correlation}
\end{subfigure}
\begin{subfigure}{0.7\textwidth}
  \centering
  \includegraphics[width=\textwidth]{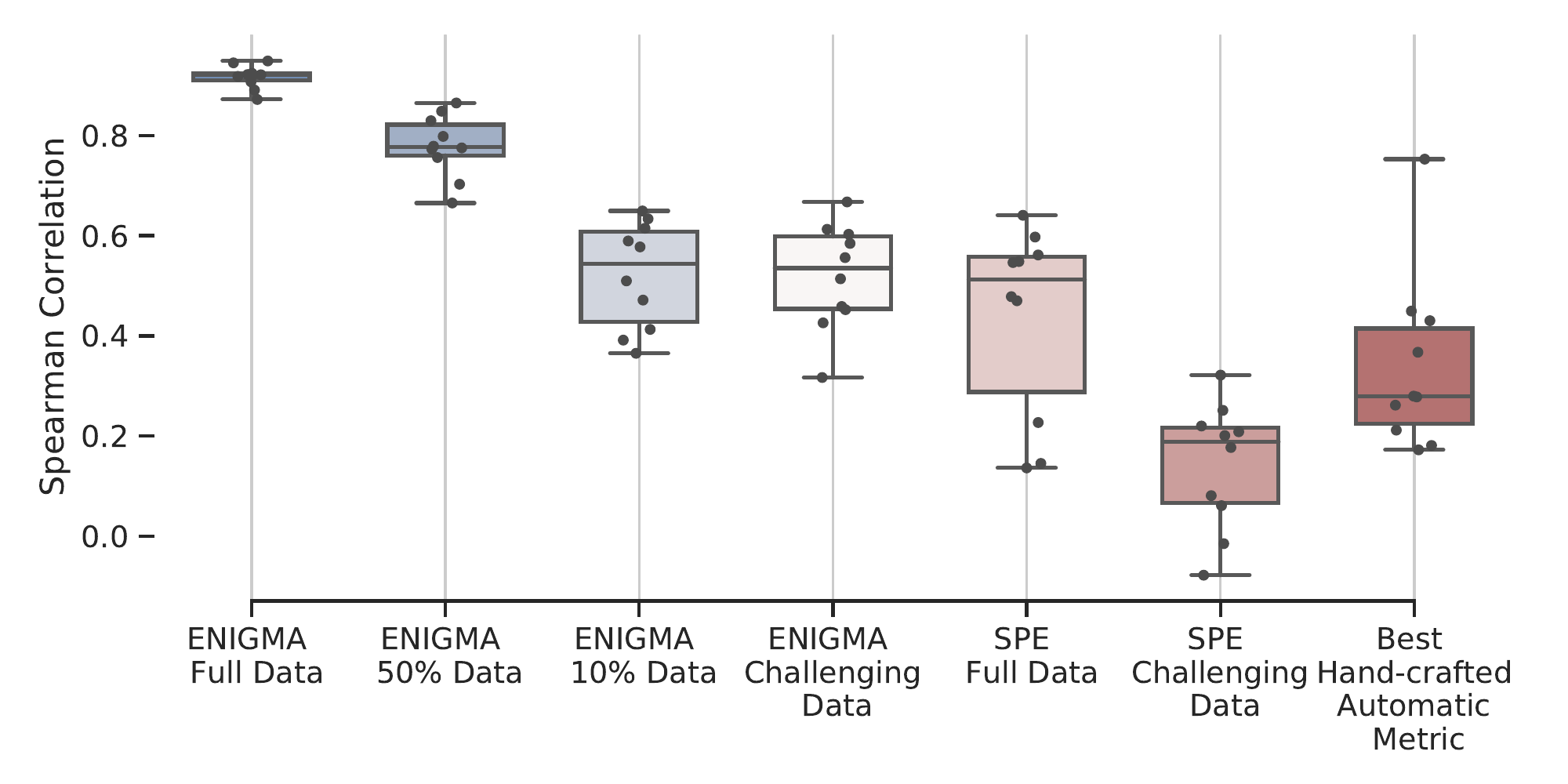}
  \caption{Spearman's Rank Correlation}
\end{subfigure}

\end{center}
\caption{Box plot of performance. Each box corresponds to each method. There are 10 points for each box representing correlations to 10 different human evaluation metrics.}
\label{fig:boxplot_convai2}
\end{figure}

\clearpage

% \vspace{-0.125in}
\subsection{Error Analysis}
% \vspace{-0.075in}

% %\jhm{split into air and convai2}
% %\noindent $\bullet$ {\bf Detailed Error Analysis.}

We analyze the detailed errors to identify the error pattern for better understand the limit of {\ours}. 
% % We analyze the relationship between estimations from {\ours} and the true average reward obtained from interactive dialogs with the environment (model/human). 
We calculate the absolute difference between the estimation and the true average reward. The results are summarized in Figure~\ref{fig:error_analysis}. A common pattern we see in ConvAI2 is that, when the true average reward is too high or too low, the {\ours} becomes less accurate. One possible reason for that is the lack of samples of dialogs with the extreme rewards in the experience data. We empirically verify this conjecture by comparing the the error with the reward distribution in the experience data in Figure~\ref{fig:error_analysis}. 
% As we can see, this is true for the ConvAI2 dataset, since it mainly about language quality metrics. 
For AirDialog, such pattern is not obvious. That is because the quality of the decision module is more important to the agent performance for this task completion scores. As a result, even performance of the target agent is much higher/lower than the experience data, as long as they share similar languages, {\ours} can estimate the performance accurately.

\begin{figure}[!hbt]
% \vspace{-0.1in}
\begin{center}
\begin{subfigure}{0.49\textwidth}
  \centering
    \includegraphics[width=\textwidth]{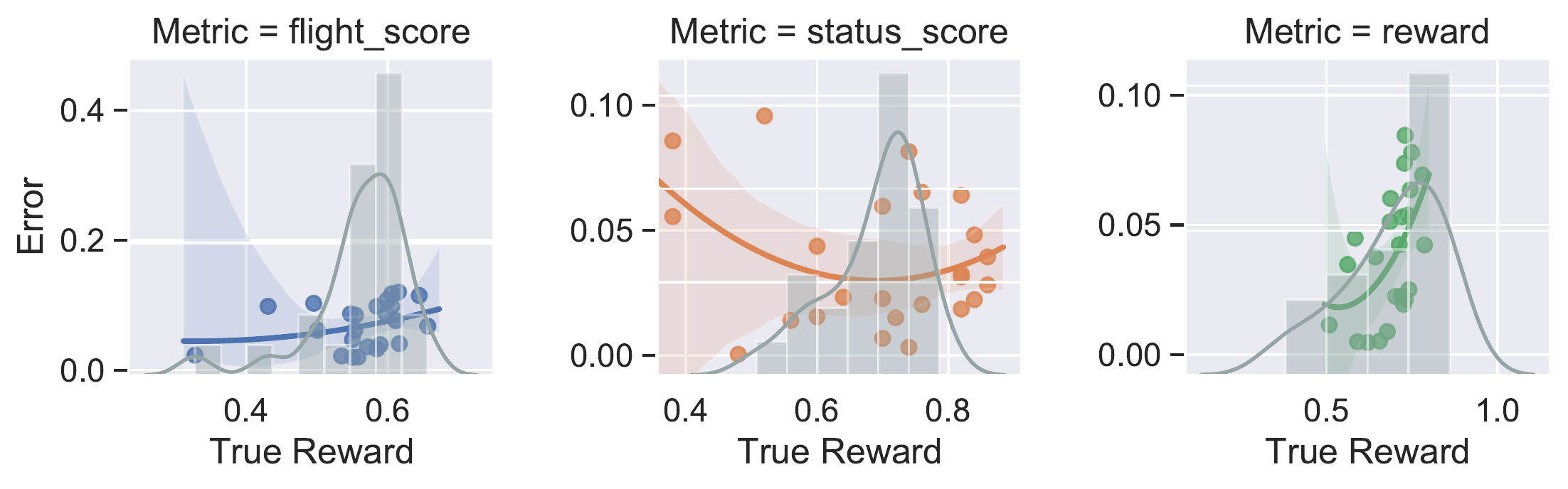}
      \vspace{-0.25in}
  \caption{AirDialog}
\end{subfigure}
% \vspace{0.1in}
\begin{subfigure}{0.8\textwidth}
  \centering
    % \newlength{\imageheightt}
    % \settoheight{\imageheightt}{\includegraphics{figure/convai2/all/error_analysis.pdf}}
    % \includegraphics[trim=0 0 0 0.5\imageheightt{}, clip, width=\textwidth]{figure/convai2/all/error_analysis.pdf}
    \includegraphics[width=\textwidth]{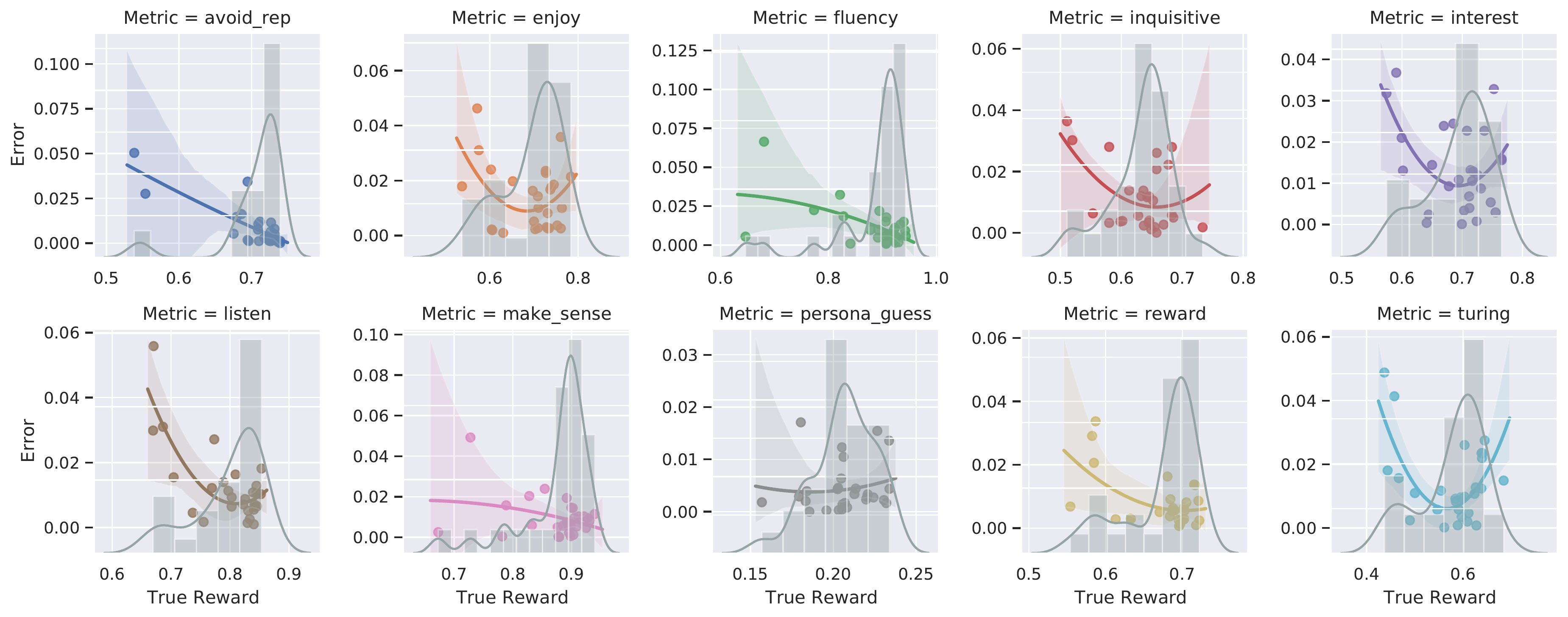}
      \vspace{-0.25in}
  \caption{ConvAI2}
\end{subfigure}
\end{center}
% \vspace{-0.3in}
\caption{Error Analysis on AirDialog and ConvAI2. The x-axis is the true reward. The y-axis is the Estimation error. The solid line is the fitted quadratic function. The histogram is the empirical distribution of the true rewards of all the experience data.}
\label{fig:error_analysis}
% \vspace{-0.15in}
\end{figure}

\subsection{Embedding Visualization}
\label{app:embed_vis}
In Figure~\ref{fig:embedding}, we present the t-SNE plots for the embedding of the state-action pairs from the behavior experience data and the target policy. The two sets of embeddings provided by the pre-trained language models are largely overlapped with rich semantic information. On the other hand, the embeddings provided by a randomly initialized model spread over the entire high-dimensional space.

\begin{figure}[!htb]
\begin{center}
\begin{subfigure}{0.7\textwidth}
  \centering
  \includegraphics[trim=4cm 1cm 3cm 1cm, clip=true,width=\textwidth]{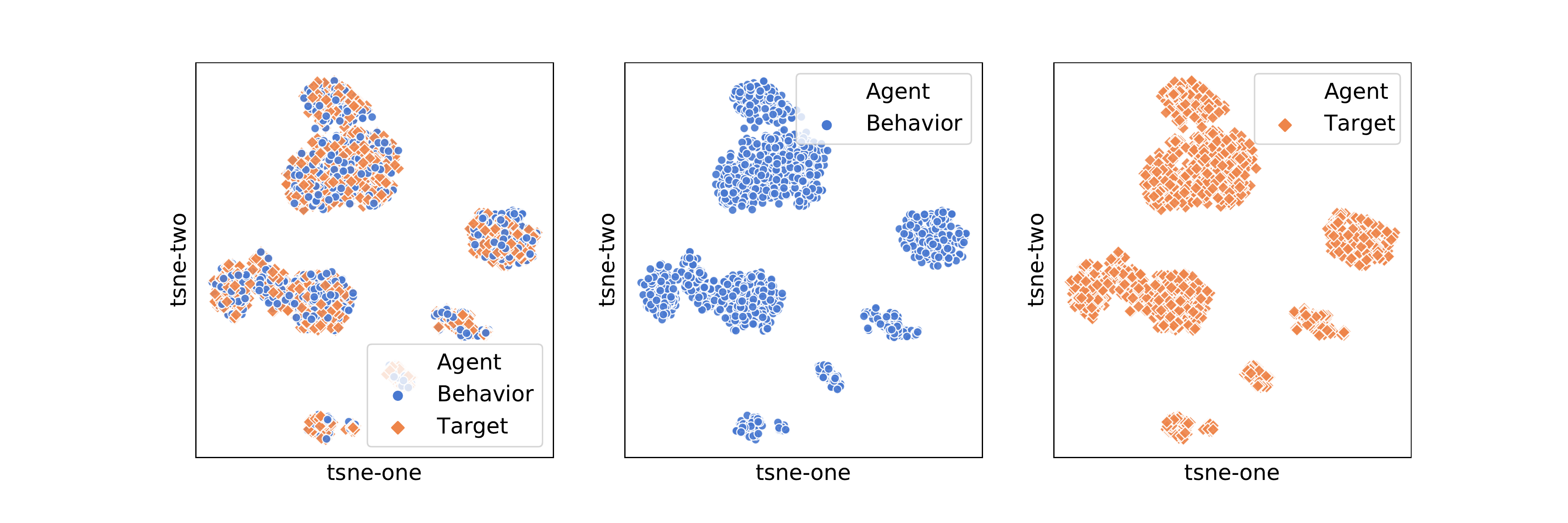}
  \caption{Initialized from RoBERTa-base.}
\end{subfigure}
\begin{subfigure}{0.7\textwidth}
  \centering
  \includegraphics[trim=4cm 1cm 3cm 1cm, clip=true,width=\textwidth]{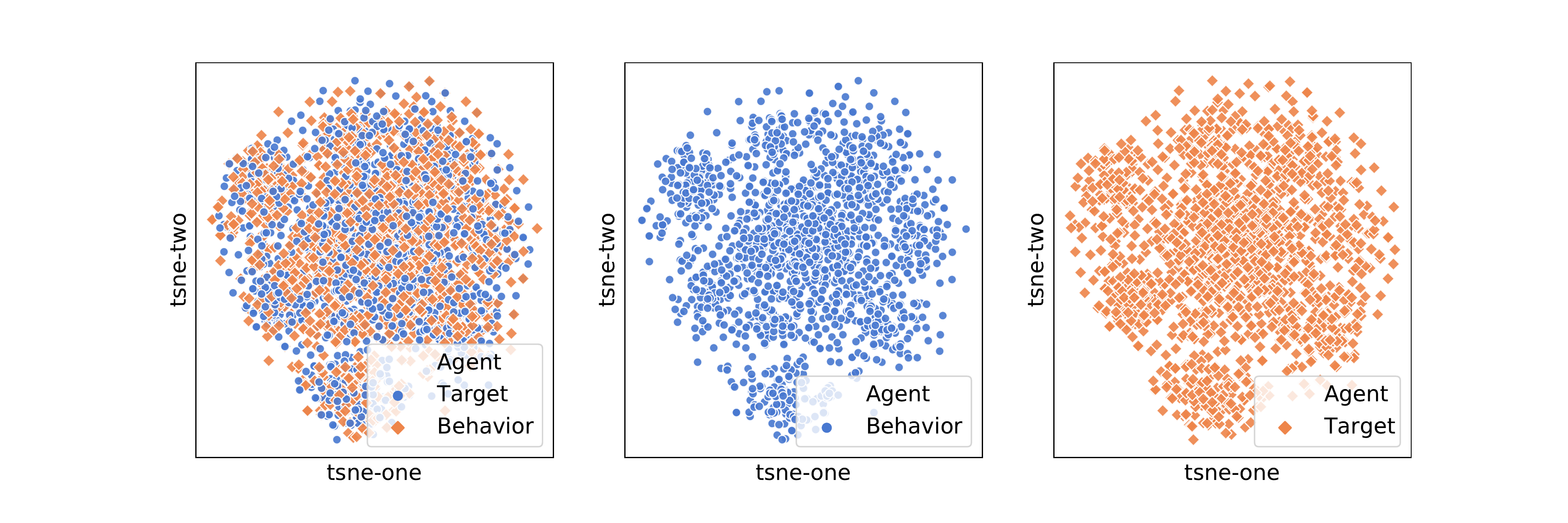}
  \caption{Random initialization.}
\end{subfigure}

\end{center}
\caption{t-SNE Plots for contextual embedding extracted from RoBERTa-$\zeta$ and  RoBERTa-$\nu$ on AirDialog.}
\label{fig:embedding}
\end{figure}

\clearpage
\section{Automatic Dialog Evaluation Comparison}
\label{app:dialog_cmp}

\newcommand{\notshortstack}[2][l]{%
  \begin{tabular}{@{}#1@{}}#2\end{tabular}%
}

\begin{table}[htb!]
    \tiny
	\centering
	\caption{Comparison between current automatic evaluation approaches. Part of the table is collected from two comprehensive surveys \cite{finch2020towards,deriu2020survey}. \textbf{\color{red} Red: Drawback}; \textbf{\color{green} Green: Advantage}. }
    \begin{tabularx}{\textwidth}{ >{\hsize=0.2\hsize\centering}X  *{1}{|>{\hsize=0.07\hsize\centering\arraybackslash}X} 
     *{2}{|>{\hsize=0.05\hsize\centering\arraybackslash}X} 
    *{2}{|>{\hsize=0.11\hsize\centering\arraybackslash}X} *{1}{|>{\hsize=0.06\hsize\centering\arraybackslash}X}
    | >{\hsize=0.26\hsize\centering\arraybackslash}X}
    % \begin{tabular}{m{5em}|c|m{5em}|c|c|c|c|c}
    \toprule \hline
    	Method & Criterion  & Dynamic (RL)
    	& Model Free& Experience Data &  Behavior Policy Similar to Target Policy  & Behavior Agnostic & Description / Examples \\
    	\hline
    	 {BLEU, Perplexity,METEOR,ROUGE \\ \cite{papineni2002bleu,brown1992estimate,banerjee2005meteor,lin2004rouge,galley2015deltableu} } & Language Quality  Score & \cellcolor[HTML]{F6DDCC} No  & N/A & Human-Human & \cellcolor[HTML]{F6DDCC} Yes & N/A & 
    	 The most widely use metrics: 
    	 Given a \textbf{\color{red}  fixed} dialog history, 
    	 they compute heuristic scores / 
    	 statistics based on comparing
    	 \textbf{\color{red} single turn}
    	 response given by the model 
    	 and reference human  responses.  
    	 E.g., BLEU, perplexity \\  
    	 \hline
    	  \citet{mitchell2008vector,rus2012optimal,forgues2014bootstrapping,
    	  higashinaka2014evaluating,xiang2014problematic,wieting2015towards,
    	  gandhe2016semi,tao2017ruber,shimanaka2019machine,zhang2019bertscore,
    	  ghazarian2019better,li2020task,mehri2020unsupervised,gao2020dialogue,
    	  lan2020pone,pang2020towards,zhang2020modeling,yuma2020ubleu,
    	  zhao2020designing,sai2020improving}
    	   &  {Language  Quality  Score}  & \cellcolor[HTML]{F6DDCC} {No}  & N/A & Human-human experience data or specially designed data.  & \cellcolor[HTML]{F6DDCC}  Yes (Implicitly) ~~~~ 
    	   \textit{Although they do not explicitly require such similarity, the single-turn responses of models trained from the same data are usually similar to human responses. } & N/A & 
    	  {
    	 Given a \textbf{\color{red}  fixed} dialog history, 
    	 they compute some scores for \textbf{\color{red}  single-turn} response given by the model 
    	 using \textbf{\color{green} an evaluator}, e.g., 
    	 pretrained word embeddings 
    	 and pretrained language models. 
    	 These method are the so-called ``embedding-based metrics''.
    	 The evaluator usually require training on a large-scale text dataset.
    	 They may or may not depends on 
    	 reference human responses. E.g. RUBER \citep{tao2017ruber}. }
    	 \\
    	 \hline
    	 \citet{lowe2017towards,huang2020grade,sellam2020bleurt}
    	 &  {Language  Quality  Score} & \cellcolor[HTML]{F6DDCC}  No & N/A & Human-Human and Human-Model & \cellcolor[HTML]{F6DDCC} Yes & N/A & 
    	 Mostly the same as above. In addition, the data for training the evaluator 
    	 includes human-model experience data to improve performance.
    	 E.g., ADEM \cite{lowe2017towards}. 
    	 \\
    	 \hline
    	 \citet{hemphill1990atis,williams-etal-2013-dialog} & {Task Completion Score} & \cellcolor[HTML]{F6DDCC}  No  & N/A & Human-Human & No & N/A & {
    	 They compute task related score 
    	 of task-specific actions 
    	 (e.g., intent detection) 
    	 given by the model for 
    	 a \textbf{\color{red} fixed complete} dialog. 
    	 These can only be used to test classification / information retrieval module.
    	 E.g., Intent Detection Accuracy.}
    	 \\ 
    	 \hline
    	 \citet{wei2018airdialogue} & Task Completion Score & \cellcolor[HTML]{D4EFDF} Yes  & \cellcolor[HTML]{F6DDCC} No & Human-Human and/or Human-Model & \cellcolor[HTML]{F6DDCC} Yes (Implicitly)  & N/A &
    	 {
    	 They compute task related score 
    	 of task-specific actions 
    	 (e.g., intent detection) 
    	 given by the model for 
    	 a dialog that is obtained by 
    	 \textbf{\color{green} interaction} with a \textbf{\color{red} user simulator}. 
    	 E.g., Self-Play Evaluation \cite{wei2018airdialogue}.} \\
    	 \hline
    	 \citet{ghandeharioun2019approximating} & Language Quality Score & \cellcolor[HTML]{D4EFDF} Yes  & \cellcolor[HTML]{F6DDCC}  No  & Human-Model & \cellcolor[HTML]{F6DDCC} Yes &  N/A &  {Basically the same as above. 
    	 In addition to modeling human 
    	 responses, they usually require 
    	 \textbf{modeling human reward function}. 
    	 E.g., Self-Play Evaluation \cite{ghandeharioun2019approximating}.}
    	 \\ 
    	 \hline
    	  Inverse Proportional Score E.g., \citet{horvitz1952generalization,wang2020reliable,precup2000eligibility} (not practical for dialog ) 
    	 & \cellcolor[HTML]{D4EFDF} Both & \cellcolor[HTML]{D4EFDF} Yes  &  \cellcolor[HTML]{D4EFDF}  Yes  & Human-Model & \cellcolor[HTML]{F6DDCC} Yes &  \cellcolor[HTML]{F6DDCC} No (not practical for dialog) & {
    	 \textbf{\color{green} Directly model the performance} 
    	 under the \textbf{\color{green} interaction} environment 
    	 using experience collected from 
    	 \textbf{\color{red} known} probabilistic models. 
    	 E.g., Inverse Proportional Score.}
    	 \\ 
    	 \hline
    	 \textbf{\ours} & \cellcolor[HTML]{D4EFDF} Both & \cellcolor[HTML]{D4EFDF} Yes  & \cellcolor[HTML]{D4EFDF} Yes  & Human-Model &  \cellcolor[HTML]{F6DDCC} Yes & \cellcolor[HTML]{D4EFDF} Yes & {
    	 \textbf{\color{green} Directly model the performance} 
    	 under \textbf{\color{green}  interaction} environment 
    	 using experience collected from 
    	 \textbf{\color{green} unknown} distribution. 
    	 E.g., Q-Learning, ENIGMA.}
    	 \\ 
    \hline \bottomrule
    \end{tabularx}
\end{table}

\subsection{Static Methods}
As can be seen, most previous methods only focus on evaluating \textit{language quality} for \textbf{\color{red} single-turn} response of a \textbf{\color{red} fixed} context. These methods can not evaluate agents under interactive context. As a result, they can not be extended to \textit{goal-oriented} dialogs. 

For goal-oriented dialogs, the static evaluation methods are very limited. The static methods can only evaluate the model actions to a \textbf{\color{red} fixed complete} dialog, e.g., intent detection.

\textbf{Comparison to Meena Paper \citep{adiwardana2020towards}}: 1. They only show that PPL correlates with \textbf{one specific} metric: Sensibleness and Specificity Average. We consider a wide range of metrics for both task-completion scores and dialog quality scores (listed in Table 7). No evidence shows
PPL correlates well with most metrics. 
2. They draw the conclusion using \textbf{only 7} chatbots. This conclusion is not statistically reliable, i.e. for $R^2=0.93$ with 7 data points, the $95\%$ confident interval is $0.64\leq R^2 \leq 0.99$. On the other hand, we use \textbf{24/29} agents. With 24 data points, the $95\%$ CI is $0.87\leq R^2 \leq 0.96$, which is much more reliable.

\subsection{Dynamic Methods}
Previous dynamic methods under RL framework are based on self-play evluation, which requires learning the environment, i.e, human. As discussed in the main paper, learning a human model is significantly beyond the current technical limit. 

{\ours} overcome learning the environment by directly modeling the performance of agents.

\subsection{Information Theoretic Limit}
The common limitation of all existing methods is that they require similarity between the target policy and behavioral policies, so that the experience data can cover sufficient interaction patterns between the target policy and human.

For example, BLEU score requires the agent response being similar to the reference response.
Another example is ADEM \citep{lowe2017towards}, they include the target policy into the experience data collection to achieve decent performance ($0.37$ Pearson correlation to human ratings). If the target policy is excluded from the behavior policies, ADEM only achieves $0.13$ Pearson correlation, which is even lower than the one between dialog length and human ratings $0.27$.

For static single-turn evaluation for language quality, one might satisfy the requirement by just using human as the behavior policy and large-scale diverse experience data.
That is because the single-turn responses of the target model have a very similar pattern to the human responses, as they are usually trained to mimic one-turn human response. 
However, high similarity of responses between the target model and human requires a very strong target model trained with large-scale data, which is not practical in most settings. Some existing work try to alleviate such requirement and increase the coverage of experience data by external knowledge graph \citep{huang2020grade} and synthetic samples \citep{sellam2020bleurt}. We remark that although the static methods only require single-turn similarity between behavior and target policies, their empirical performance is unsatisfactory comparing with multi-turn interactive human evaluation \citep{ghandeharioun2019approximating}. 

In multi-turn interactive evaluation, we can not just use human as the behavior policy especially for goal-oriented dialogs. That is because the multi-turn behavior of the target model is very different from the human behavior. Take Airdialog as an example, human agents can always book the correct tickets while the target model may fail for many times. 

Such a limitation is the theoretical requirement of bounded state-action density ratio between target and behavior policies, which
has been discussed in many off-policy evaluation literature \cite{wang2020statistical,xie2019towards}.

Due to such theoretical limitation, a large amount of \textbf{human-model} interactive evaluation data is needed to study automatic interactive evaluation. However, most evaluation logs are not publicly available, and research in this direction has largely lagged behind. 
To the best of our knowledge, ConvAI2
\citep{see2019what} 
is the only public comprehensive human-model interactive evaluation data. \footnote{Our human-model evaluation data on Airdialog will also be released soon.} Therefore, we recommend that the research community release human-model interaction evaluation data to promote dialog evaluation/learning research and benefit the entire community.

% Reading list 
% \cite{finch2020towards}
% % \citep{moller2006memo,li2016deep,yu2016strategy,shah2018bootstrapping,ghandeharioun2019approximating,jaques2019way}
% \citep{deriu2020survey}
% \cite{sai2020survey}

\end{document}